\documentclass[12pt]{article}
\usepackage[margin=0.8in]{geometry}

\usepackage[utf8]{inputenc} 
\usepackage[T1]{fontenc}    
\usepackage{hyperref}       
\usepackage{url}            
\usepackage{booktabs, multicol, multirow}       
\usepackage{amsfonts}       
\usepackage{amsmath}
\usepackage{nicefrac}       
\usepackage{microtype}      

\usepackage{makecell,colortbl}

\usepackage{morenotations,rotating}
\usepackage{color,soul,upgreek}
\usepackage{arydshln}
\usepackage{wrapfig}
\usepackage{soul}
\usepackage{amsmath}
\usepackage{empheq}
\usepackage{mdframed,enumerate}

\usepackage{commath}
\usepackage{graphicx} 
\usepackage{cleveref}
\usepackage{bm} 
\usepackage{fancyhdr}
\usepackage{graphicx}
\usepackage{color}
\usepackage{natbib}
\usepackage{booktabs}

\usepackage{algorithm}
\usepackage{algorithmic}

\usepackage{graphicx}
\usepackage{caption}
\usepackage{subcaption}
\usepackage{tcolorbox}

\usepackage{tikz}
\usetikzlibrary{arrows,decorations.pathmorphing,backgrounds,positioning,fit,petri}

\title{\papertitleArXiv}

\author{
Hisham Husain$^{\ddagger,\dagger}$ $\:\:\:$ Zac
Cranko$^{\ddagger,\dagger}$ $\:\:\:$ Richard Nock$^{\dagger,\ddagger,\clubsuit}$\\\\
{\small $^\dagger$Data61, $^\ddagger$the Australian National
  University, $^\clubsuit$the University of Sydney
 }\\\\
  \texttt{firstname.lastname@$\{$data61.csiro.au,anu.edu.au$\}$}
}
\date{}

\begin{document}
\thispagestyle{empty}
\maketitle

%

%

\begin{abstract}
$\varepsilon$-differential privacy is a leading protection setting, focused by design on individual privacy. Many applications, in medical / pharmaceutical domains or social networks, rather posit privacy at a \textit{group} level, a setting we call \textit{integral privacy}. We aim for the strongest form of privacy: the group size is in particular \textit{not known} in advance.
We study a problem with related applications in domains cited above that have recently met with substantial recent press: sampling.

Keeping correct utility levels in such a strong model of statistical indistinguishability looks difficult to be achieved with the usual differential privacy toolbox because it would typically scale in the worst case the sensitivity by the sample size and so the noise variance by up to its \textit{square}. We introduce a trick specific to sampling that bypasses the sensitivity analysis. Privacy enforces an information theoretic barrier on approximation, and we show how to reach this barrier with guarantees on the approximation of the target non private density. We do so using a recent approach to non private density estimation relying on the original boosting theory, learning the sufficient statistics of an exponential family with classifiers. Approximation guarantees cover the mode capture problem. In the context of learning, the sampling problem is particularly important: because integral privacy enjoys the same closure under post-processing as differential privacy does, \textit{any} algorithm using integrally privacy sampled data would result in an output equally integrally private. We also show that this brings fairness guarantees on post-processing that would eventually elude classical differential privacy: any decision process has bounded data-dependent bias when the
data is integrally privately sampled.
Experimental results against private kernel density estimation and private GANs displays the quality of our results.
\end{abstract}

\section{Introduction}\label{sec:intro}

Over the past decade, ($\varepsilon$-)differential privacy (DP) has evolved as the leading statistical protection model for individuals \citep{drTA}, as it guarantees plausible deniability regarding the presence of an individual in the input of a mechanism, from the observation of its output.

DP has however a limitation inherent to its formulation regarding \textit{group} protection: what if we wish to extend the guarantee to \textit{subsets} of the input, not just individuals ? Several recent work have started to tackle the problem in different settings where privacy naturally occurs at a feature level, either by inclusion (protect buyers of psychiatric drugs \citep{plkGD}) or by exclusion (protect non-targeted individuals \citep{krwyPA,wDP}). When the group size is limited, it is a simple textbook matter to extend the privacy guarantee by using the subadditivity of the classical sensitivity functions, see \textit{e.g.} \cite[Proposition 1.13]{gTI}, \cite[p 192]{drTA}. This is however not very efficient to retain information as standard randomized mechanisms would get their variance \textit{scaled up to the square of} the maximal group size \cite[Section 3]{drTA} (see also Figure \ref{fig:Princ}, left).

There is fortunately a workaround which we develop in this paper from a recent boosting algorithm for non private density estimation \citep{cnBD}, and grants protection for \textit{all} subgroups of the population (and not just singletons as for DP), a setting we refer to as \textit{integral privacy}. Just as DP, integral privacy is not a binary notion of privacy: it comes with a budget whose relaxation can allow for better approximations of the non-private objective. The take-home message from our paper
challenges the misleading intuition that integral privacy would push too far the constraints on statistical indistinguishability to allow for efficient learning: there is indeed an information-theoretic barrier --- which we give --- for solutions to be integrally private, but we show how to reach it (Theorem \ref{kl-upper-lower}) while delivering guaranteed approximations of the target under just slightly stronger assumptions than those of the boosting model \citep{cnBD} (Theorem \ref{kl-upper}); furthermore, we are also able to give approximation guarantees on a crucial problem for sampling and generative models: mode capture (Theorem \ref{first-mode-capture}). As the integral privacy constraint vanishes, approximations converge to the best possible results, inline with \cite{cnBD}. In the other direction, as the integral privacy constraint is reinforced, stronger guarantees hold on the relative independence of the output of any sensitive algorithm (\textit{e.g.} deciding a loan or hire for a particular input individual) with respect to any group input data. In other words, we get guarantees on unbiasedness or fairness that can elude individual privacy mechanisms like classical differential privacy (Section \ref{sec:disc}). This is an important by-product of our model, considering the recent experimental evidence of the potential negative impact of differential privacy on fairness \citep{bsDPH}.

\begin{figure}[t]
  \centering
\begin{tabular}{lcc"c"c|ccc}
 & \multirow{2}{*}{\makecell{\vspace{-2cm}{\Large $\frac{\var(m)}{\var(1)}$}}} & \multirow{2}{*}{\makecell{\vspace{-2cm}{\Large $\frac{\varepsilon(k)}{\varepsilon(1)}$}}} & \multirow{3}{*}{\makecell{\includegraphics[trim=90bp 410bp 480bp
                                                                         10bp,clip,width=.24\linewidth]{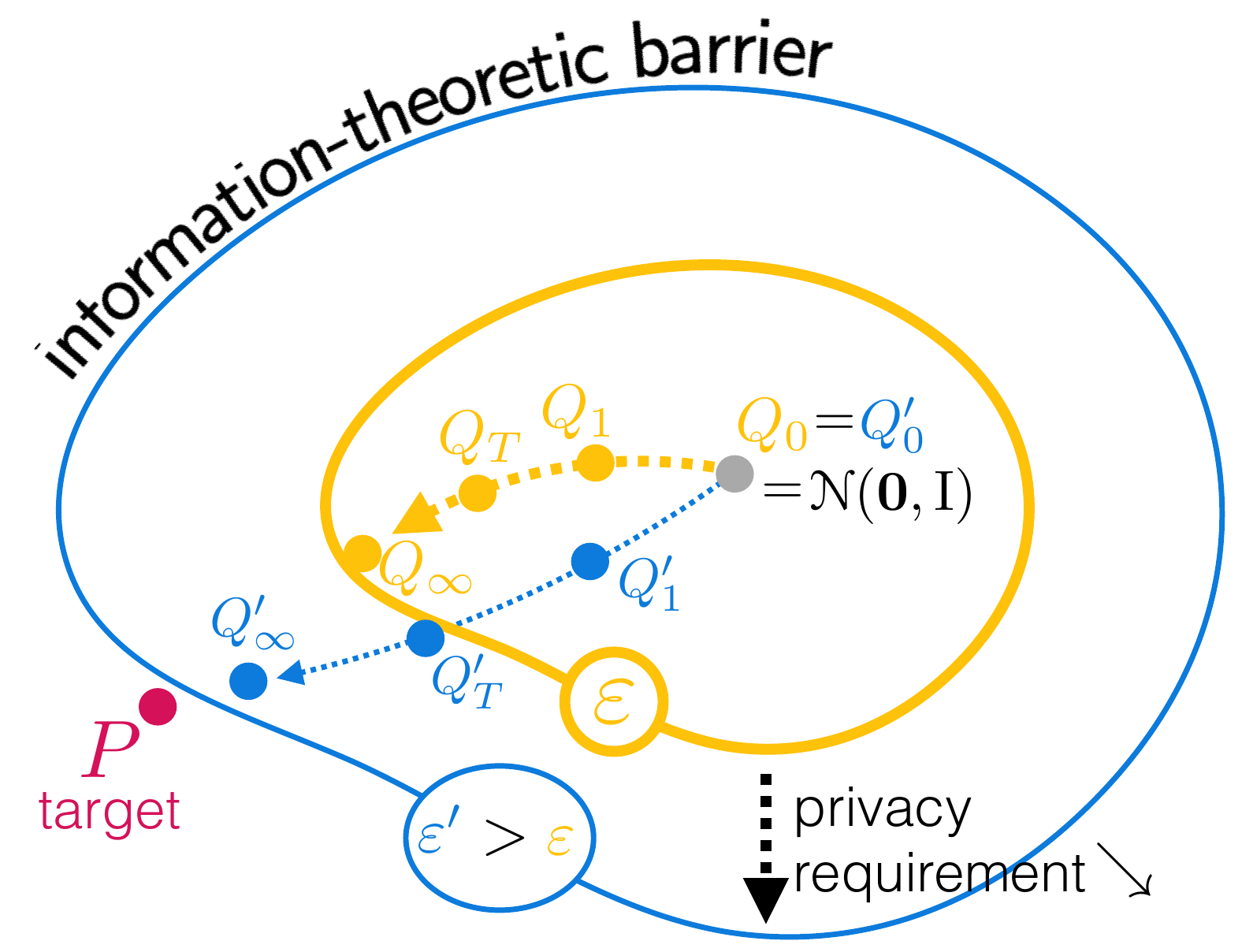}}} & \hspace{-0.35cm} target $P$ \hspace{-0.35cm} & \multicolumn{3}{c}{\hspace{-0.2cm} $Q$ learned} \\
  & & & & \hspace{-0.30cm} \includegraphics[trim=20bp 20bp 20bp
20bp,clip,width=0.115\textwidth,height=0.115\textwidth]{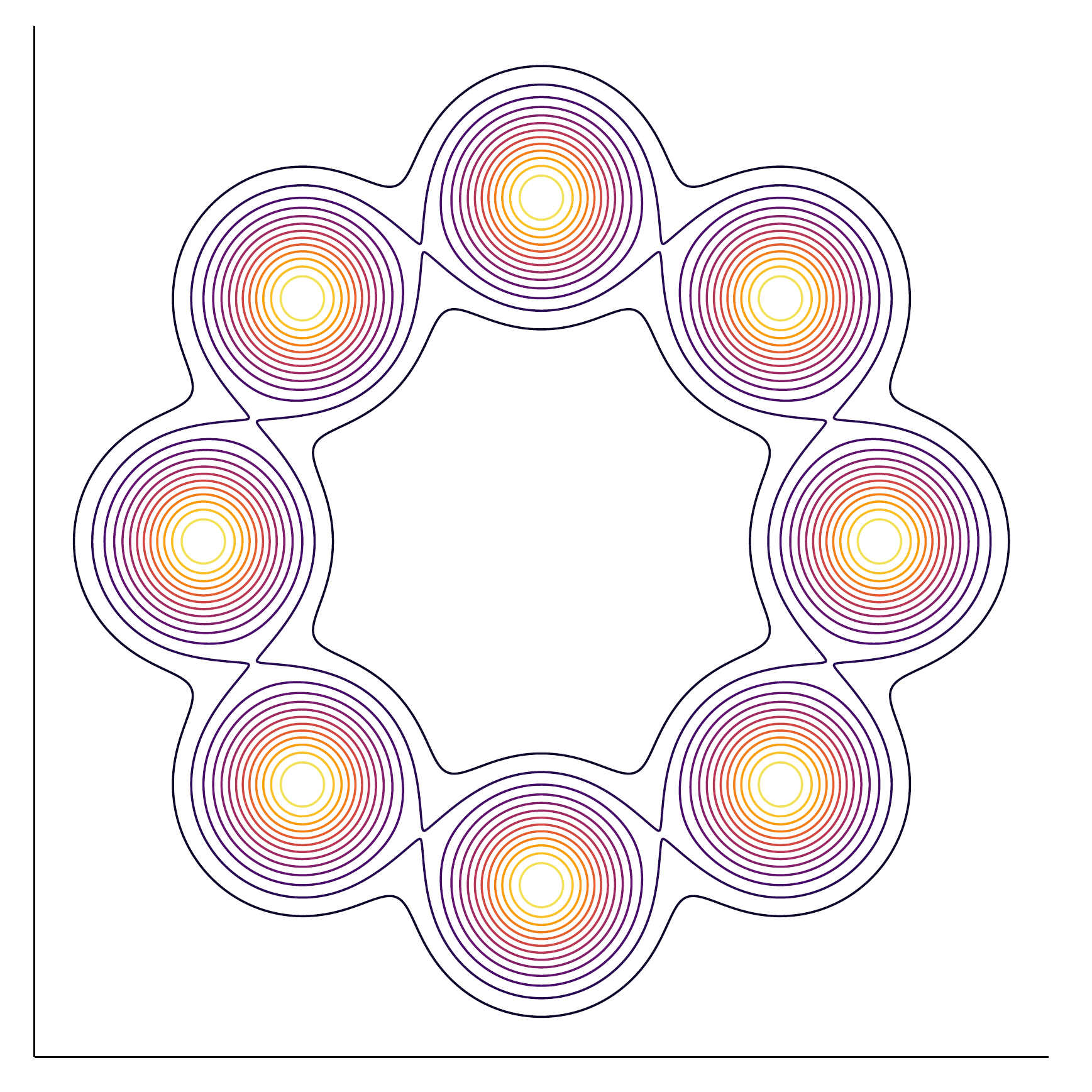}
\hspace{-0.35cm} & \hspace{-0.2cm}  \includegraphics[trim=20bp 20bp 20bp
20bp,clip,width=0.115\textwidth,height=0.115\textwidth]{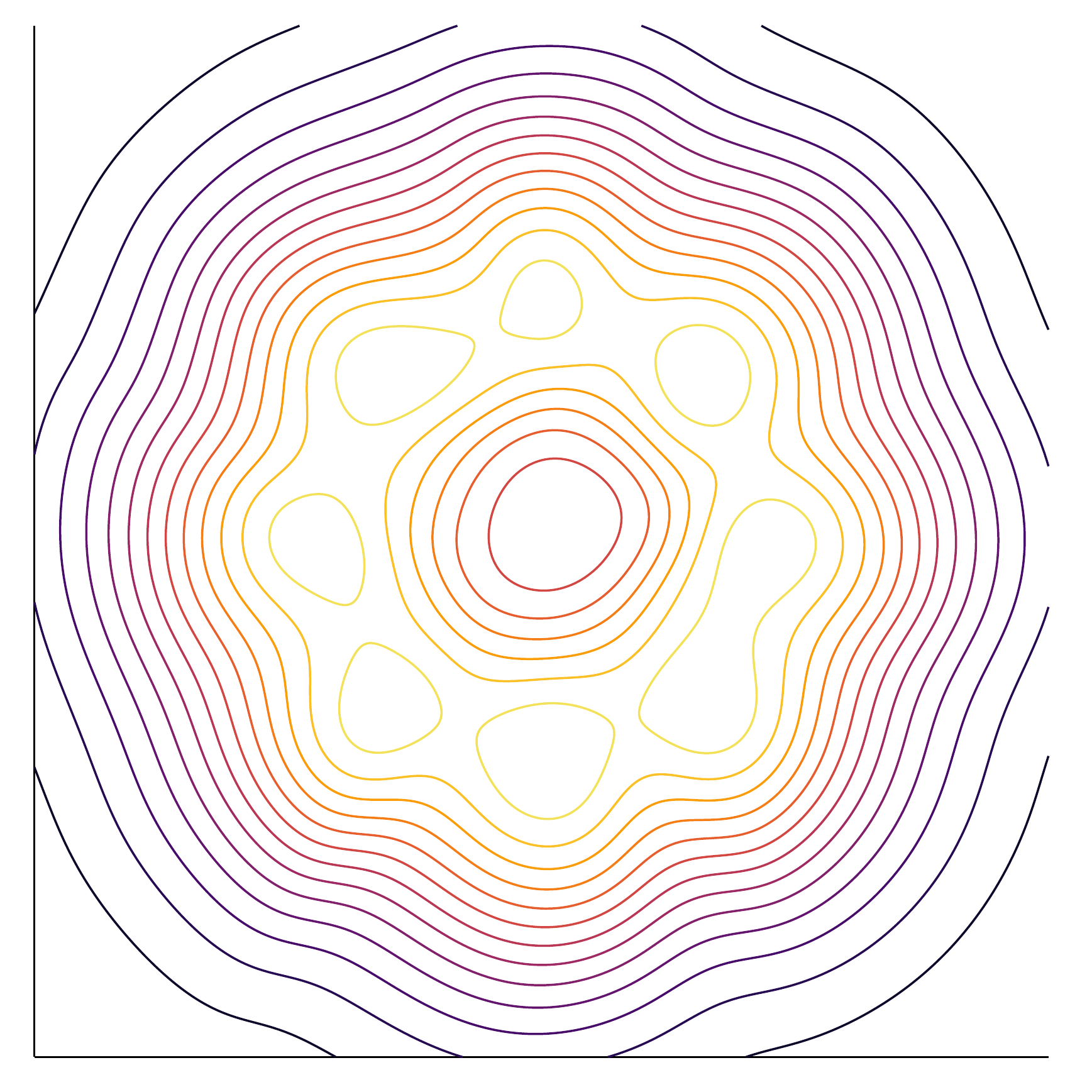}
\hspace{-0.2cm} & \hspace{-0.2cm}  \includegraphics[trim=15bp 15bp 15bp
15bp,clip,width=0.115\textwidth,height=0.115\textwidth]{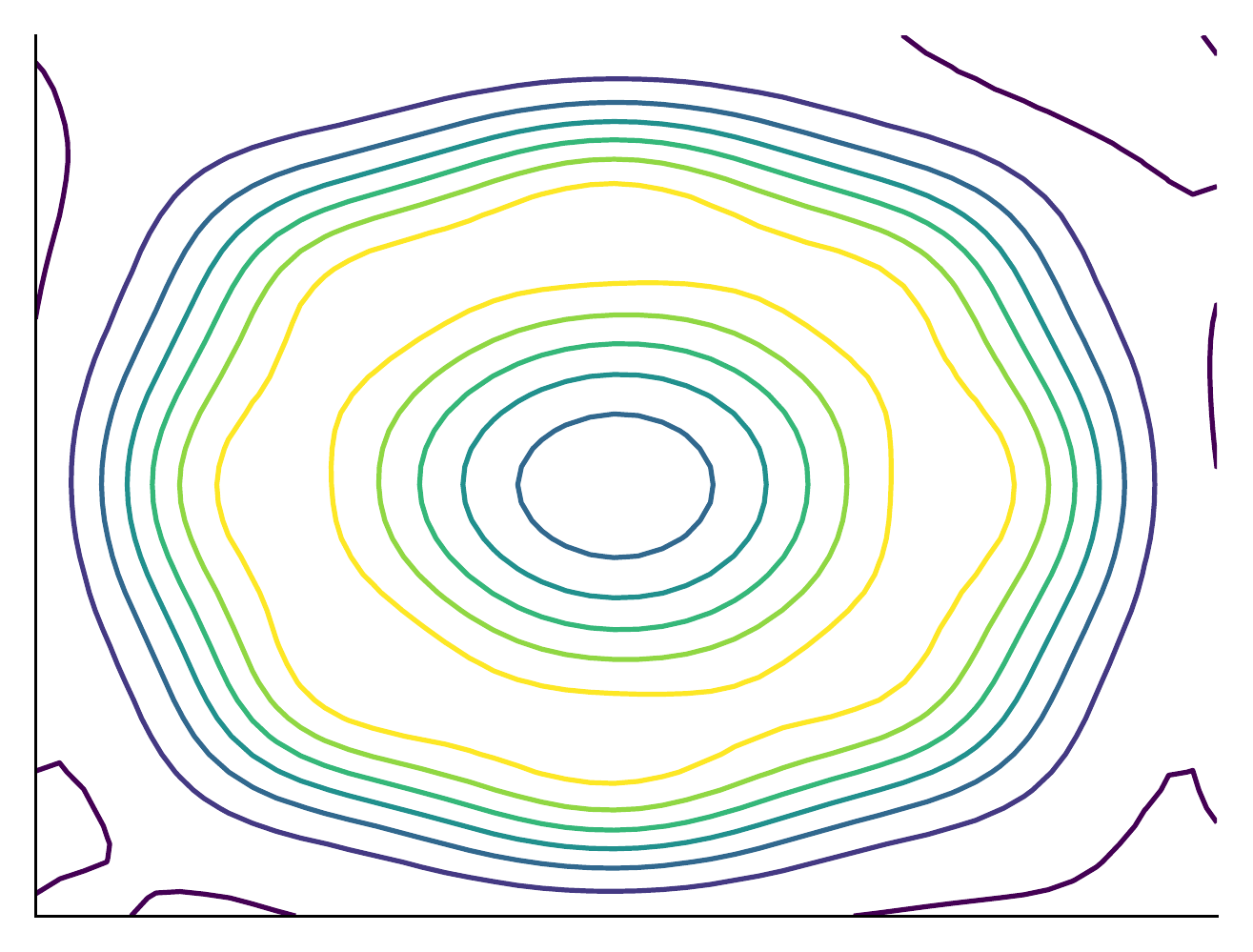}
\hspace{-0.35cm} & \hspace{-0.35cm} 
\imagetop{\includegraphics[trim=15bp 80bp 15bp
120bp,clip,width=0.125\textwidth,height=0.05\textwidth]{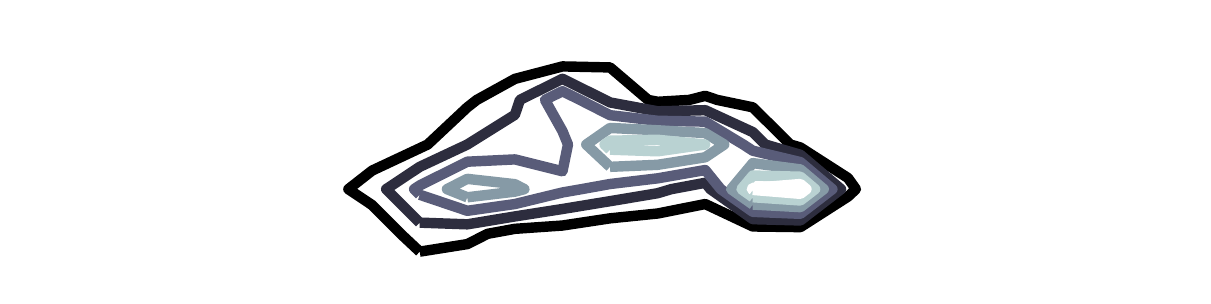}}\\
  {AR} & {\Large $m^2$} & {\Large $1$} & &  \hspace{-0.35cm} & \hspace{-0.35cm}  Us \hspace{-0.2cm} & \hspace{-0.2cm}  Private KDE \hspace{-0.35cm} & \hspace{-0.35cm}  DPGAN\\
  {  us} & {\Large $1$} & {\Large $k$} & Us: $Q_T \rightarrow P$ as $T, \varepsilon \nearrow$ & \hspace{-0.2cm} & \hspace{-0.2cm}   $\epsilon = 0.25$              \hspace{-0.2cm} & \hspace{-0.2cm}  $\epsilon = 100$ \hspace{-0.35cm} & \hspace{-0.35cm} $\epsilon = 5000$\\
\end{tabular}
\caption{\textit{Center}: our method is guaranteed to get a $Q_T$ that converges to $P$ as privacy constraint is relaxed and the number of boosting iterations increases, under a weak learning assumption. \textit{Left}: the variance of the integrally private sampling in \cite{arTB} (AR) ($\var(m)$) over the differentially private sampling ($\var(1)$) scales quadratically with the number of points $m$ to train the private sampler. We do not suffer this drawback because our distributions do not change. However, our privacy budget to sample $k$ points ($\varepsilon(k)$) scales linearly compared to the privacy budget to release one ($\varepsilon(1)$); in comparison, \cite{arTB}'s does not change. \textit{Right}: Our method vs private KDE \citep{arTB} and DPGAN \citep{xlwwzDP} on a ring Gaussian mixture (see Section \ref{sec:exp}, $m = k = 10 000$). Remark that the GAN is subject to mode collapse.\label{fig:Princ}}
\vspace{-0.4cm}
\end{figure}

\noindent $\triangleright$ \textbf{The trick we exploit} bypasses sensitivity analysis which would risk blowing up the noise variance (Figure \ref{fig:Princ}). This trick considers subsets of densities with prescribed range, that we call \textit{mollifiers}\footnote{It bears superficial similarities with functional mollifiers \cite[Section 7.2]{gtEP}.}, which \textit{directly} grants integral privacy when sampling. Related tricks, albeit significantly more constrained and/or tailored to weaker models of privacy, have recently been used in the context of private Bayesian inference, sampling and regression \citep[Section 3]{dnmrRA}, \citep[Chapter 5]{mDP}, \citep[Theorem 1]{wfsPF}, \citep[Section 4.1]{wzAS}. We end up with a bound on the density ratio as in DP but \textit{we do not require anymore samples to be neighbors nor even have related size}; because integral privacy trivially enjoys the same closure under post-processing as differential privacy does, focusing on the upstream task of sampling has a major interest for learning: \textit{any} learning algorithm using integrally private data would be equally integrally private in its output;\\[0.2\baselineskip]
\noindent $\triangleright$ in this set of mollifier densities, we show how to modify the boosted density estimation algorithm of \cite{cnBD} to learn in a mollifier exponential family --- we in fact learn its sufficient statistics using \textit{classifiers} ---, with new guarantees on the approximation of the non private target density and the covering of its modes that degrade gracefully as the privacy requirement increases.\\[0.2\baselineskip]
\noindent $\triangleright$ Our approach comes with a caveat: the privacy budget spent is proportional to the size of the output (private) sample $k$, see Figure \ref{fig:Princ} (left), which makes our technique worth typically when $k \ll m$. Such was the setting of Australia's Medicare data hack for which $k/m = 0.1$ \citep{rtcUT,crtHD}, and could be the setting of many others \citep{nlTT}. Such a case makes our technique highly competitive against traditional DP techniques that would be scaled for integral privacy by scaling the sensitivity: Figure \ref{fig:Princ} (right) compares with \cite{arTB}'s technique. It is clear that our total integrally private budget ($k\epsilon$) is much smaller than \cite{arTB}'s ($m\epsilon$). The technique of Ald{\`a} and Rubinstein has the advantage to compute a private density: it can generate any number of points keeping the same privacy budget. It however suffers from significant drawbacks that we do not have: (i) its sensitivity is \textit{exponential} in the domain's dimension and (ii) is does not guarantee to output positive measures. Finally, our results' quality would be kept even by dividing $k$ by order of magnitudes, which we did not manage to keep for \cite{arTB}.

The rest of this paper is organized as follows. $\S$ \ref{sec:related} presents related work. $\S$ \ref{sec:def} introduces key definitions and basic results. $\S$ \ref{sec:algo} introduces our algorithm, \pkde, and states its key privacy and approximation properties. $\S$ \ref{sec:exp} presents experiments and two last Sections respectively discuss and conclude our paper.
Proofs are postponed to an \supplement.

\section{Related work}\label{sec:related}

A broad literature has been developed early for discrete distributions \citep{mkagvPT} (and references therein). For a general $Q$ not necessarily discrete, more sophisticated approaches have been tried, most of which exploit randomisation and the basic toolbox of differential privacy \citep[Section 3]{drTA}: given non-private $\tilde{Q}$, one compute the \textit{sensitivity} $s$ of the approach, then use a standard mechanism $M(\tilde{Q}, s)$ to compute a private $Q$. If mechanism delivers $\epsilon$-DP, like Laplace mechanism \citep{drTA}, then we get an $\epsilon$-DP density. Such general approaches have been used for $Q$ being the popular kernel density estimation (KDE, \citep{ghCS}) with variants \citep{arTB,hrwDP,raPF}. 
A convenient way to fit a private $Q$ is to \textit{approximate} it in a specific \textit{function space}, being Sobolev \citep{djwLP,hrwDP,wzAS}, Bernstein polynomials \citep{arTB}, Chebyshev polynomials \citep{tuvFA}, and then compute the coefficients in a differentially private way. This approach suffers several drawbacks. First, the sensitivity $s$ depends on the quality of the approximation: increasing it can blow-up sensitivity in an exponential way \citep{arTB,raPF}, which translates to a significantly larger amount of noise. Second, one always pays the price of the underlying function space's assumptions, even if limited to smoothness \citep{djwLP,djwLPD,hrwDP,wCF,wzAS}, continuity or boundedness \citep{arTB,djwLP,djwLPD,tuvFA}. We note that we have framed the general approach to private density estimation in $\epsilon$-DP. While the state of the art we consider investigate privacy models that are closely related, not all are related to ($\epsilon$) differential privacy. Some models opt for a more \textit{local} (or "on device", because the sample size is one) form of differential privacy \citep{adptLW,djwLP,djwLPD,wCF}, others for \textit{relaxed} forms of differential privacy \citep{hrwDP,raPF}. Finally, the quality of the approximation of $Q$ with respect to $P$ is much less investigated. The state of the art investigates criteria of the form $J(P, Q) \defeq \E I(P, Q)$ where the expectation involves all relevant randomizations, \textit{including} sampling of $S$, mechanism $M$, etc. \citep{djwLP,djwLPD,wCF,wzAS}; minimax rates $J^* \defeq \inf_Q \sup_P J (P, Q)$ are also known \citep{djwLP,djwLPD,wCF}. Pointwise approximation bounds are available \citep{arTB} but require substantial assumptions on the target density or sensitivity to remain tractable.

\section{Basic definitions and results}\label{sec:def}

\noindent $\triangleright$ Basic definitions: let $\mathcal{X}$ be a set (typically, $\mathcal{X} = \mathbb{R}^d$) and let $P$ be the target density. Without loss of generality, all distributions considered have the same support, $\mathcal{X}$. We are given a dataset $D = \braces{x_i}_i$, where each $x_i \sim P$ is an i.i.d. observation. As part of our goal is to learn and then sample from a distribution $Q$ such that $\text{KL}(P,Q)$ is small, where KL denotes the Kullback-Leibler divergence:
\begin{eqnarray} 
\text{KL}(P,Q) &= & \int_{\mathcal{X}} \log\bracket{\frac{P}{Q}}dP
\end{eqnarray}
(we assume for the sake of simplicity the same base measure for all densities, allowing to simplify our notations at the expense of slight abuses of language). We pick the KL divergence for its popularity and the fact that it is the canonical divergence for broad sets of distributions \citep{anMO}. \\[0.2\baselineskip]
\noindent $\triangleright$ Boosting: in supervised learning, a classifier is a function $c : \mathcal{X} \rightarrow \mathbb{R}$ where $\mathrm{sign}(c(x)) \in \{-1, 1\}$ denotes a class. We assume that $c(x) \in [-\log2,\log2]$ and so the output of $c$ is bounded. This is not a restrictive assumption: many other work in the boosting literature make the same boundedness assumption \citep{ssIBj}. We now present the cornerstone of boosting, the weak learning assumption. It involves a weak learner, which is an oracle taking as inputs two distributions $P$ and $Q$ and is required to always return a classifier $c$ that weakly guesses the sampling from $P$ vs $Q$.
\begin{definition}[WLA]\label{def-wla}
Fix $\gamma_P, \gamma_{Q} \in (0,1]$ two constants. We say that $\text{WeakLearner}(.,.)$ satisfies the \textbf{weak learning assumption} (WLA) for $\gamma_P, \gamma_{Q}$ iff for any $P, Q$, $\text{WeakLearner}(P,Q)$ returns a classifier $c$ satisfying $\frac{1}{c^{*}} \cdot \E_P[c] > \gamma_P$ and $\frac{1}{c^{*}} \cdot \E_{Q} [-c] > \gamma_Q$, where $c^{*} = \esssup_{x \in \mathcal{X}} \card{c(x)}$.
\end{definition}
Remark that as the two inputs $P$ and $Q$ become "closer" in some sense to one another, it is harder to satisfy the WLA. However, this is not a problem as whenever this happens, we shall have successfully learned $P$ through $Q$. The classical theory of boosting would just assume one constraint over a distribution $M$ whose marginals over classes would be $P$ and $Q$ \citep{kTO}, but our definition can in fact easily be shown to coincide with that of boosting \citep{cnBD}. A \textit{boosting} algorithm is an algorithm which has only access to a weak learner and, throughout repeated calls, typically combines a sufficient number of weak classifiers to end up with a combination arbitrarily more accurate than its parts.\\[0.2\baselineskip]
\noindent $\triangleright$ Differential privacy, intregral privacy: we introduce a user-defined parameter, $\epsilon$, which represents a privacy budget; $\epsilon > 0$ and the smaller it is, the stronger the privacy demand. Hereafter, $D$ and $D'$ denote input datasets from $\mathcal{X}$, and $D \approx D'$ denotes the predicate that $D$ and $D'$ differ by one observation.  Let $\mathcal{A}$ denote a randomized algorithm that takes as input datasets and outputs samples from $\mathcal{X}$. 
\begin{definition}\label{basicDef}
For any fixed $\epsilon>0$, $\mathcal{A}$ is said to meet $\epsilon$-\textbf{differential privacy} (DP) iff
\begin{eqnarray} \text{Pr}[\mathcal{A}(D) \hspace{-0.05cm} \in \hspace{-0.05cm}  S] \hspace{-0.05cm} \leq \hspace{-0.05cm} \exp(\epsilon) \hspace{-0.1cm} \cdot \hspace{-0.1cm} \text{Pr}[\mathcal{A}(D') \hspace{-0.05cm} \in \hspace{-0.05cm}  S], \hspace{-0.1cm} \forall S \hspace{-0.1cm} \subseteq \hspace{-0.1cm} \mathcal{X}, \hspace{-0.1cm} \forall D \hspace{-0.1cm} \approx \hspace{-0.1cm}  D', \mbox{where $\approx$ means differ by 1 observation}.\label{eqdiffpriv}
\end{eqnarray}
$\mathcal{A}$ meets $\epsilon$-\textbf{integral privacy} (IP) iff \eqref{eqdiffpriv} holds when replacing $\forall D \approx D'$ by a general $\forall D, D'$.
\end{definition} 
Note that by removing the $D \approx D'$ constraint, we also remove any size constraint on $D$ and $D'$. 

\noindent $\triangleright$ \textbf{Mollifiers}. We now introduce a property for sets of densities that shall be crucial for privacy. 
\begin{definition}\label{defMOL}
Let $\mathcal{M}$ be a set of densities with the same support, $\varepsilon > 0$. $\mathcal{M}$ is an \textbf{$\varepsilon$-mollifier} iff 
\begin{eqnarray}
Q(x) & \leq & \exp(\varepsilon) \cdot Q'(x), \forall Q, Q' \in \mathcal{M}, \forall x\in \mathcal{X}\label{ipriv}.
\end{eqnarray}
\end{definition}
Before stating how we can simply transform any set of densities with finite range into an $\varepsilon$-mollifier, let us show why such sets are important for integrally private sampling. 
\begin{lemma}\label{molpriv}
Let $\mathcal{A}$ be a sampler for densities within an $\varepsilon$-mollifier $\mathcal{M}$. Then $\mathcal{A}$ is $\varepsilon$-integrally private. 
\end{lemma}
(Proof in \supplement, Section \ref{proof_molpriv}) Notice that we do not need to require that $D$ and $D'$ be sampled from the same density $P$. This trick which essentially allows to get "privacy for free" using the fact that sampling carries out the necessary randomization we need to get privacy, is not new: a similar, more specific trick was designed for Bayesian learning in \cite{wfsPF} and in fact the first statement of \citep[Theorem 1]{wfsPF} implements in disguise a specific $\varepsilon$-mollifier related to one we use, $\mathcal{M}_\varepsilon$ (see below). We now show examples of mollifiers and properties they can bear.

\begin{figure}[t]
\centering
\begin{tabular}{cc}
\includegraphics[width=.40\linewidth]{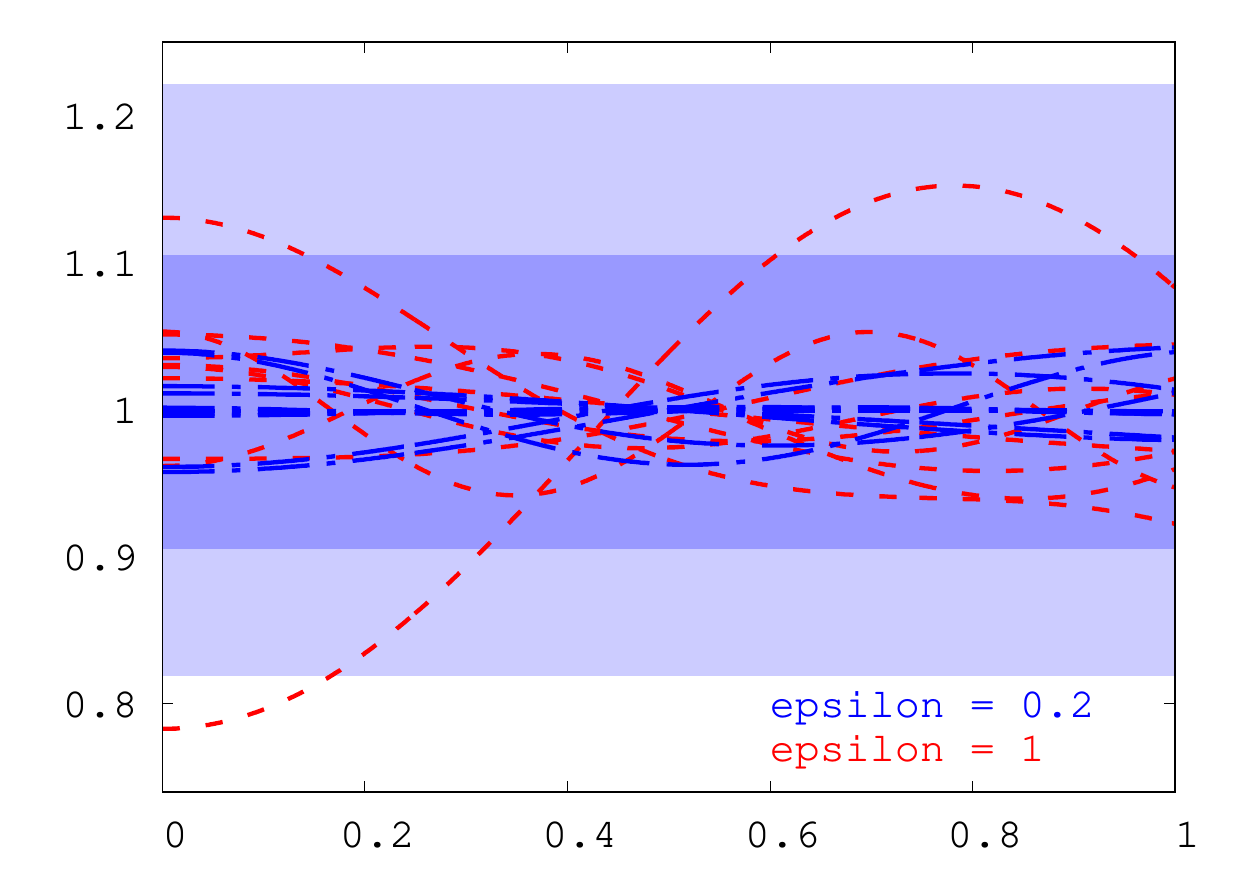} &
\includegraphics[trim=80bp 410bp 480bp
50bp,clip,width=.40\linewidth]{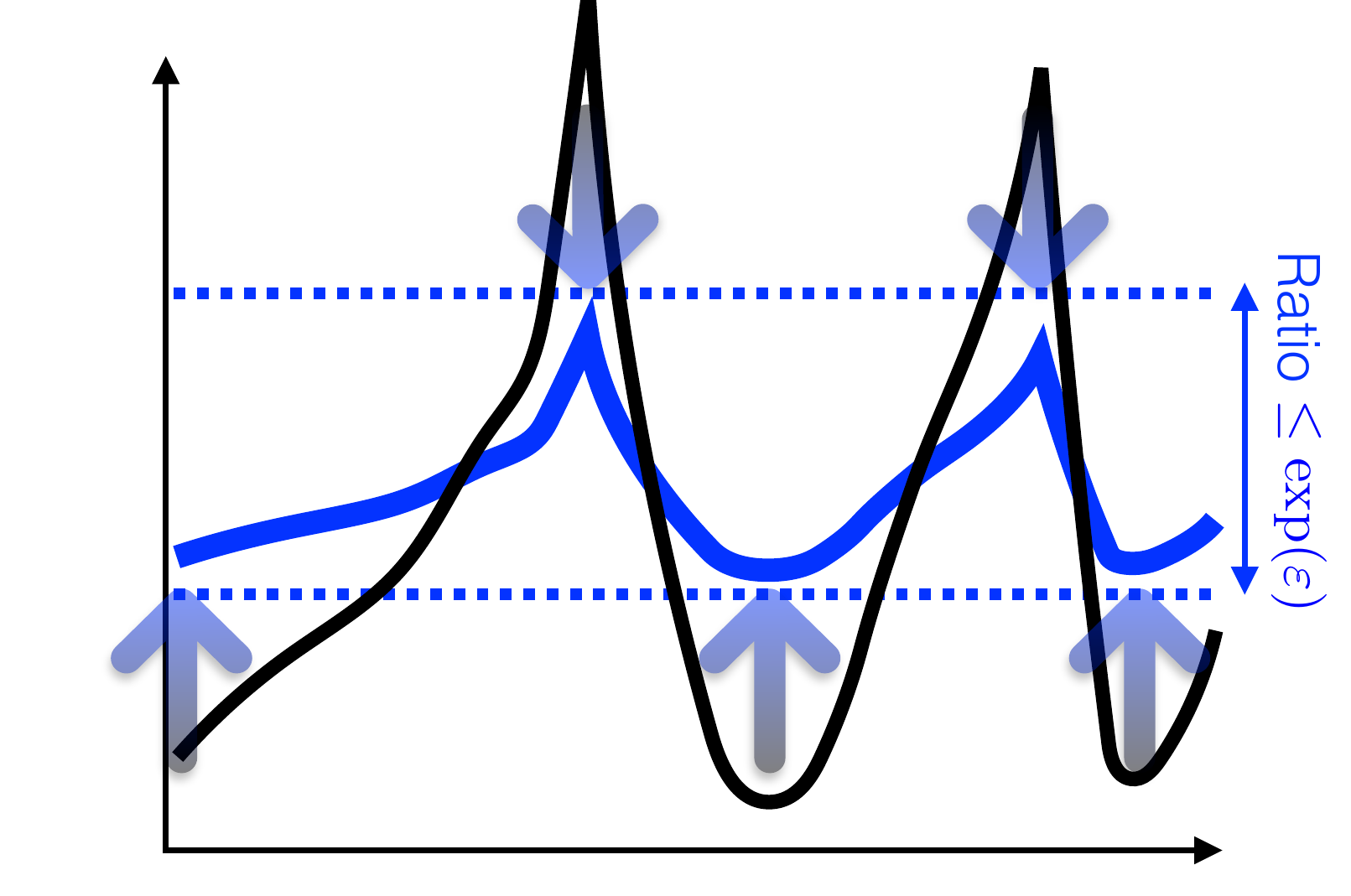}
\end{tabular}
\caption{Left: example of mollifiers for two values of $\varepsilon$, $\epsilon = 1$ (red curves) or $\epsilon = 0.2$ (blue curves), with $\mathcal{X} = [0,1]$. For that latter case, we also indicate in light blue the \textit{necessary} range of values to satisfy \eqref{ipriv}, and in dark blue a \textit{sufficient} range that allows to satisfy \eqref{ipriv}. Right: schematic depiction of how one can transform any set of finite densities in an $\varepsilon$-mollifier without losing the modes and keeping derivatives up to a positive constant scaling.\label{fig:Dampe}}
\vspace{-0.4cm}
\end{figure}

Our examples are featured in the simple case where the support of the mollifier is $[0,1]$ and densities have finite range and are continuous: see Figure \ref{fig:Dampe} (left). The two ranges indicated are featured to depict necessary or sufficient conditions on the overall range of a set of densities to be a mollifier. For the necessary part, we note that any continuous density must have $1$ in its range of values (otherwise its total mass cannot be unit), so if it belongs to an $\varepsilon$-mollifier, its maximal value cannot be $\geq \exp(\varepsilon)$ and its minimal value cannot be $\leq \exp(-\varepsilon)$. We end up with the range in light blue, in which any $\varepsilon$-mollifier has to fit. For the sufficiency part, we indicate in dark blue a possible range of values, $[\exp(-\varepsilon/2), \exp(\varepsilon/2)]$, which gives a sufficient condition for the range of all elements in a set $\mathcal{M}$ for this set to be an $\varepsilon$-mollifier\footnote{We have indeed $Q(x)/Q'(x) \leq \exp(\varepsilon/2) / \exp(-\varepsilon/2) = \exp(\varepsilon), \forall Q, Q' \in \mathcal{M}_\varepsilon, \forall x\in \mathcal{X}$.}. Let us denote more formally this set as $\mathcal{M}_\varepsilon$.

Notice that as $\varepsilon \rightarrow 0$, any $\varepsilon$-mollifier converges to a singleton. In particular,
all elements of $\mathcal{M}_\varepsilon$ converge in distribution to the uniform distribution, which would also happen for sampling using standard mechanisms of differential privacy \citep{drTA}, so we do not lose qualitatively in terms of privacy. However, because we have no constraint apart from the range constraint to be in $\mathcal{M}_\varepsilon$, this freedom is going to be instrumental to get guaranteed approximations of $P$ via the boosting theory. Figure \ref{fig:Dampe} (right) also shows how a simple scale-and-shift procedure allows to fit any finite density in $\mathcal{M}_\varepsilon$ while keeping some of its key properties, so "mollifying" a finite density in $\mathcal{M}_\varepsilon$ in this way do not change its modes, which is an important property for sampling, and just scales its gradients by a positive constant, which is is an important property for learning and optimization.
\section{Mollifier density estimation with approximation guarantees}\label{sec:algo}
The cornerstone of our approach is an algorithm that (i) learns an explicit density in an $\varepsilon$-mollifier and (ii) with approximation guarantees with respect to the target $P$. This algorithm, \pkde, for Mollified Boosted Density Estimation, is depicted below.

\begin{algorithm}[t]
\caption{\pkde ($\textsc{wl}, T,\epsilon,Q_0$)}\label{mainalg}
\begin{algorithmic}[1]
\STATE \textbf{input}: Weak learner \textsc{wl}, $\#$ boosting iterations $T$, privacy parameter $\varepsilon$,\\\hspace{1cm}initial integrally private distribution $Q_0$, non-private target $P$;
\FOR{$t = 1,\ldots,T$}
	\STATE $\theta_t(\epsilon) \gets \bracket{\frac{\epsilon}{\epsilon + 4\log(2)}}^t$
	\STATE $c_t \gets \textsc{wl} (P,Q_t)$
    \STATE $Q_t \propto Q_{t-1} \cdot \exp(\theta (\epsilon) \cdot c_t)$
\ENDFOR
\STATE \textbf{return}: $Q_T$
\end{algorithmic}
\end{algorithm}

\pkde~is a private refinement of the \textsc{Discrim} algorithm of \cite[Section 3]{cnBD}. It uses a weak learner whose objective is to distinguish between the target $P$ and the current guessed density $Q_t$ --- the index indicates the iterative nature of the algorithm. $Q_t$ is progressively refined using the weak learner's output classifier $c_t$, for a total number of user-fixed iterations $T$.
We start boosting by setting $Q_0$ as the starting distribution, typically a simple non-informed (to be private) distribution such as a standard Gaussian (see also Figure \ref{fig:Princ}, center). The classifier is then aggregated into $Q_{t-1}$ as:
\begin{eqnarray}
Q_t = \frac{\exp(\theta_t (\epsilon) c_t) Q_{t-1}}{\int \exp(\theta_t (\epsilon) c_t) Q_{t-1} dx} = \exp\bracket{\ip{\theta(\epsilon)}{c} - \phi(\theta(\epsilon))} Q_0. \label{defEXP1}
\end{eqnarray}
where $\theta(\epsilon) = (\theta_1 (\epsilon),\ldots,\theta_t (\epsilon))$, $c = (c_1,\ldots,c_t)$ (from now on, $c$ denotes the vector of all classifiers) and $\phi(\theta(\epsilon))$ is the log-normalizer given by
\begin{align} \phi(\theta(\epsilon)) = \log \int_{\mathcal{X}} \exp \bracket{\ip{\theta(\epsilon)}{c}}dQ_0. \end{align}
This process repeats until $t = T$ and the proposed distribution is $Q_{\epsilon}(x;D) = Q_T$. It is not hard to see that $Q_{\epsilon}(x;D)$ is an exponential family with natural parameter $\theta(\epsilon)$, sufficient statistics $c$, and base measure $Q_0$ \citep{anMO,cnBD}. We now show three formal results on \pkde.\\[0.2\baselineskip]
\noindent $\triangleright$ \textbf{Sampling from $Q_T$ is $\varepsilon$-integrally private} --- Recall $\mathcal{M}_\varepsilon$ is the set of densities whose range is in $\exp [-\varepsilon/2, \varepsilon/2]$. We now show that the output $Q_T$ of \pkde~is in $\mathcal{M}_\varepsilon$, guaranteeing $\epsilon$-integral privacy on sampling (Lemma \ref{molpriv}). 
\begin{theorem}
\label{privacy-theorem}
$Q_T \in \mathcal{M}_\varepsilon$.
\end{theorem}
(Proof in \supplement, Section \ref{proof_privacy-theorem}) 
We observe that privacy comes with a price, as for example $\lim_{\epsilon \rightarrow 0} \theta_t(\epsilon) = 0$, so as we become more private, the updates on $Q_.$ become less and less significant and we somehow flatten the learned density --- as already underlined in Section \ref{sec:def}, such a phenomenon is not a particularity of our method as it would also be observed for standard DP mechanisms \citep{drTA}.\\[0.2\baselineskip]
\noindent $\triangleright$ \textbf{\pkde~approximates the target distribution in the boosting framework} --- As explained in Section \ref{sec:def}, it is not hard to fit a density in $\mathcal{M}_\varepsilon$ to make its sampling private. An important question is however what guarantees of approximation can we still have with respect to $P$, given that $P$ may not be in $\mathcal{M}_\varepsilon$. We now give such guarantees to \pkde~in the boosting framework, and we also show that the approximation is within close order to the best possible given the constraint to fit $Q_.$ in $\mathcal{M}_\varepsilon$. We start with the former result, and for this objective include the iteration index $t$ in the notations from Definition \ref{def-wla} since the actual weak learning guarantees may differ amongst iterations, even when they are still within the prescribed bounds (as \textit{e.g.} for $c_t$).
\begin{theorem}
\label{kl-upper}
For any $t\geq 1$, suppose \textsc{wl} satisfies at iteration $t$ the WLA for $\gamma_P^{t}, \gamma_Q^{t}$. Then we have:
\begin{eqnarray}
\text{KL}(P,Q_t) & \leq &\text{KL}(P,Q_{t-1}) - \theta_t(\epsilon) \cdot \Lambda_t,\label{bKLdrop}
\end{eqnarray} where (letting $\Gamma(z) \defeq \log(4/(5 - 3z))$):
\begin{eqnarray}
\Lambda_t \hspace{-0.1cm}= \hspace{-0.1cm}\left\{\begin{array}{ccl}
  \hspace{-0.2cm}  c_t^{*} \gamma_P^{t} + \Gamma(\gamma_Q^{t})    \hspace{-0.3cm} &  \hspace{-0.3cm} \text{if}  \hspace{-0.3cm} &  \hspace{-0.2cm} \gamma_Q^{t} \in [1/3,1]  \mbox{ ("high boosting regime")}\\  
  \hspace{-0.2cm} \gamma_P^{t} + \gamma_Q^{t} - \frac{c_t^{*} \cdot \theta_t(\epsilon)}{2}  \hspace{-0.3cm} &  \hspace{-0.3cm} \text{if}  \hspace{-0.3cm} &  \hspace{-0.2cm} \gamma_Q^{t} \in (0,1/3) \mbox{ ("low boosting regime")} 
\end{array}\right. . \label{consteqs}
\end{eqnarray}
\end{theorem}
(Proof in \supplement, Section \ref{proof_kl-upper}) Remark that in the \textit{high} boosting regime, we are guaranteed that $\Lambda_t\geq 0$ so the bound on the KL divergence is guaranteed to decrease. This is a regime we are more likely to encounter during the first boosting iterations since $Q_{t-1}$ and $P$ are then easier to tell apart --- we can thus expect a larger $\gamma_Q^{t}$. In the low boosting regime, the picture can be different since we need $\gamma_P^{t} + \gamma_Q^{t} \geq c_t^{*} \cdot \theta_t(\epsilon)/2$ to make the bound not vacuous. Since $\theta_t(\epsilon) \to_{t} 0$ exponentially fast and $c_t^{*} \leq \log 2$, a constant, the constraint for \eqref{consteqs} to be non-vacuous vanishes and we can also expect the bound on the KL divergence to also decrease in the \textit{low} boosting regime.
We now check that the guarantees we get are close to the best possible in an information-theoretic sense given the two constraints: (i) $Q$ is an exponential family as in \eqref{defEXP1} and (ii) $Q \in \mathcal{M}_\varepsilon$. Let us define $\mathcal{M}^{\exp}_\varepsilon \subset \mathcal{M}_\varepsilon$ the set of such densities, where $Q_0$ is fixed, and let $\Delta(Q) \defeq \text{KL}(P,Q_0) - \text{KL}(P,Q)$. Intuitively, the farther $P$ is from $Q_0$, the farther we should be able to get from $Q_0$ to approximate $P$, and so the larger should be $\Delta(Q)$. Notice that this would typically imply to be in the high boosting regime for \pkde. For the sake of simplicity, we consider $\gamma_P, \gamma_Q$ to be the same throughout all iterations.
\begin{theorem} \label{kl-upper-lower}
We have $\Delta(Q) \leq (\epsilon/2), \forall Q \in \mathcal{M}^{\exp}_\varepsilon,$, and if \pkde~is in the high boosting regime, then 
\begin{eqnarray} \Delta(Q_T) & \geq & \frac{\epsilon}{2} \cdot \braces{\frac{\gamma_P + \gamma_Q}{2}\cdot \bracket{1 - \theta_T(\epsilon)}}. \label{bsupKL}\end{eqnarray}
\end{theorem}
(Proof in \supplement, Section \ref{proof_kl-upper-lower}) Hence, as $\gamma_P \to 1$ and $\gamma_Q \to 1$, we have  $\Delta(Q_T) \geq (\epsilon/2) \cdot (1 - \theta_T(\epsilon))$ and since $\theta_T(\epsilon) \to 0$ as $T \to \infty$, \pkde~indeed reaches the information-theoretic limit in the high boosting regime.\\[0.2\baselineskip]
\noindent $\triangleright$ \textbf{\pkde~and the capture of modes of $P$} --- Mode capture is a prominent problem in the area of generative models \citep{tgbssAB}. We have already seen that enforcing mollification can be done while keeping modes, but we would like to show that \pkde~is indeed efficient at building some $Q_T$ with guarantees on mode capture. For this objective, we define for any $B \subseteq \mathcal{X}$ and density $Q$,
\begin{eqnarray}
\textsc{m}(B,Q) \defeq  \int_B d Q \:;\:  KL(P, Q; B) \defeq \int_B \log\bracket{\frac{P}{Q}}dP,
\end{eqnarray}
respectively the total mass of $B$ on $Q$ and the KL divergence between $P$ and $Q$ restricted to $B$.
\begin{theorem}
\label{first-mode-capture}
Suppose \pkde~stays in the high boosting. 
Then $\forall \alpha \in [0, 1]$, $\forall B \subseteq \mathcal{X}$, if
\begin{eqnarray}
\textsc{m}(B,P) & \geq & \varepsilon \cdot \frac{h((2- \gamma_P - \gamma_Q) \cdot T)}{h(\alpha)\cdot h(T)}, \label{eqcondP}
\end{eqnarray}
then $\textsc{m}(B,Q_T) \geq (1 - \alpha)\textsc{m}(B,P) - KL(P, Q_0; B)$, where $h(x) \defeq \varepsilon + 2x$..
\end{theorem}
(Proof in \supplement, Section \ref{proof_mode-capture}) There is not much we can do to control $KL(P, Q_0; B)$ as this term quantifies our luck in picking $Q_0$ to approximate $P$ in $B$ but if this restricted KL divergence is small compared to the mass of $B$, then we are guaranteed to capture a substantial part of it through $Q_T$. As a mode, in particular "fat", would tend to have large mass over its region $B$, Theorem \ref{first-mode-capture} says that we can indeed hope to capture a significant part of it as long as we stay in the high boosting regime. As $\gamma_P \to 1$ and $\gamma_Q \to 1$, the condition on $\textsc{m}(B,P)$ in \eqref{eqcondP} vanishes with $T$ and we end up capturing any fat region $B$ (and therefore, modes, assuming they represent "fatter" regions) whose mass is sufficiently large with respect to $KL(P, Q_0; B)$. 

To finish up this Section, recall that $\mathcal{M}_\varepsilon$ is also defined (in disguise) and analyzed in \citep[Theorem 1]{wfsPF} for posterior sampling. However, convergence \citep[Section 3]{wfsPF} does not dig into specific forms for the likelihood of densities chosen --- as a result and eventual price to pay, it remains essentially in weak asymptotic form, and furthermore later on applied in the weaker model of $(\varepsilon, \delta)$-differential privacy. We exhibit particular choices for these mollifier densities, along with a specific training algorithm to learn them, that allow for significantly better approximation, quantitatively and qualitatively (mode capture) without even relaxing privacy.
\section{Experiments}\label{sec:exp}

\begin{figure*}
\centering
\scalebox{.76}{\begin{tabular}{cccccccc}
                 \nspp \includegraphics[width=\whe,height=\whe]{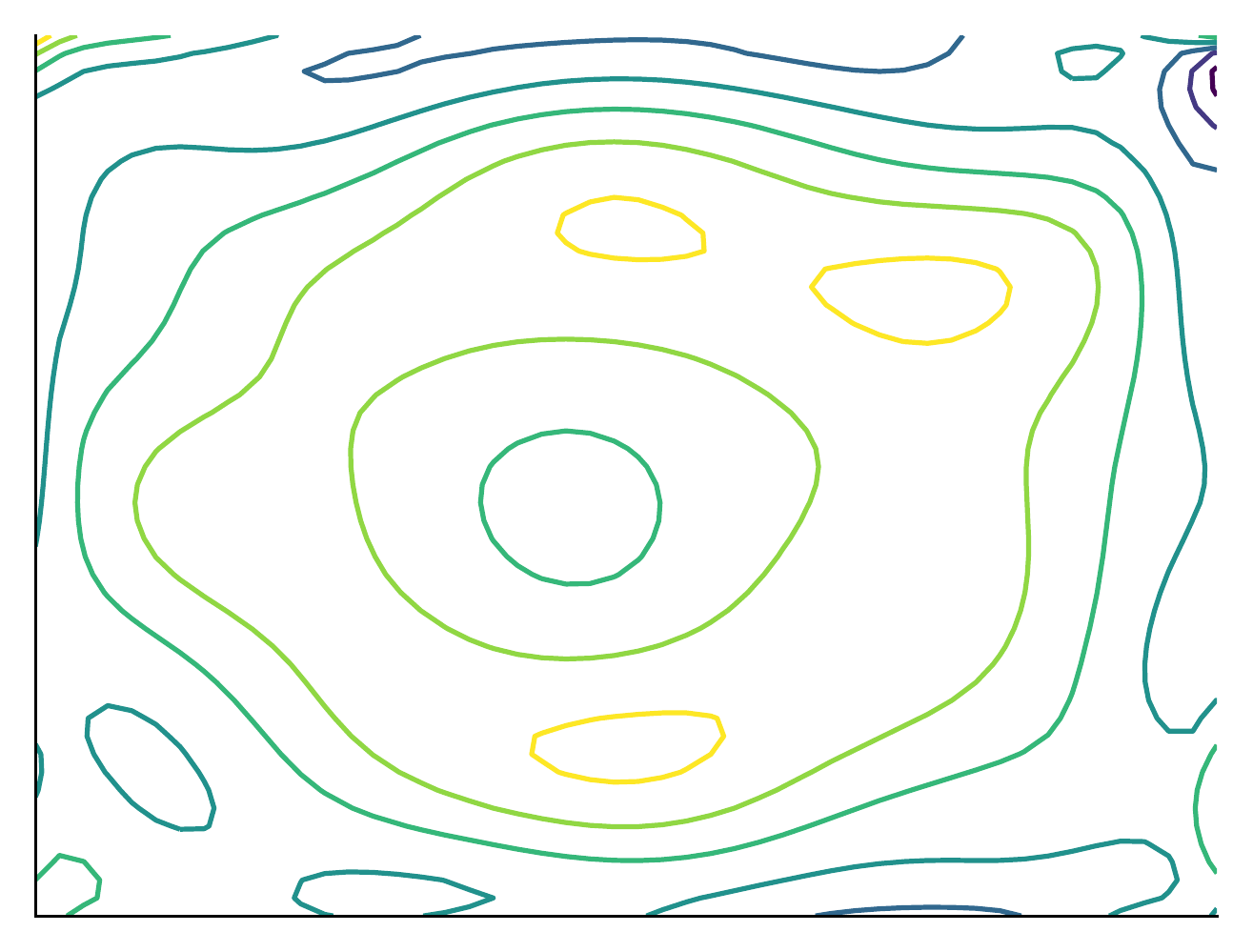}
                 \nsp & \nsp
                   \includegraphics[width=\whe,height=\whe]{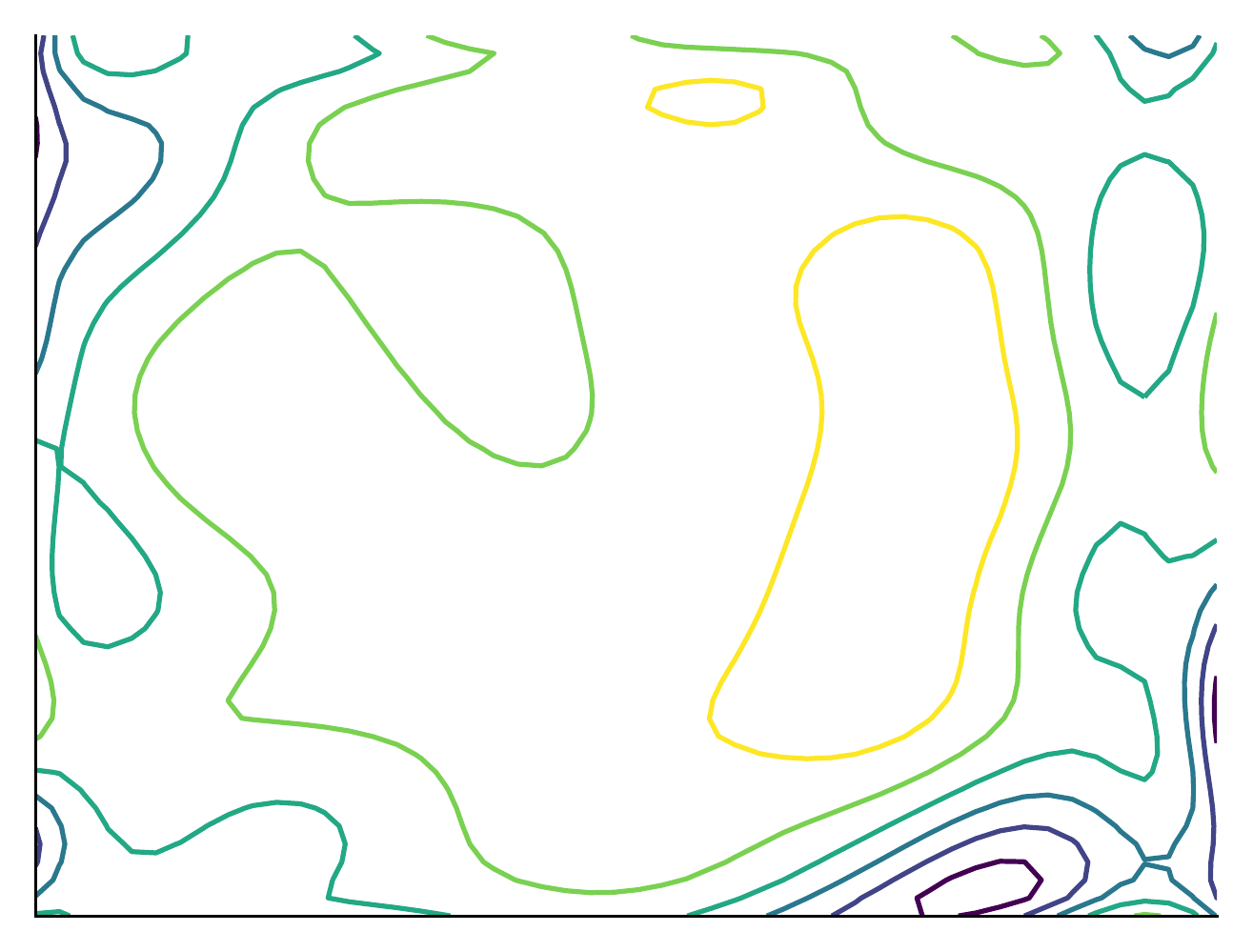}
                 \nsp & \nsp \includegraphics[width=\whe,height=\whe]{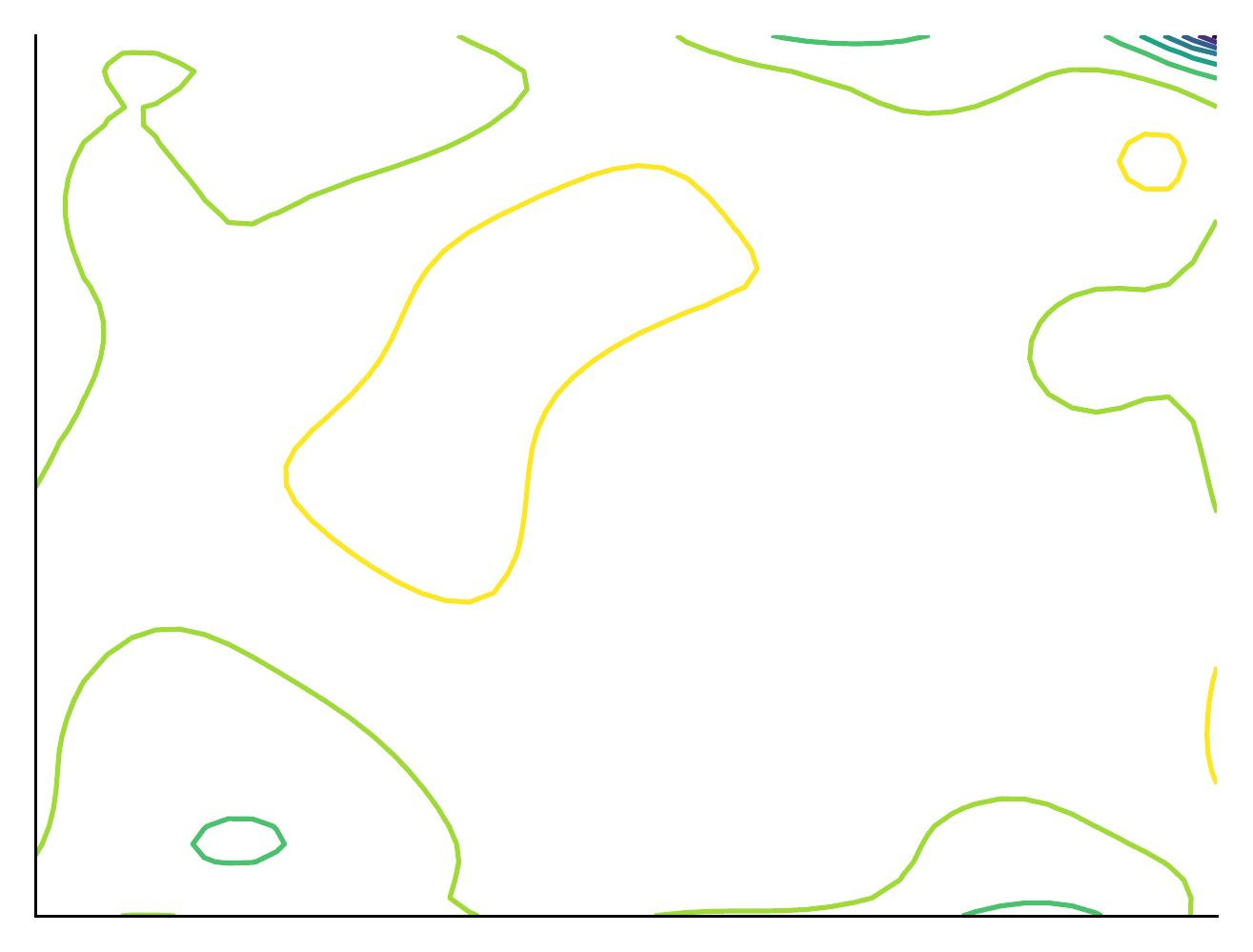}
                   \nsp & \nsp
                     \includegraphics[width=\whe,height=\whe]{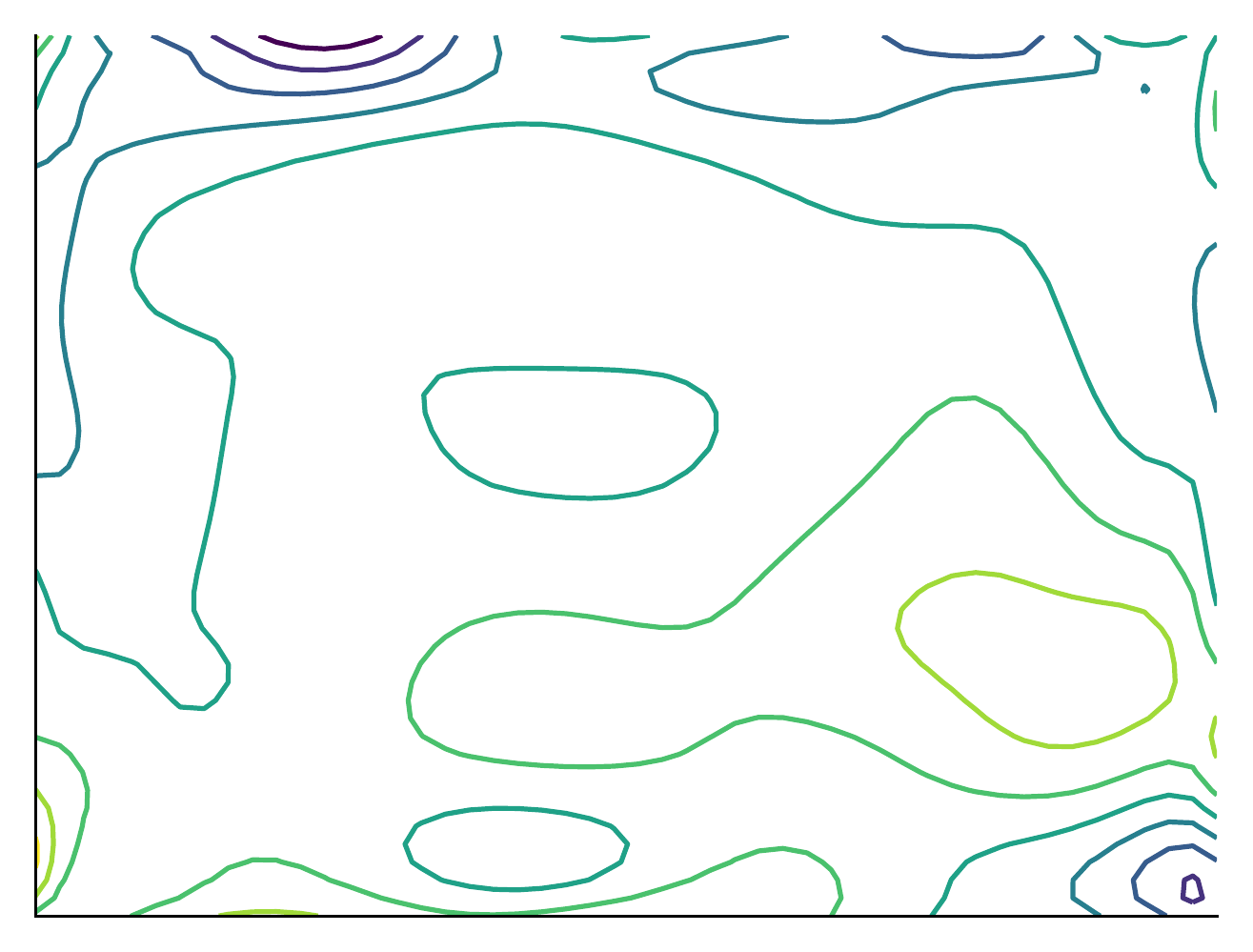}
                     \nsp & \nsp
                       \includegraphics[width=\whe,height=\whe]{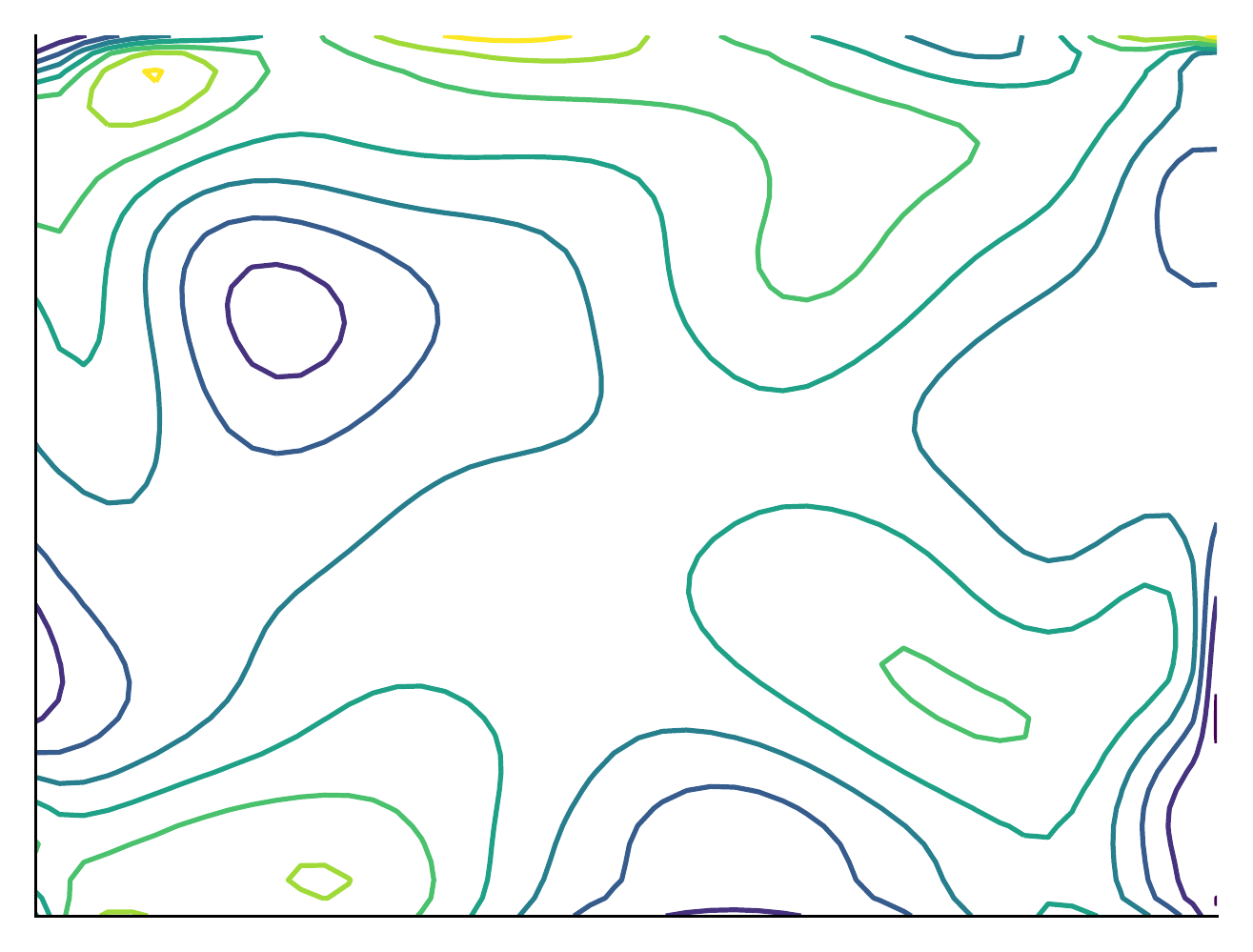}
                       \nsp & \nsp
                         \includegraphics[width=\whe,height=\whe]{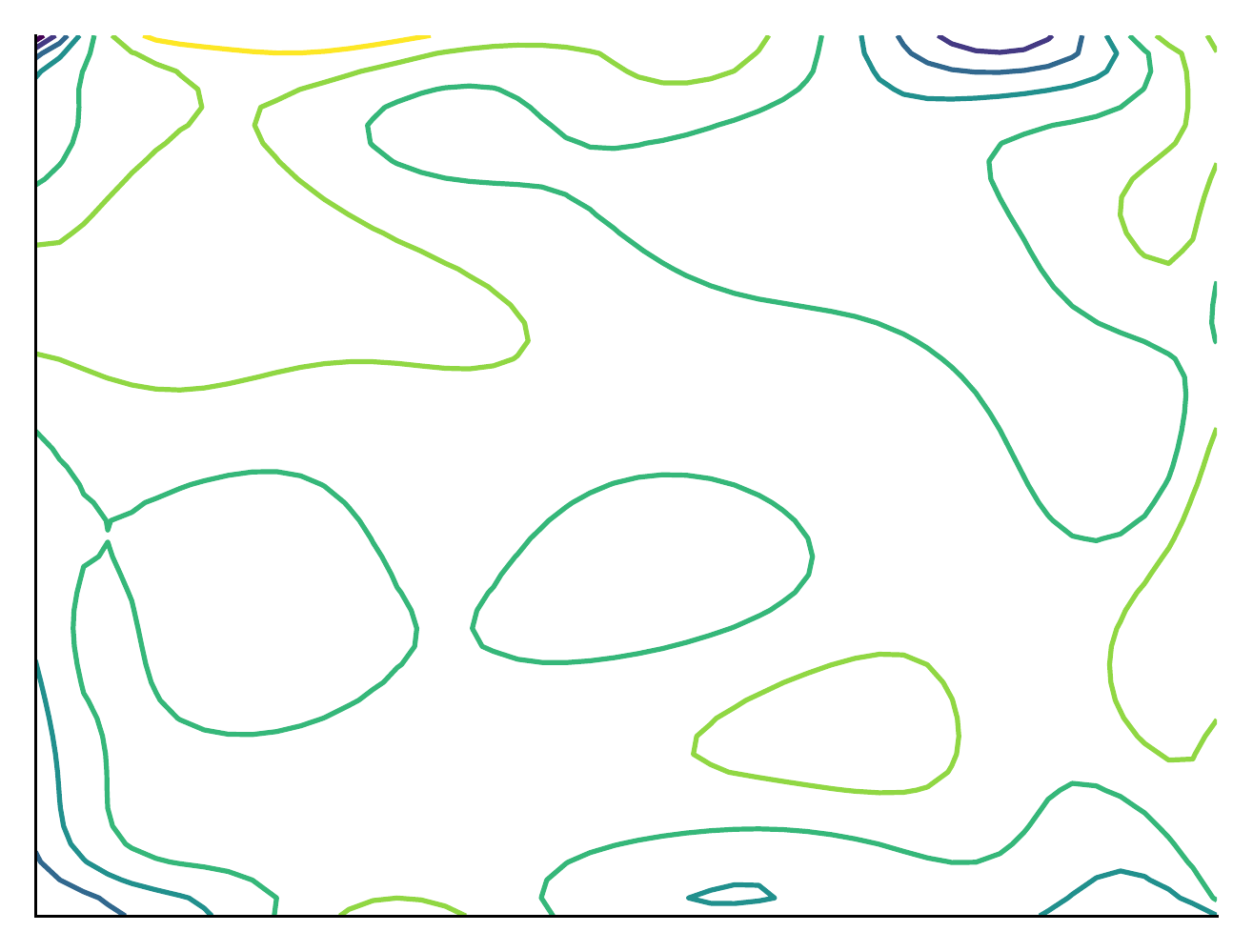}
                         \nsp & \nsp
                           \includegraphics[width=\whe,height=\whe]{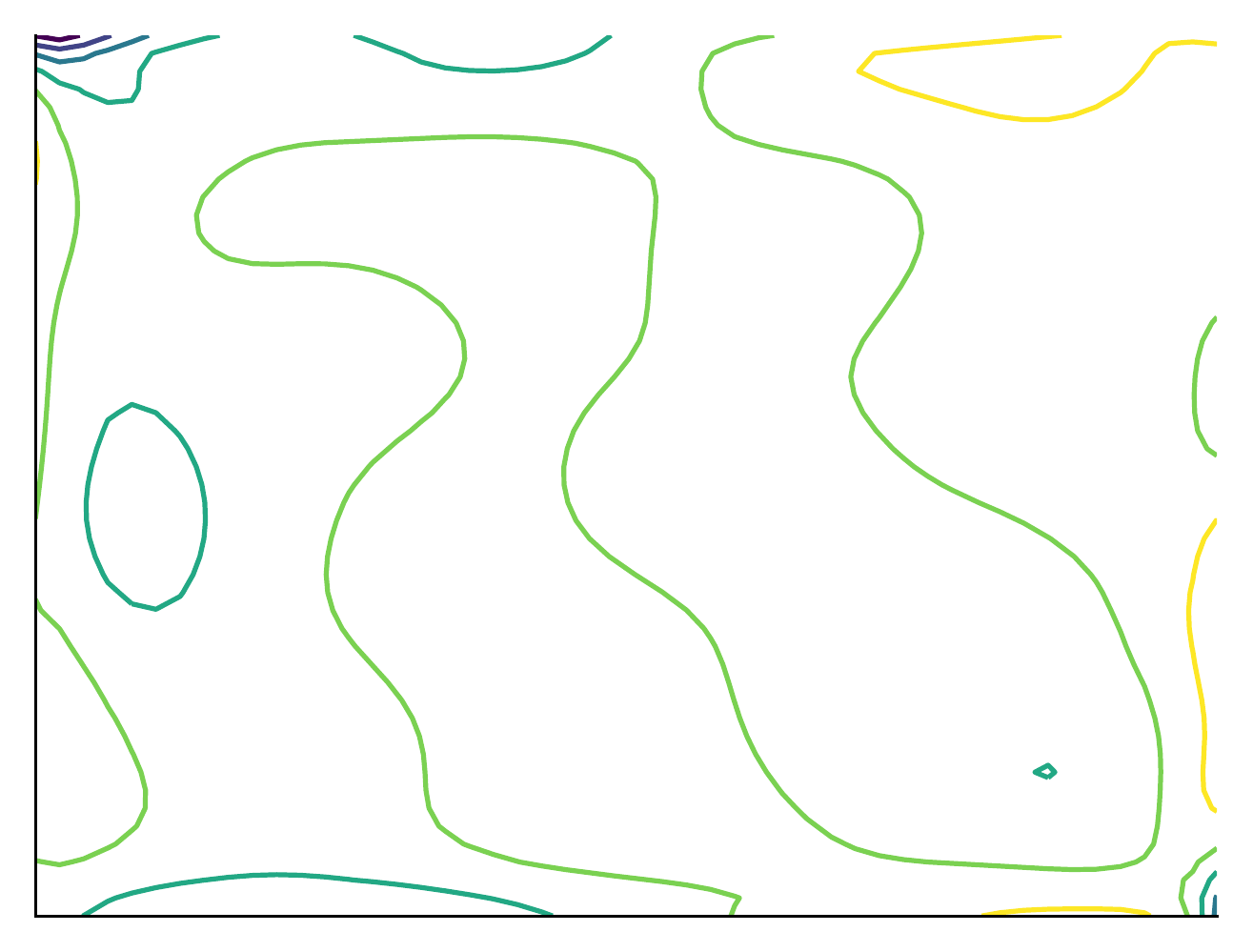}
                           \nsp & \nsp
                                  \includegraphics[width=\whe,height=\whe]{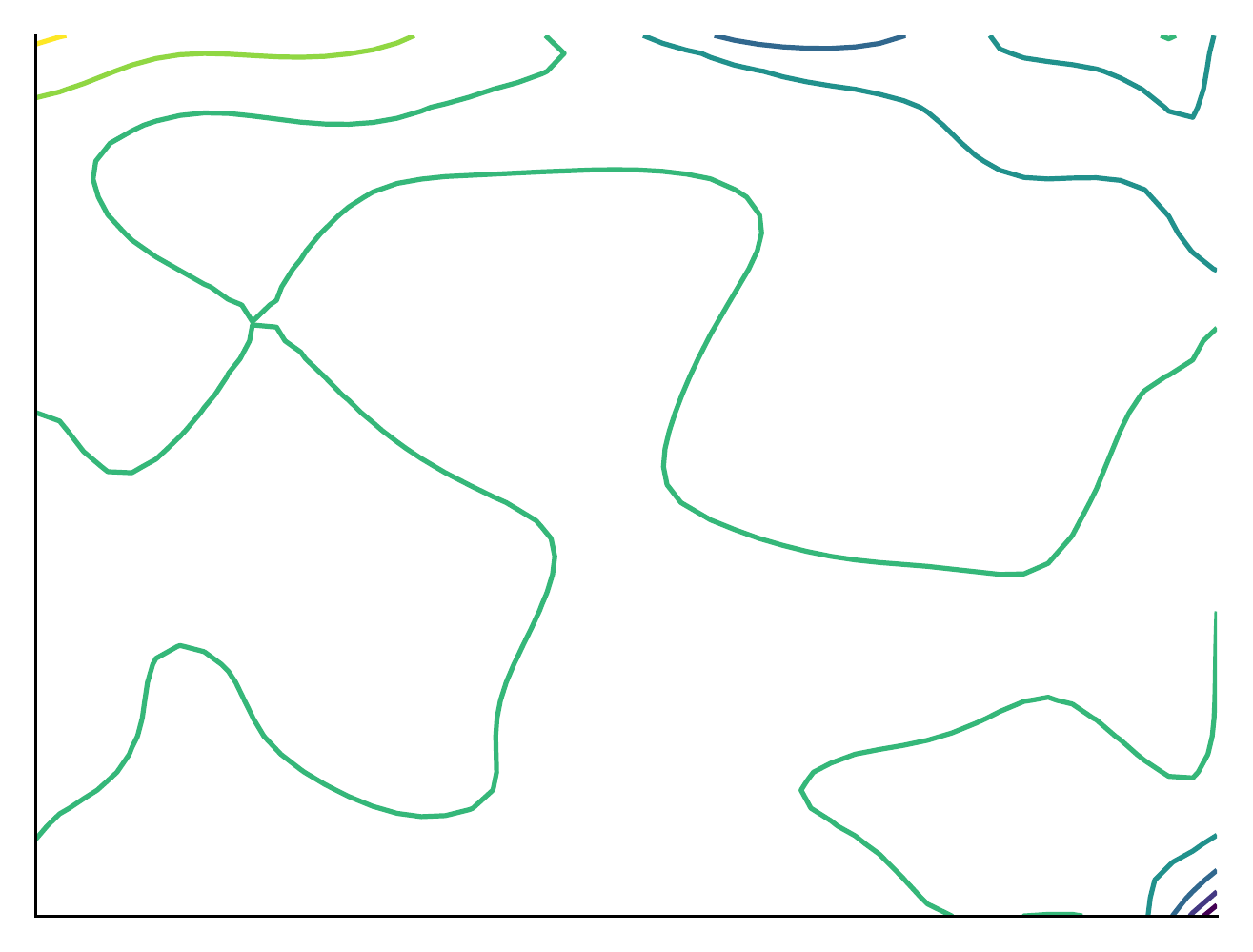}\\
                 \nspp \includegraphics[width=\whe,height=\whe]{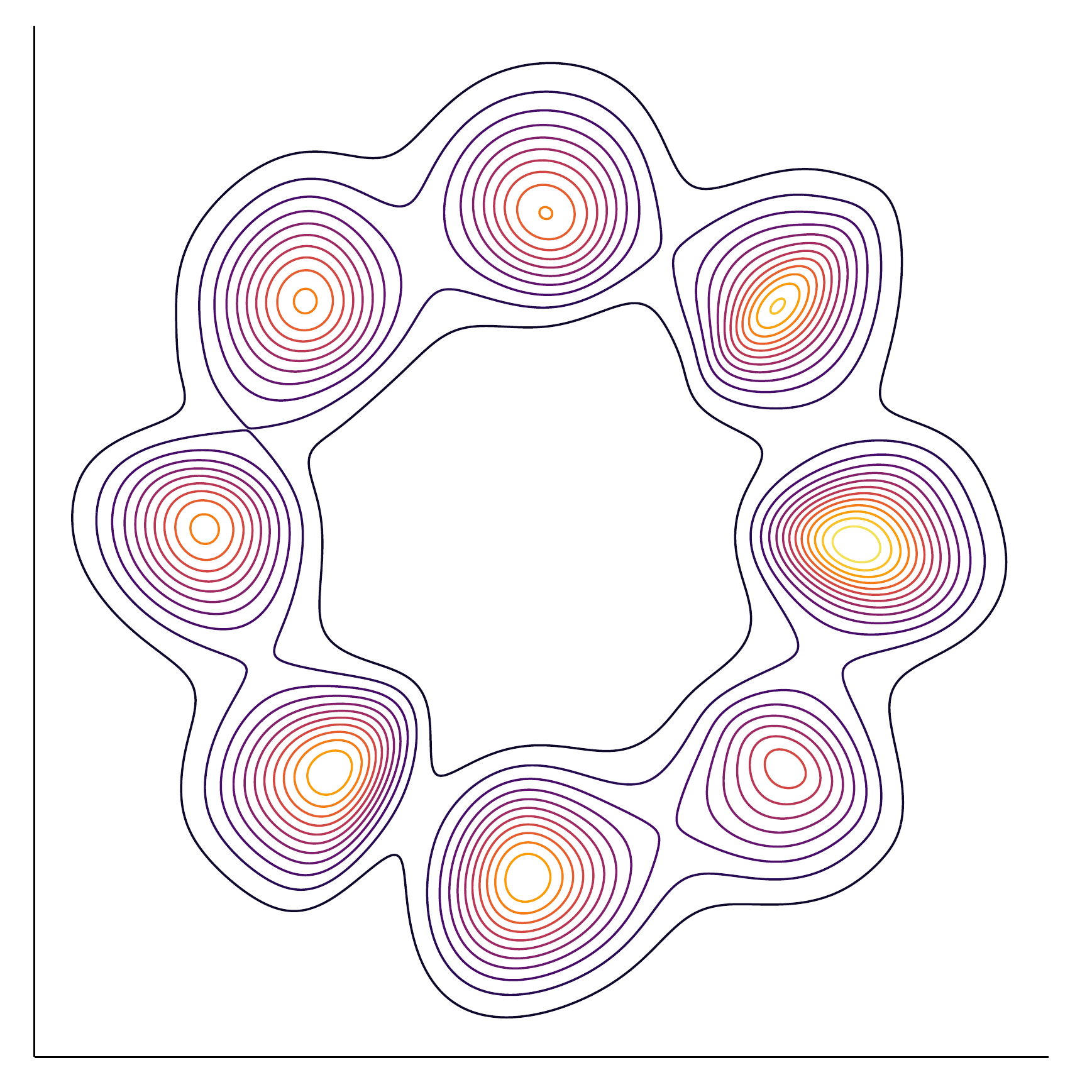}
                 \nsp & \nsp
                   \includegraphics[width=\whe,height=\whe]{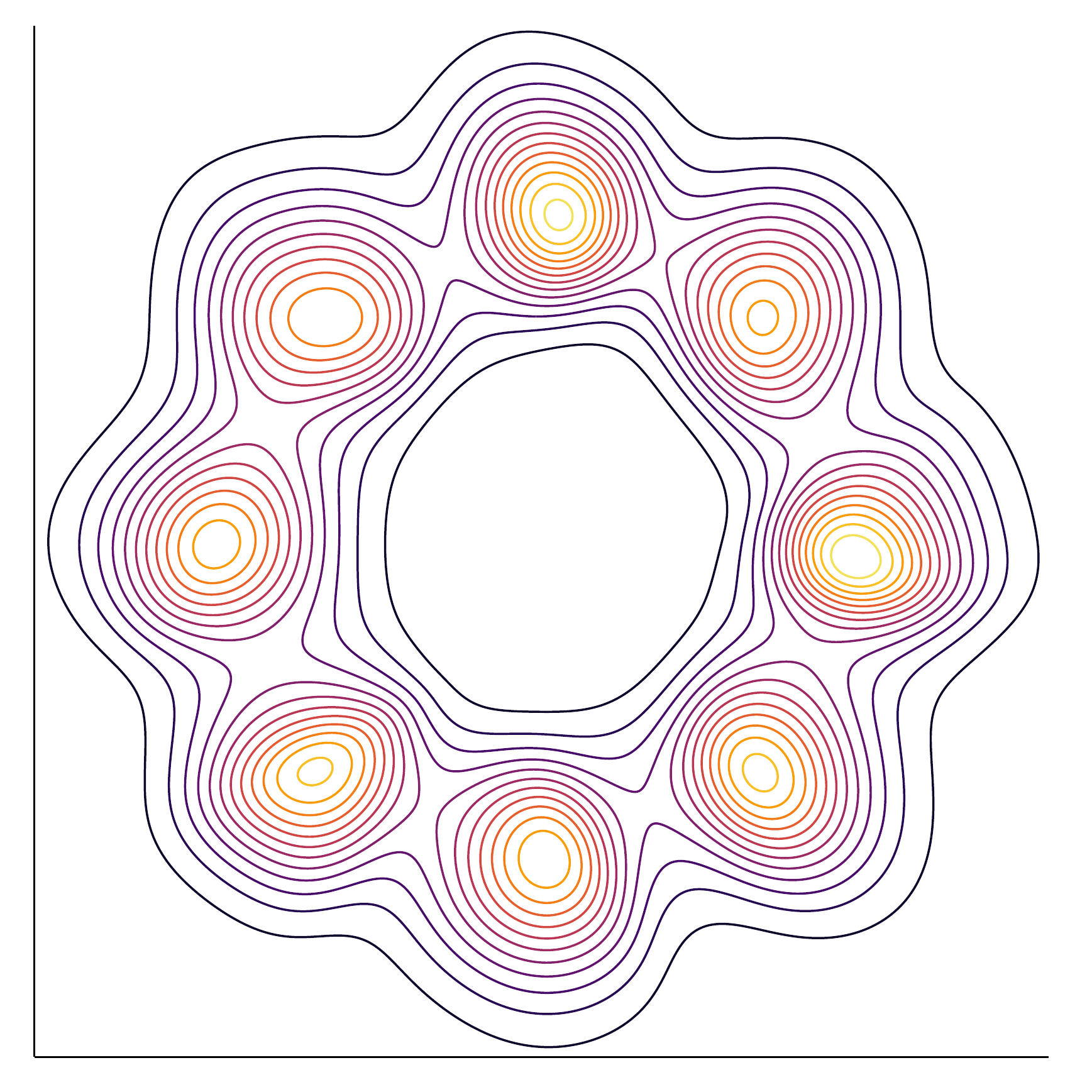}
                 \nsp & \nsp \includegraphics[width=\whe,height=\whe]{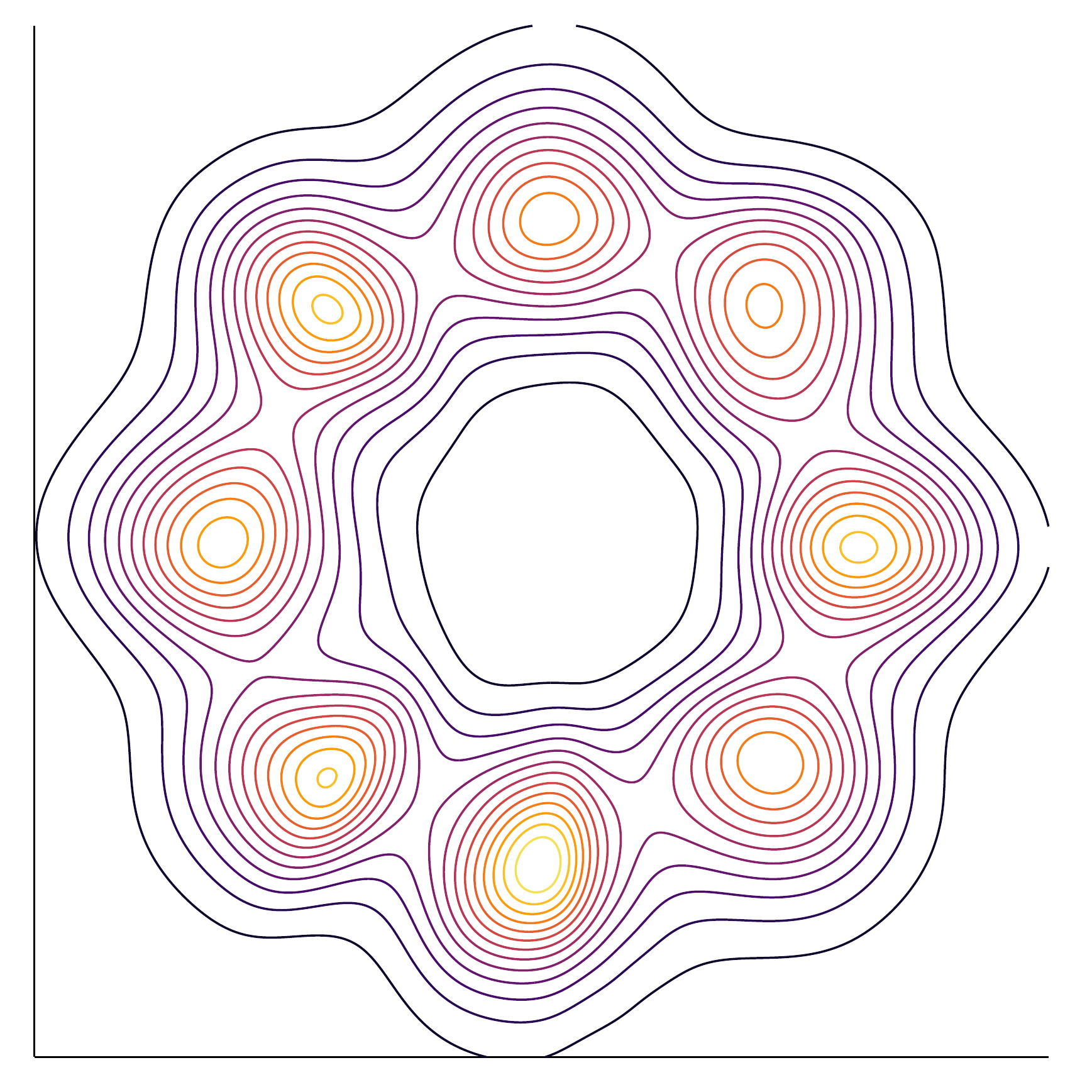}
                   \nsp & \nsp
                     \includegraphics[width=\whe,height=\whe]{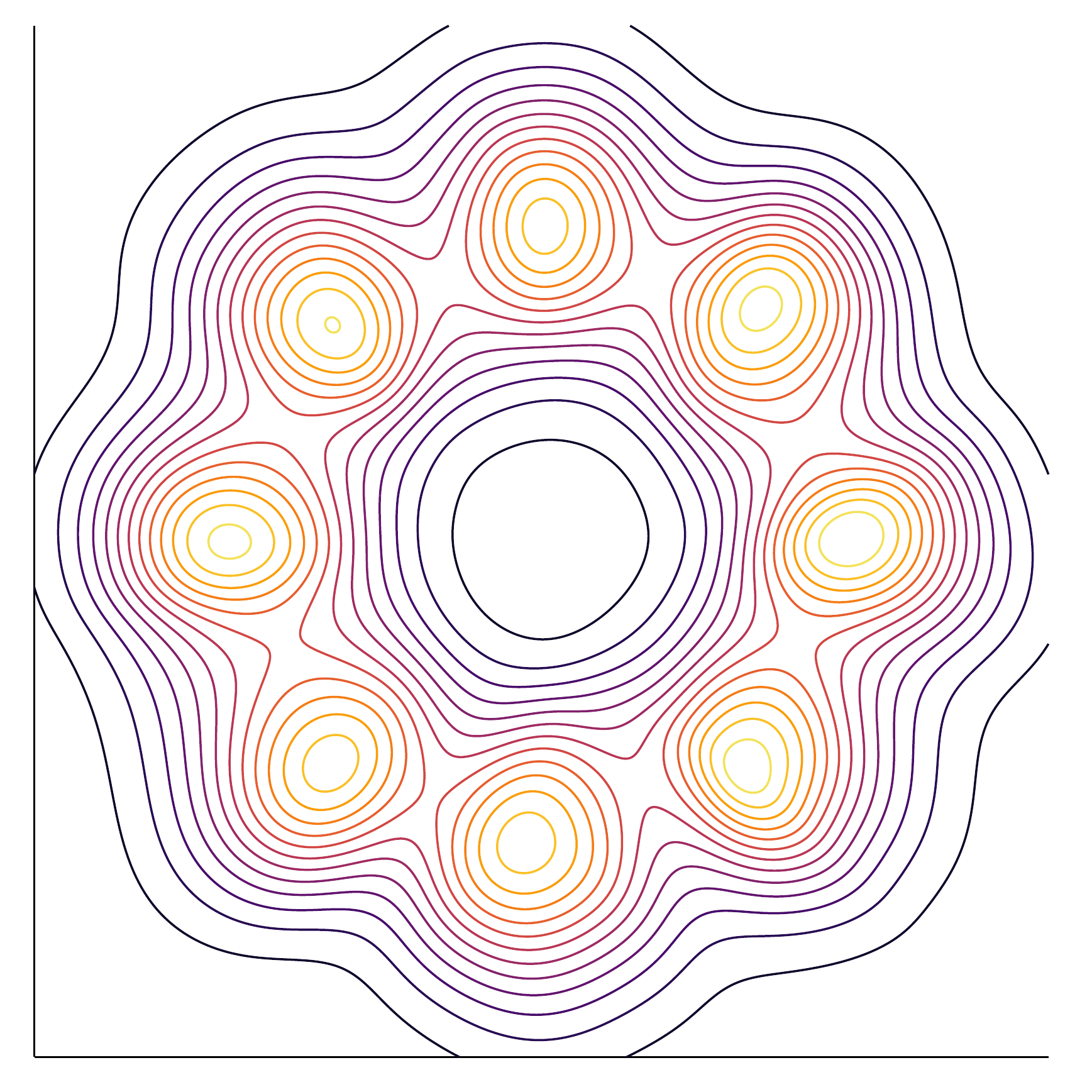}
                     \nsp & \nsp
                       \includegraphics[width=\whe,height=\whe]{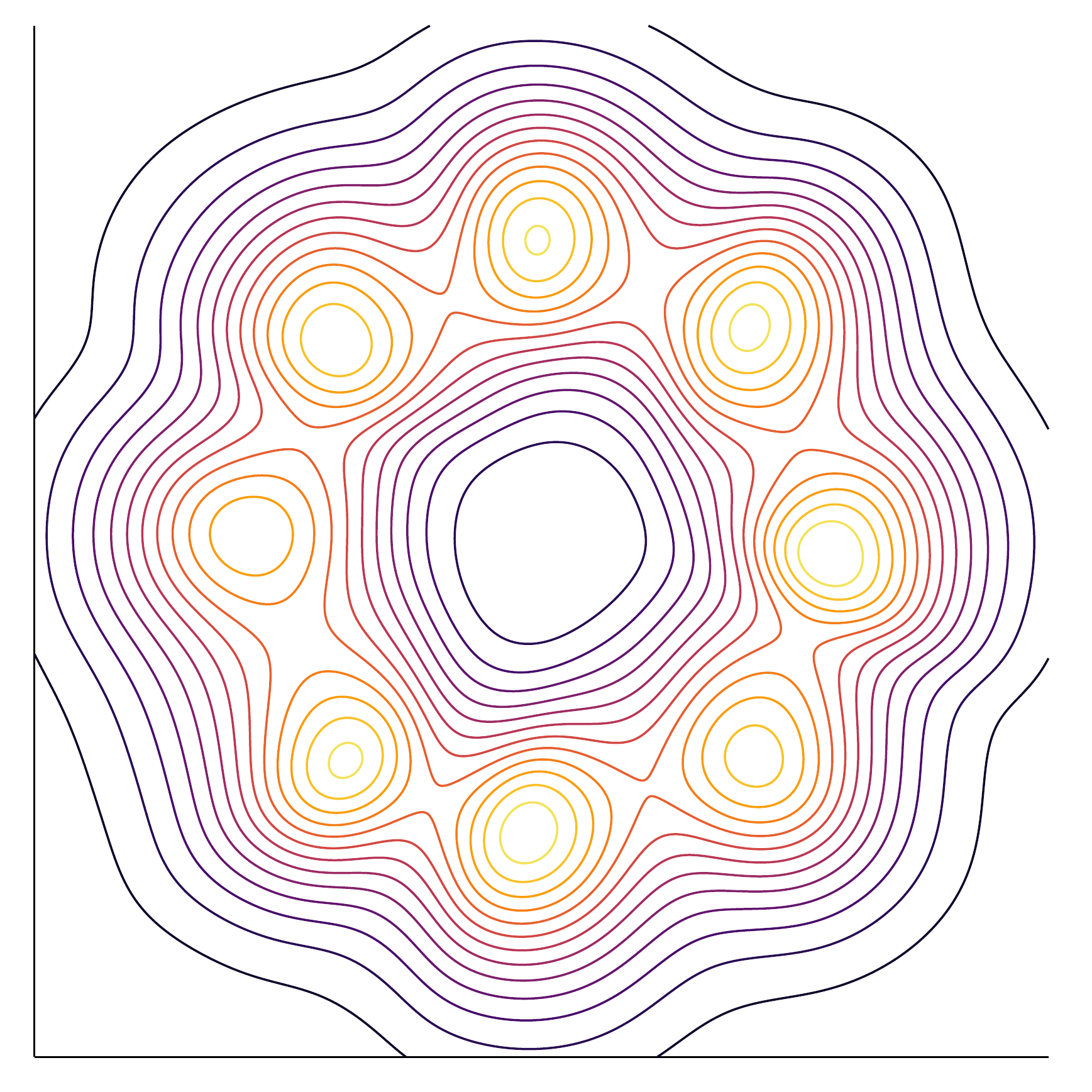}
                       \nsp & \nsp
                         \includegraphics[width=\whe,height=\whe]{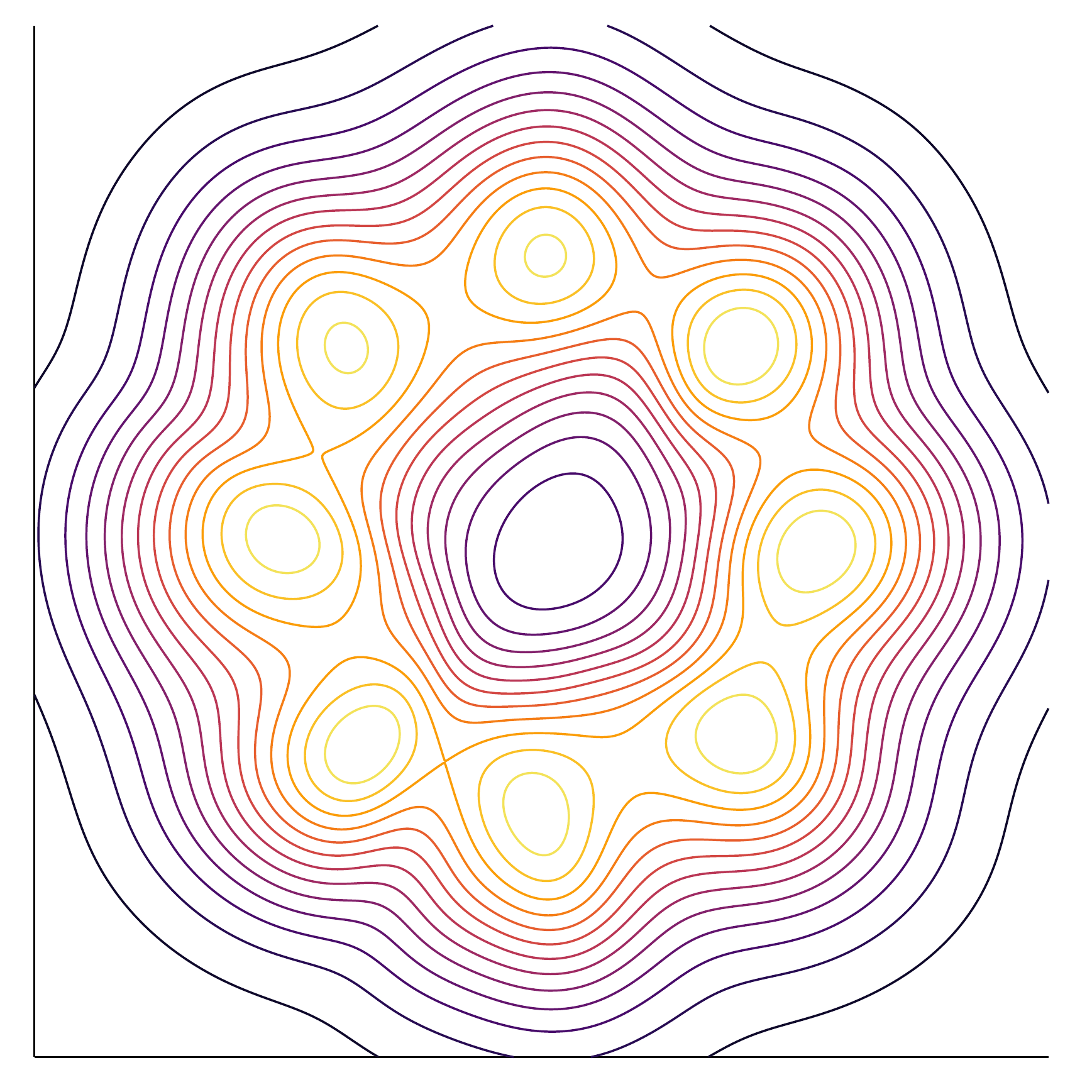}
                         \nsp & \nsp
                           \includegraphics[width=\whe,height=\whe]{Figs/ring_gaussians_eps_0_25}
                           \nsp & \nsp
                                  \includegraphics[width=\whe,height=\whe]{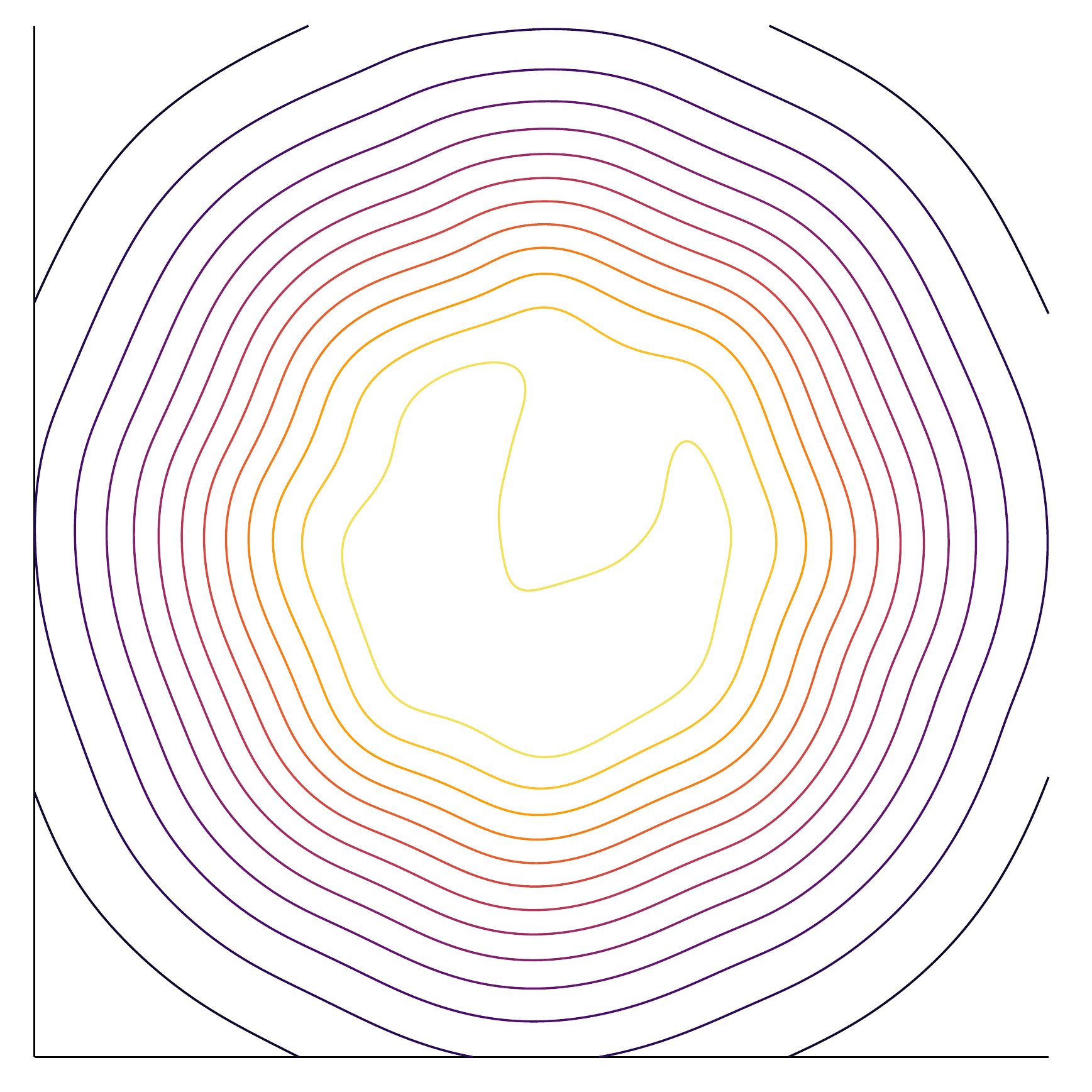}\\
                 \nspp $\epsilon = 5$
                 \nsp & \nsp
                   $\epsilon = 2$
                 \nsp & \nsp $\epsilon = 1.5$
                   \nsp & \nsp
                     $\epsilon = 1$
                     \nsp & \nsp
                       $\epsilon = 0.75$
                       \nsp & \nsp
                         $\epsilon = 0.5$
                         \nsp & \nsp
                           $\epsilon = 0.25$
                           \nsp & \nsp
                                  $\epsilon = 0.1$
                 \end{tabular}
}
\caption{Gaussian ring: densities obtained for DPB (upper row) against
  \pkde~(lower row)}
\label{ring:DPB-vs-proposed}
\end{figure*}

\begin{figure*}
\centering
\begin{tabular}{cc||cc}\\ \hline \hline
\multicolumn{2}{c||}{Gaussian ring} & \multicolumn{2}{c}{1D non random
                                      Gaussian} 
  \\ 
\hspace{\whep} \includegraphics[trim=20bp 20bp 12bp
20bp,clip,width=\whed]{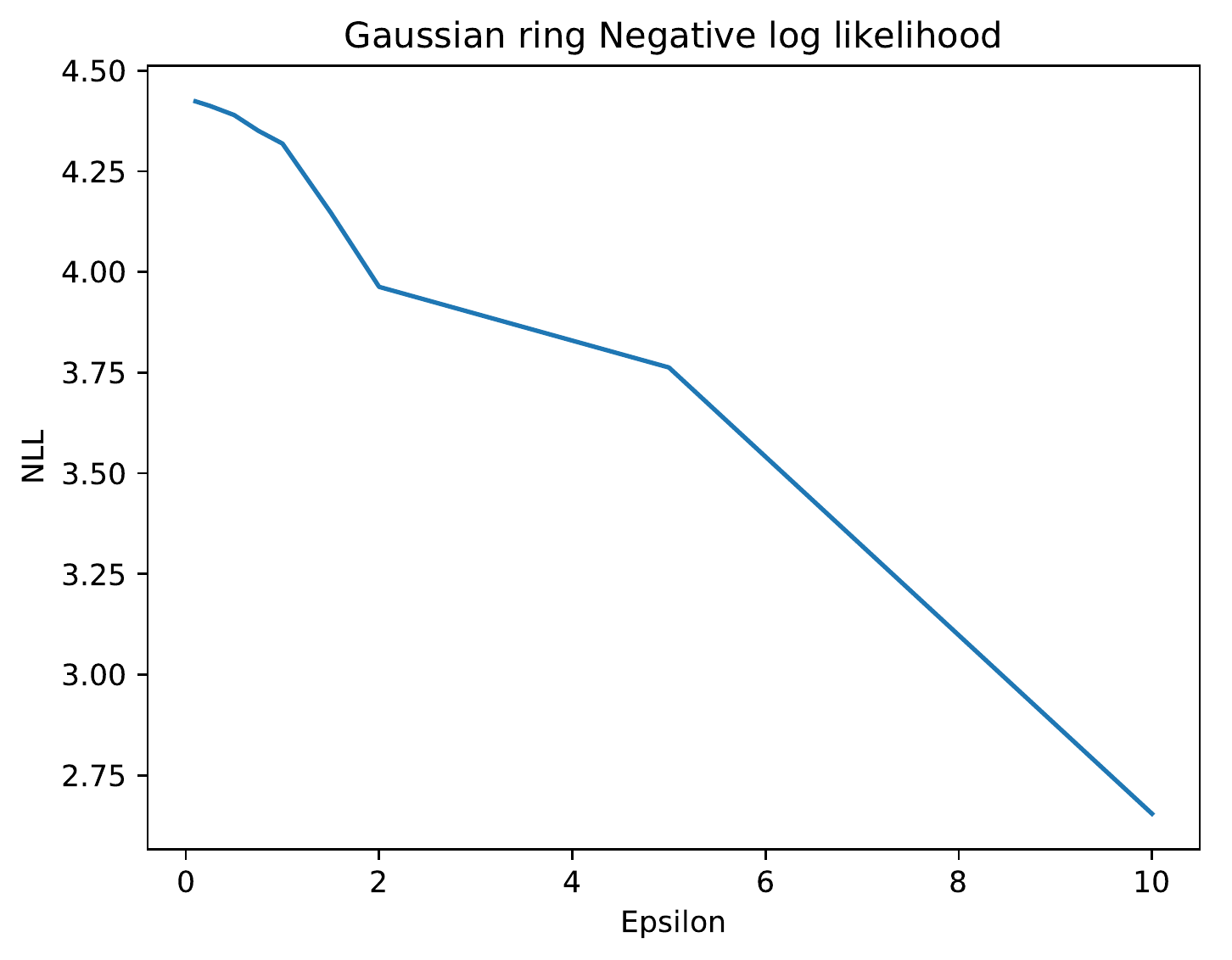}
\hspace{\whep} & \hspace{\whep} \includegraphics[trim=20bp 20bp 20bp
20bp,clip,width=\whed]{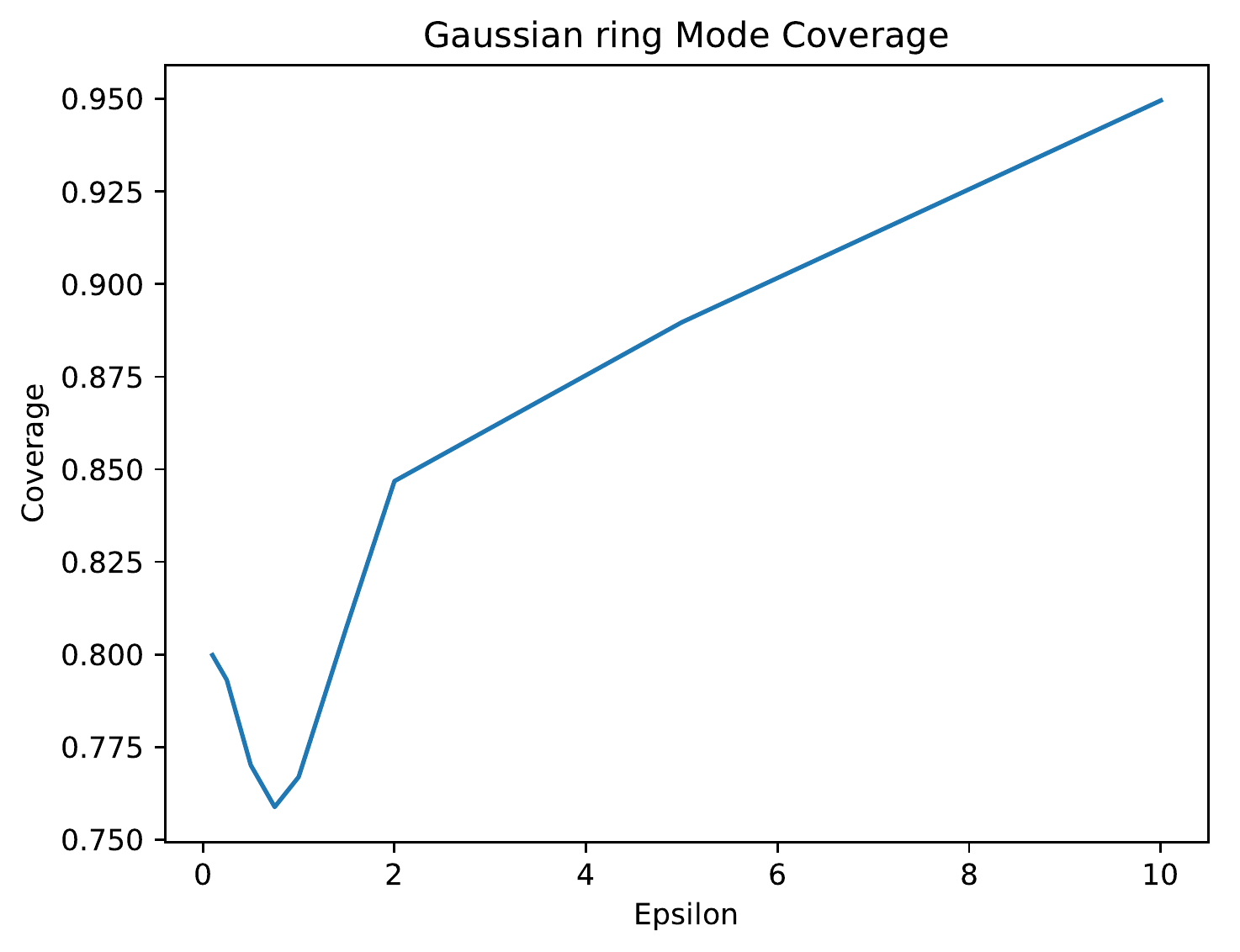} \hspace{\whep} & \hspace{\whep} \includegraphics[trim=20bp 20bp 20bp
20bp,clip,width=\whed]{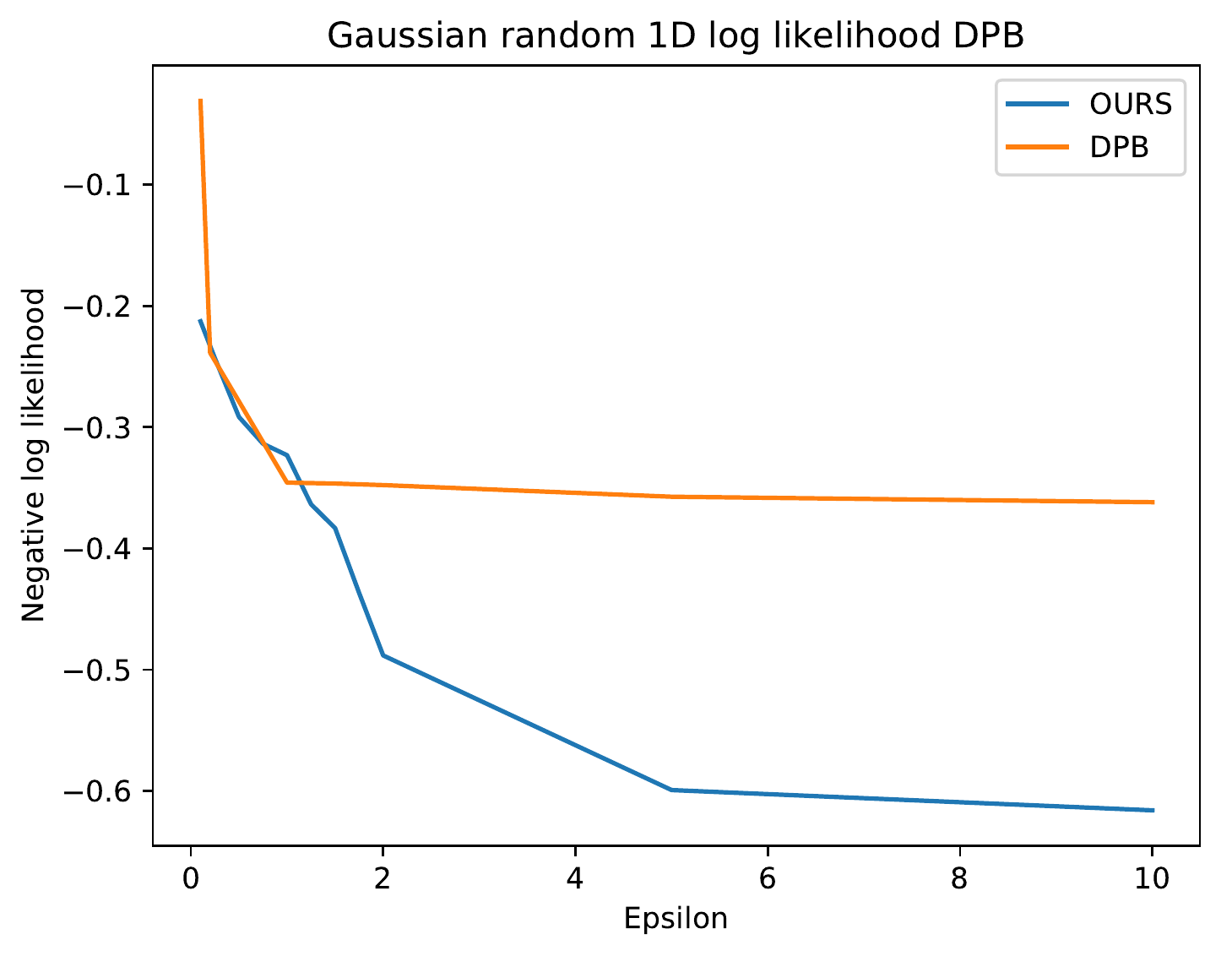} 
\hspace{\whep} & \hspace{\whep}\includegraphics[trim=20bp 20bp 12bp
20bp,clip,width=\whed]{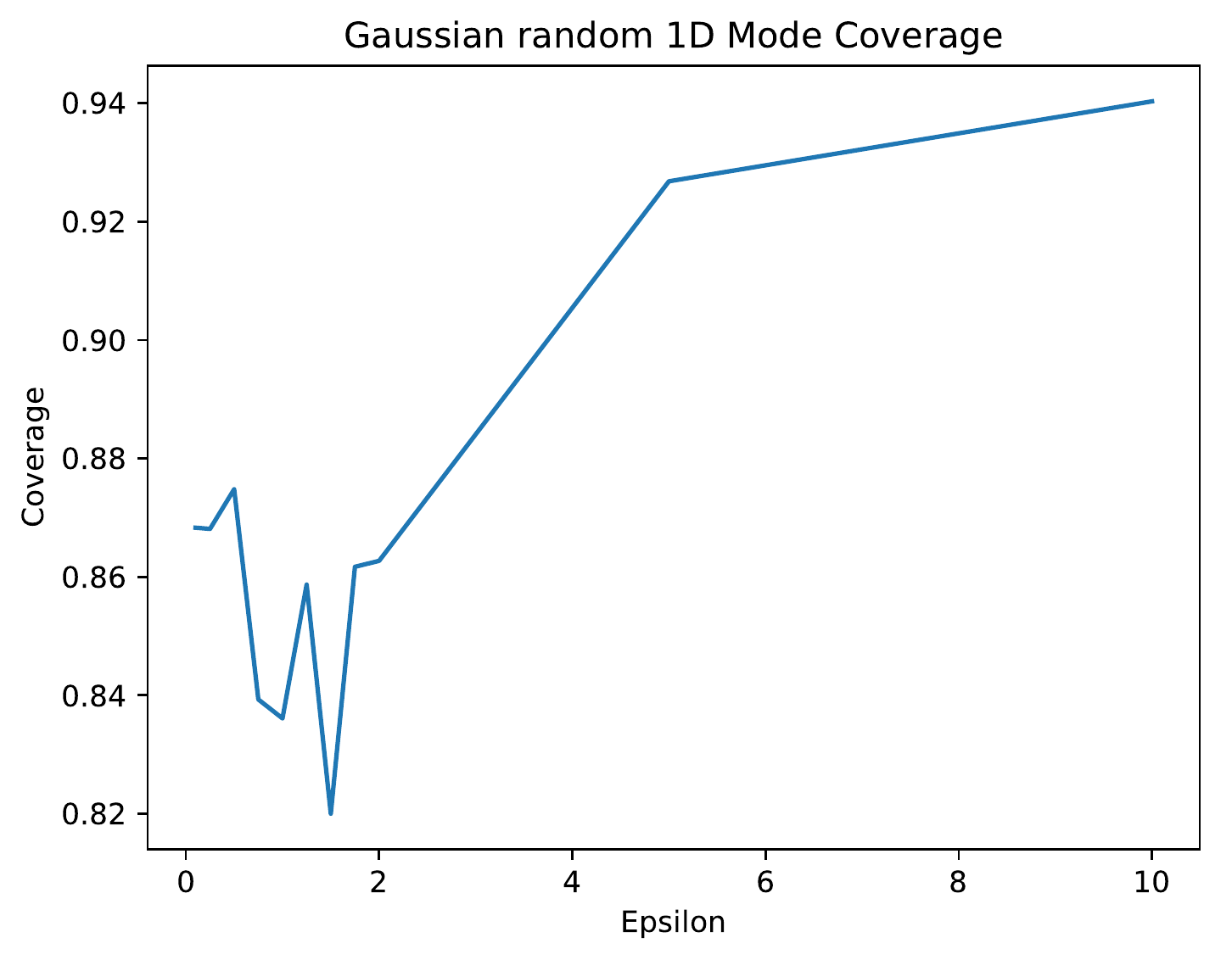} \hspace{\whep} \\
\hspace{\whep}  NLL = f($\varepsilon$) \hspace{\whep} & \hspace{\whep}
                                                        Mode coverage
                                                        =
                                                        f($\varepsilon$)
                                                        \hspace{\whep}
               & \hspace{\whep} NLL = f($\varepsilon$) \hspace{\whep}
                                                      & \hspace{\whep}
                                                        Mode coverage
                                                        =
                                                        f($\varepsilon$)
  \\ \hline\hline
\end{tabular}
\caption{Metrics for \pkde~(blue): NLL (lower is
  better) and mode coverage (higher is better). Orange: DPB (see text).}
\label{GRing+G1D_metrics}
\end{figure*}

\noindent $\triangleright$ \textbf{Architectures} (of $Q_t$, private KDE and
private GANs): we carried out experiments on a simulated setting inspired by
\cite{arTB}, to compare \pkde~(implemented following its description
in Section \ref{sec:algo}) against differentially private KDE
\citep{arTB}. To learn the sufficient statistics for \pkde, we fit for
each $c_t$ a neural network (NN) classifier:
\begin{align}\mathcal{X} \xrightarrow[\text{dense}]{\tanh}
  \mathbb{R}^{25} \xrightarrow[\text{dense}]{\tanh} \mathbb{R}^{25}
  \xrightarrow[\text{dense}]{\tanh} \mathbb{R}^{25}
  \xrightarrow[\text{dense}]{\mathrm{sigmoid}} (0,1),  \end{align}
 where $\mathcal{X} \in \{\mathbb{R}, \mathbb{R}^2\}$ depending on the experiment.
At each iteration $t$ of boosting, $c_t$
is trained using $10 000$ samples from $P$ and $Q_{t-1}$ using
Nesterov's accelerated gradient descent with $\eta = 0.01$ based on
cross-entropy loss with $750$ epochs. Random walk Metropolis-Hastings
is used to sample from $Q_{t-1}$ at each iteration. For the number of
boosting iterations in \pkde, we pick $T=3$. This is quite a small
value but given the rate of decay of $\theta_t(\epsilon)$ and the
small dimensionality of the domain, we found it a good compromise for
complexity vs accuracy. Finally, $Q_0$ is a standard Gaussian
${\mathcal{N}}(\bm{0}, \textsc{i}_d)$. \\[0.2\baselineskip]
\noindent $\triangleright$ \textbf{Contenders}: we know of no integrally
private sampling approach operating under conditions equivalent to ours, so our main contender is going to be a particular state of the
art $\varepsilon$-differentially private approach which provides a
private \textit{density}, DPB \citep{arTB}. We choose this approach
because digging in its technicalities reveal that \textit{its \textbf{integral
  privacy budget} would be roughly equivalent to ours, mutatis mutandis}. Here is why:
this approach allows to
sample a dataset of arbitrary size (say, $k$) while keeping the same privacy budget, \textit{but}
needs to be scaled to accomodate integral privacy, while in our case,
\pkde~allows to obtain integral privacy for one observation ($k=1$),
\textit{but} its privacy budget
needs to be scaled to accomodate for larger $k$. It turns out that in
both approaches, the scaling of the privacy parameter to accomodate for
arbitrary $k$ and integral privacy is roughly the \textit{same}. In our case,
the change is obvious: the
privacy parameter $\varepsilon$ is naturally scaled by $k$. In
the case of \cite{arTB}, the requirement of integral privacy
multiplies the sensitivity\footnote{Cf \cite[Definition 4]{arTB} for
the sensitivity, \cite[Section 6]{arTB} for the key function $F_H(.,.)$
involved.} by $k$, which implies that the Laplace mechanism does not change
only if $\varepsilon$ is scaled by $k$ \cite[Section
3.3]{drTA}.\\
We have also compared with a private GAN approach, which
has the benefit to yield a simple sampler \textit{but} involves a
weaker privacy model \citep{xlwwzDP} (DPGAN).  For DPB, we use a bandwidth kernel and
learn the bandwidth parameter via $10$-fold cross-validation. For
DPGAN, we train the WGAN base model using batch sizes of $128$ and
$10000$ epochs, with $\delta = 10^{-1}$. We found that DPGAN is
significantly outperformed by both DPB and \pkde, so to save
space we have only included the experiment in Figure
\ref{fig:Princ} (right). We observed that DPB does not always yield a
positive measure. To ensure positivity, we shift
and scale the output, \textit{without} caring for privacy in doing so, which
degrades the privacy guarantee for DPB but keeps the
approximation guarantees of
its output \citep{arTB}. \\[0.2\baselineskip]
\noindent $\triangleright$ \textbf{Metrics}: we consider two metrics, inspired by those we consider
for our theoretical analysis and one investigated in \cite{tgbssAB}
for mode capture. We first investigate the ability of
our method to learn highly dense regions by computing \textit{mode
  coverage}, which is defined to be $P(dQ < t)$ for $t$ such that
$Q(dQ < t) = 0.95$. Mode coverage essentially attempts to find high
density regions of the model $Q$ (based on $t$) and computes the mass
of the target $P$ under this region. Second, we compare the
negative log likelihood, $-E_P[\log Q]$ as a general loss measure. \\[0.2\baselineskip]
\noindent $\triangleright$ \textbf{Domains}: we essentially consider three
different problems. The first is the ring Gaussians problem now common
to generative approaches \citep{gGA}, in which 8 Gaussians have
their modes regularly spaced on a circle. The target $P$ is shown in
Figure \ref{fig:Princ}. Second, we consider a mixture of three 1D
gaussians with pdf $P(x) = \frac{1}{3} \bracket{\mathcal{N}(0.3,0.01) + \mathcal{N}(0.5,0.1) + \mathcal{N}(0.7,0.1)}$. For the final
experiment, we consider a 1D domain and randomly place
$m$ gaussians with means centered in the interval $[0,1]$ and
variances $0.01$. We vary $m = 1,\ldots,10$, $\epsilon \in
(0,2]$ and repeat the experiment four times to get means and standard deviations. \supplement~(Section \ref{supp-exp}) shows more experiments. \\[0.2\baselineskip]
\noindent $\triangleright$ \textbf{Results}: Figure \ref{ring:DPB-vs-proposed} displays contour plots of
the learned $Q$ against DPB \citep{arTB}. Figure \ref{GRing+G1D_metrics} provides metrics. We
indicate the metric performance for DPB on one plot only since density
estimates obtained for some of the other metrics could not allow for
an accurate computation of metrics.
The experiments bring the following observations: \pkde~is significantly
better at integrally private density estimation than DPB if we look at the ring Gaussian
problem. \pkde~essentially obtains the same results as DPB for values
of $\epsilon$ that are 400 times smaller as seen from Figure
\ref{fig:Princ}. We also remark that the density modelled are more
smooth and regular for \pkde~in this case. One might attribute the fact that our
performance is much better on the ring Gaussians to the fact that
our $Q_0$ is a standard Gaussian, located at the middle of the ring in
this case, but experiments on random 2D Gaussians (see \supplement)
display that our performances also remain better in other settings where
$Q_0$ should represent a handicap. All domains, including the 1D
random Gaussians experiments in Figure \ref{RG_NLL_DPBvsUS} (\supplement),
 display a consistent decreasing NLL for \pkde~as $\epsilon$
increases, with sometimes very sharp decreases for $\epsilon < 2$ (See
also \supplement, Section \ref{supp-exp}). We attribute it to
the fact that it is in this regime of the privacy parameter that
\pkde~captures all modes of the mixture. For larger values of
$\epsilon$, it justs fits better the modes already discovered. We also
remark on the 1D Gaussians that DPB rapidly reaches a plateau of NLL
which somehow show that there is little improvement as $\epsilon$
increases, for $\epsilon \geq 1$. This is not the case for \pkde,
which still manages some additional improvements for $\epsilon > 5$
and significantly beats DPB. We attribute it to the 
flexibility of the sufficient statistics as (deep) classifiers in
\pkde. The 1D random Gaussian problem (Figure \ref{RG_NLL_DPBvsUS} in \supplement)
displays the same pattern for \pkde. We also observe that the standard deviation
of \pkde~is often 100 times \textit{smaller} than for DPB, indicating
not just better but also 
much more stable results.
In the case of mode coverage, we observe for several experiments
(\textit{e.g.} ring Gaussians) that the mode coverage
\textit{decreases} until $\epsilon \approx 1$, and then increases, on
all domains, for \pkde. This, we believe is due to our choice of
$Q_0$, which as a Gaussian, already captures with its mode a part of
the existing modes. As $\epsilon$ increases however, \pkde~performs
better and obtains in general a
significant improvement over $Q_0$. We also observe this phenomenon
for the random 1D Gaussians (Figure \ref{RG_coverage_US}) where the
very small standard deviations (at least for $\varepsilon > .25$ or
$m>1$) display a significant stability for the solutions of \pkde.

\begin{figure}
\centering
\begin{tabular}{cc}\\
\hspace{\whep}\includegraphics[trim=20bp 20bp 12bp
20bp,clip,width=\whed]{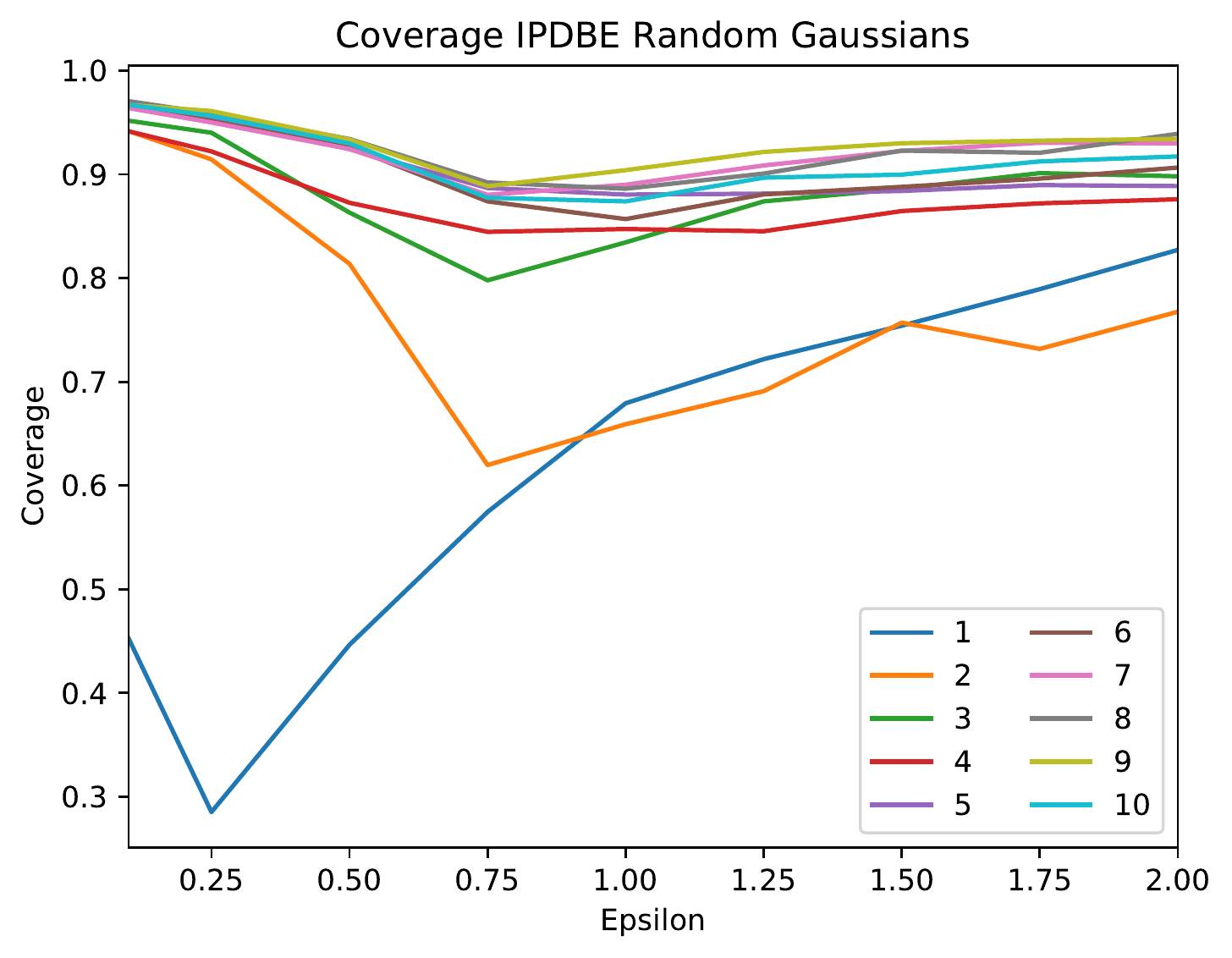}
\hspace{\whep}& \hspace{\whep}\includegraphics[trim=20bp 20bp 12bp
20bp,clip,width=\whed]{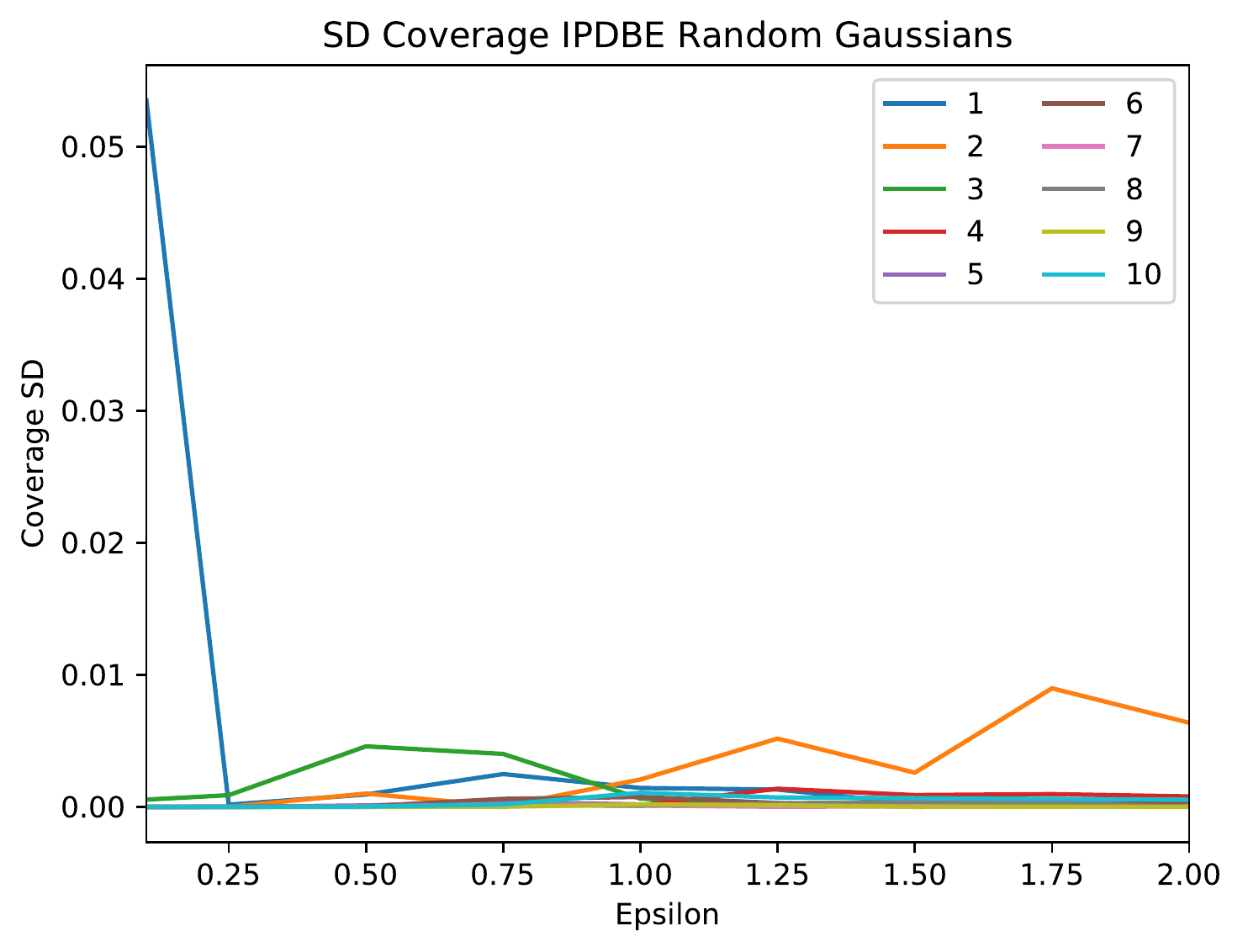}\\\
\hspace{\whep} Mean = f($\varepsilon$) \hspace{\whep}& \hspace{\whep} StDev = f($\varepsilon$)
\end{tabular}
\caption{Mode coverage for \pkde~on 1D random Gaussian.}
\label{RG_coverage_US}
\end{figure}

\section{Discussion: integral privacy, bias and fairness}\label{sec:disc}

Over the past years, privacy has not been the only issue facing the
deployment of machine learning at scale: bias and fairness are other major
issues for the field \citep{bsBD}. In fact, it has been argued that
both should be simultaneously ensured
\citep{jkmorsuDP}, on the basis that privacy may in fact restrict the
access to fairness-checking information in the data. On the other
hand, it has independently been observed that differential privacy
can have unfair consequences \citep{bsDPH}. 
This is not surprising: differential privacy is an individual notion
of privacy and
the
noisfying process that usually goes with it is therefore likely to affect information from
small groups before it affects the majority. Bias and fairness are in general notions
that handle disparate treatments on groups, and often focus on related
minorities: ethnicity,
age, religion, gender, sexual orientation, etc. \citep{zAS}. Hence,
there is a risk that individual noisification for privacy washes out the information of
the small groups that fairness would in fact seek to protect.
By extending the notion of differential privacy to the group level,
to any group of any size, one might wonder whether integral
privacy does bring guarantees from the bias and fairness
standpoints. Thanks to the closure under post-processing of integral privacy, it is easy
to show that any source (in data) of
potential bias or unfairness in any post-processing
gets tampered when data has been integrally privately
sampled. Let $\mathcal{A}$ denote an $\varepsilon$-integrally private sampler that takes as
input datasets and outputs samples from $\mathcal{X}$, as in
Definition \ref{basicDef}. Hereafter, we let $\mathcal{T}$ denote a
space of possibly sensitive outcomes such as the decision to hire, to
give a loan, etc. Let $f : 2^{\mathcal{X}} \rightarrow \mathcal{T}$ be any
algorithm processing data to provide with such decisions. We make no
further assumptions about $f$ nor the eventual additional inputs it
may have. We just reason about the input in ${\mathcal{X}}$ and the
output in $\mathcal{T}$.
\begin{lemma}
Let $\mathcal{A}, \mathcal{T}$ be defined as above and let $D, D'$ any
two datasets. Then for any $T \subseteq \mathcal{T}$,
$\text{Pr}[f\circ \mathcal{A}(D) \in T] \leq \exp(\varepsilon) \cdot \text{Pr}[f\circ \mathcal{A}(D') \in T]$.
\end{lemma}
\begin{proof}
Like in the proof of \cite[Proposition 2.1]{drTA}, we suppose
without loss of generality that $f$ is deterministic. 
Fix any event $T \subseteq \mathcal{T}$, noting that $f^{-1}(T) \subseteq
\mathcal{X}$. Since $\mathcal{A}$ is an $\varepsilon$-integrally
private sampler, 
\begin{eqnarray}
\text{Pr}[f\circ \mathcal{A}(D) \in T] & = & \text{Pr}[\mathcal{A}(D)
                                             \in f^{-1}(T)]\nonumber\\
& \leq & \exp(\varepsilon) \cdot \text{Pr}[\mathcal{A}(D')
                                             \in f^{-1}(T)]\nonumber\\
& & = \exp(\varepsilon) \cdot \text{Pr}[f\circ \mathcal{A}(D')
                                             \in T],
\end{eqnarray}
which proves the Lemma. 
\end{proof}
Hence, any decision process has limited data-dependent bias when the
data is integrally privately sampled. We emphasize, as already noted from the abstract, that
such a strong guarantee of unbiasedness does not come for free from
the standpoint of approximating the true data distribution, but it
is intuitive that unbiasedness from input data should prevent too much
learning or overfitting this input data -- which otherwise could eventually reveal
bias. It is however possible, under some assumptions, to come close to the best
possible approximation, as explained in Theorem \ref{kl-upper-lower}.
\section{Conclusion}\label{sec:conc}

In this paper, we have proposed an extension of
$\varepsilon$-differential privacy to handle the protection of groups
of arbitrary size, and applied it to sampling. The technique bypasses
noisification and the sensitivity analysis as
usually carried out in DP. The privacy parameter $\epsilon$ also acts
as a slider between approximation vs fairness guarantees: 
higher approximation guarantees of the target can be obtained
when $\epsilon$ is large, while higher guarantees on data-dependent unbiasedness of any decision
process that would use the integrally privately sampled data follow when
$\epsilon$ is small.
An efficient learning algorithm is proposed, with approximation guarantees in the context of the
boosting theory. Experiments demonstrate the quality of the
solutions found, in particular in the context of the mode capture
problem. 
\section*{Acknowledgements and code availability}\label{sec:ack}

We are indebted to Benjamin Rubinstein for providing us with the Private KDE code, Borja de Balle Pigem and anonymous reviewers for
significant help in correcting and improving focus, clarity and presentation, and finally Arthur Street for stimulating discussions around this material. Our code is
available at: 
\begin{center}
\texttt{https://github.com/karokaram/PrivatedBoostedDensities}
\end{center}

\bibliographystyle{abbrvnat}
\bibliography{references,bibgen}

\clearpage
\onecolumn
\section*{\supplement: table of contents}

\noindent \textbf{Proofs and formal results}\hrulefill Pg \pageref{supp-formal}\\
\noindent Proof of Lemma \ref{molpriv}\hrulefill Pg
\pageref{proof_molpriv}\\
\noindent Proof of Theorem \ref{privacy-theorem}\hrulefill Pg
\pageref{proof_privacy-theorem}\\
\noindent Proof of Theorem \ref{kl-upper}\hrulefill Pg
\pageref{proof_kl-upper}\\
\noindent Proof of Theorem \ref{kl-upper-lower}\hrulefill Pg
\pageref{proof_kl-upper-lower}\\
\noindent Proof of Theorem \ref{first-mode-capture}\hrulefill Pg
\pageref{proof_mode-capture}\\
\noindent Additional formal results\hrulefill Pg
\pageref{proof_nonTrivial}\\

\noindent \textbf{Additional experiments}\hrulefill Pg \pageref{supp-exp}\\

\newpage

\section{Proofs and formal
  results}\label{supp-formal}

\subsection{Proof of Lemma \ref{molpriv}}\label{proof_molpriv}
The proof follows from two simple observations: (i) ensuring \eqref{eqdiffpriv} is equivalent to $\text{Pr}[\mathcal{A}(d) \in  S] \leq \exp(\epsilon) \cdot \text{Pr}[\mathcal{A}(D') \in S]$ since it has to holds for all $S$, and (ii) the probability to sample any $S$ is equal to the mass under the density from which it samples from: \begin{align} \text{Pr}[\mathcal{A}(D) \in S] = \int_S dQ_{\epsilon}(x;D). \end{align}
Recall that base measures are assumed to be the same, so being in $\mathcal{M}$ translates to a property on Radon-Nikodym derivatives, $dQ / dQ' \leq \exp(\varepsilon)$, and we then get the statement of the Lemma: since $Q_{\epsilon}(x;.) \in \mathcal{M}$ where $\mathcal{M}$ is a $\varepsilon$-mollifier, we get from Definition \ref{defMOL} that for any input samples $D$, $D'$ from $\mathcal{X}$ and any $S \subseteq \mathcal{X}$:
\begin{eqnarray}
\lefteqn{\text{Pr}[\mathcal{A}(D) \in S] = \int_S dQ_{\epsilon}(x;D) = \int_S \frac{dQ_{\epsilon}(x;D)}{dQ_{\epsilon}(x;D')} \cdot dQ_{\epsilon}(x;D')}\nonumber\\
 & \leq & \exp(\epsilon) \cdot \int_S dQ_{\epsilon}(x;D') = \exp(\epsilon) \cdot \text{Pr}[\mathcal{A}(D') \in S],
\end{eqnarray}
which shows that $\mathcal{A}$ is $\varepsilon$-integrally private. 

\subsection{Proof of Theorem \ref{privacy-theorem}}\label{proof_privacy-theorem}

The proof follows from two Lemma which we state and prove.
\begin{lemma}
\label{theta-convergences}
For any $T \in \mathbb{N}_*$, we have that \begin{align} \sum_{t=1}^T \theta_t(\epsilon) = \sum_{t=1}^T \bracket{\frac{\epsilon}{\epsilon + 4 \log(2)}}^t < \frac{\epsilon}{4\log(2)}. \end{align}
\end{lemma}
\begin{proof}
Since $(\epsilon / (\epsilon + 4\log(2)) < 1$ for any $\epsilon$ and noting that $\theta_t(\epsilon) = (\epsilon / (\epsilon + 4\log(2)) \theta_{t-1} (\epsilon)$, we can conclude that $\theta_t(\epsilon)$ is a geometric sequence. For any geometric series with ratio $r$, we have that \begin{align} \sum_{t=1}^T r^t &= r\bracket{\frac{1-r^T}{1-r}}\\ &=\frac{r}{1-r} -\frac{r^{T+1}}{1-r}\\ &< \frac{r}{1-r}  \end{align} Indeed, $\frac{r}{1-r}$ is the limit of the geometric series above when $T \to \infty$. In our case, we let $r = (\epsilon / (\epsilon + 4\log(2)))$ to show that \begin{align} \frac{r}{1-r} &= \frac{\frac{\epsilon}{\epsilon + 4\log(2)}}{1 - \frac{\epsilon}{\epsilon + 4\log(2)}} = \frac{\frac{\epsilon}{\epsilon + 4\log(2)}}{\frac{4 \log(2)}{\epsilon + 4\log(2)}} = \frac{\epsilon}{4 \log(2)},\end{align} which concludes the proof.
\end{proof}
\begin{lemma}
\label{privbound}
For any $\epsilon > 0$ and $T \in \mathbb{N}_*$,
let $\theta(\epsilon) =
(\theta_1(\epsilon),\ldots,\theta_T(\epsilon))$ denote the parameters
and $c = (c_1,\ldots,c_t)$ denote the sufficient statistics returned
by Algorithm \ref{mainalg}, then we have
\begin{eqnarray}
-\frac{\epsilon}{2} \leq  \ip{\theta(\epsilon)}{c} - \phi(\theta(\epsilon)) \leq \frac{\epsilon}{2}.\label{eqnorm}
\end{eqnarray}
\end{lemma}
\begin{proof}
Since the algorithm returns classifiers such that $c_t(x) \in [-\log2,\log2]$ for all $1 \leq t \leq T$, we have from Lemma \ref{theta-convergences}, \begin{align} \sum_{t=1}^T \theta_t (\epsilon) c_t \leq \log(2) \sum_{t=1}^T \theta_t (\epsilon) < \log(2) \frac{\epsilon}{4 \log(2)} = \frac{\epsilon}{4},  \end{align} and similarly, \begin{align} \sum_{t=1}^T \theta_t (\epsilon) c_t \geq -\log(2) \sum_{t=1}^T \theta_t (\epsilon) > -\log(2) \frac{\epsilon}{4 \log(2)} = -\frac{\epsilon}{4}.  \end{align} Thus we have  \begin{align} \label{ineq1} -\frac{\epsilon}{4} &\leq \ip{\theta(\epsilon)}{c} \leq  \frac{\epsilon}{4}.  \end{align}
By taking exponential, integrand (w.r.t $Q_0$) and logarithm of \ref{ineq1}, we get \begin{align}\log\int_{\mathcal{X}}\exp\bracket{-\frac{\epsilon}{4}}dQ_0 &\leq \log \int_{\mathcal{X}}\exp\bracket{\ip{\theta(\epsilon)}{c}}dQ_0 \leq  \log \int_{\mathcal{X}}\exp\bracket{\frac{\epsilon}{4}}dQ_0\\ -\frac{\epsilon}{4} &\leq \phi(\theta(\epsilon)) \leq \frac{\epsilon}{4}\end{align}
Since $\ip{\theta(\epsilon)}{c} \in [-\epsilon / 4, \epsilon/ 4]$ and $\phi(\theta(\epsilon)) \in [-\epsilon / 4, \epsilon/ 4]$, the proof concludes by considering highest and lowest values.
\end{proof}

The proof of Theorem \ref{privacy-theorem} now follows from taking the $\exp$ of all quantities in \eqref{eqnorm}, which makes appear $Q_T$ in the middle and conditions for membership to $\mathcal{M}_\varepsilon$ in the bounds.

\subsection{Proof of Theorem \ref{kl-upper}}\label{proof_kl-upper}

\newcommand{\kl}{\text{KL}}
\newcommand{\dd}{d}
\newcommand{\bracess}[1]{\left(#1\right)}

We begin by first deriving the KL drop expression. At each iteration, we learn a classifier $c_t$, fix some step size $\theta > 0$ and multiply $Q_{t-1}$ by $\exp(\theta \cdot c_t)$ and renormalize to get a new distribution which we will denote by $Q_t(\theta)$ to make the dependence of $\theta$ explicit.
\begin{lemma}
\label{kl-drop-lem}
For any $\theta > 0$, let $\phi(\theta) = \log \int_{\mathcal{X}} \exp(\theta \cdot c_t) dQ_{t-1}$. The drop in KL is \begin{align} \text{DROP}(\theta) := \kl(P,Q_{t-1}) - \kl(P,Q_t(\theta))  = \theta \cdot \int_{\mathcal{X}} c_t dP - \phi(\theta)\end{align}
\end{lemma}
\begin{proof}
Note that $Q_t(\theta)$ is indeed a one dimensional exponential family with natural parameter $\theta$, sufficient statistic $c_t$, log-partition function $\phi(\theta)$ and base measure $Q_{t-1}$. We can write out the KL divergence as
\begin{align}
\kl(P,Q_{t-1}) - \kl(P,Q_t(\theta)) &= \int_{\mathcal{X}} \log\bracket{\frac{P}{Q_{t-1}}} dP - \int_{\mathcal{X}} \log\bracket{\frac{P}{\exp(\theta \cdot c_t - \phi(\theta))Q_{t-1}}} dP\\
                                    &= \int_{\mathcal{X}} \log\bracket{\frac{\exp(\theta \cdot c_t - \phi(\theta))Q_{t-1}}{Q_{t-1}}} dP\\
                                    &= \int_{\mathcal{X}} \theta \cdot c_t - \phi(\theta) dP\\
                                    &= \theta \cdot \int_{\mathcal{X}} c_t dP - \phi(\theta)
\end{align}
\end{proof}
It is not hard to see that the drop is indeed a concave function of $\theta$, suggesting that there exists an optimal step size at each iteration. We split our analysis by considering two cases and begin when $\gamma_Q^{t} < 1/3$. Since $\theta > 0$, we can lowerbound the first term of the KL drop using WLA. The trickier part however, is bounding $\phi(\theta)$ which we make use of Hoeffding's lemma.
\begin{lemma}[Hoeffding's Lemma]
\label{hoeffding-lemma}
Let $X$ be a random variable with distribution $Q$, with $a \leq X \leq b$ such that $\E_Q[X] = 0$, then for all $\lambda > 0$, we have \begin{align}\E_Q[\exp(\lambda \cdot X)] \leq \exp\bracket{\frac{\lambda^2 (b-a)^2}{8}}  \end{align}
\end{lemma}
\begin{lemma}
\label{hoeffding-wla-bound}
For any classifier $c_t$ satisfying Assumption \ref{def-wla} (WLA), we have \begin{align} \E_{Q_{t-1}}[\exp(\theta_t (\epsilon) \cdot c_t)]  \leq \exp\bracket{\theta_t^2 (\epsilon) \cdot \frac{(c_t^{*})^2}{2}  - \theta_t (\epsilon) \cdot \gamma_Q^t \cdot c_t^{*}} \end{align}
\end{lemma}
\begin{proof}
Let $X = c_t - \cdot \E_{Q_{t-1}}[c_t]$, $b = c_t^{*}$, $a = -c_t^{*}$ and $\lambda = \theta_t(\epsilon)$ and noticing that \begin{align} \E_{Q_{t-1}}[\lambda \cdot X] = \E_{Q_{t-1}}[c_t - \E_{Q_{t-1}}[c_t]] = \E_{Q_{t-1}}[c_t] - \E_{Q_{t-1}}[c_t] = 0,\end{align} allows us to apply Lemma \ref{hoeffding-lemma}. By first realizing that
\begin{align}
\exp(\lambda \cdot X) = \exp(\theta_t (\epsilon) \cdot c_t) \cdot \exp(\theta_t (\epsilon) \cdot \E_{Q_{t-1}}[-c_t]),
\end{align}
We get that \begin{align} \E_{Q_{t-1}}[\exp(\theta_t (\epsilon) \cdot c_t)] \cdot \exp\bracket{\theta_t (\epsilon) \cdot \E_{Q_{t-1}}[-c_t]} \leq \exp\bracket{\theta_t^2 (\epsilon) \cdot \frac{(c_t^{*})^2}{2} }. \end{align} Re-arranging and using the WLA inequality yields
\begin{align} \E_{Q_{t-1}}[\exp(\theta_t (\epsilon) \cdot c_t)] &\leq  \exp\bracket{\theta_t^2 (\epsilon) \cdot \frac{(c_t^{*})^2}{2}  - \theta_t (\epsilon) \cdot \E_{Q_{t-1}}[-c_t] }\\ &\leq \exp\bracket{\theta_t^2 (\epsilon) \cdot \frac{(c_t^{*})^2}{2}  - \theta_t (\epsilon) \cdot \gamma_Q^t \cdot c_t^{*} } \end{align}
\end{proof}
Applying Lemma \ref{hoeffding-wla-bound} and Lemma \ref{kl-drop-lem} (writing $Q_t = Q_t(\epsilon)$ ) together gives us
\begin{align}
\kl(P,Q_t) &= \kl(P,Q_{t-1}) - \text{DROP}(\theta_t(\epsilon))\\
           &= \kl(P,Q_{t-1}) - \theta_t(\epsilon) \cdot \int_{\mathcal{X}}c_t dP + \log \E_{Q_{t-1}}[\exp(\theta_t (\epsilon) \cdot c_t)]\\
           &\leq \kl(P,Q_{t-1}) - c_t^{*} \cdot \theta_t(\epsilon) \cdot \bracket{\frac{1}{c_t^{*}} \int_{\mathcal{X}} c_t dP } + \bracket{\theta_t^2 (\epsilon) \cdot \frac{(c_t^{*})^2}{2}  - \theta_t (\epsilon) \cdot \gamma_Q^t \cdot c_t^{*}}\\
           &\leq \kl(P,Q_{t-1}) - c_t^{*} \theta_t(\epsilon) \bracket{\gamma_P^{t} + \gamma_Q^{t} - \frac{c_t^{*} \cdot \theta_t(\epsilon)}{2}}
\end{align}
Now we move to the case of $\gamma_Q^{t} \geq 1/3$.
\begin{lemma}
\label{expec-exp-large-gamma-bound}
For any classifier $c_t$ returned by Algorithm \ref{mainalg}, we have that \begin{align} \E_{Q_{t-1}}[\exp(c_t)] \leq \exp\bracket{-\Gamma (\gamma_Q^{t})} \end{align} where $\Gamma(z) =  \log(4/(5 - 3z))$.
\end{lemma}
\begin{proof}
Consider the straight line between $(-\log 2, 1/2)$ and $(\log 2, 2)$ given by $y = 5/4 + (3/(4 \cdot \log 2)) x$, which by convexity is greater then $y = \exp(x)$ on the interval $[-\log 2, \log 2]$. To this end, we define the function \begin{align} f(x) = \begin{cases} \frac{5}{4} +  \frac{3}{4 \cdot \log 2} \cdot x, &\text{     if   } x \in [-\log 2, \log 2]\\ 0, &\text{  otherwise}    \end{cases} \end{align} Since $c_t (x) \in [-\log 2, \log 2]$ for all $x \in \mathcal{X}$, we have that $f(c_t (x)) \geq \exp(c_t (x))$ for all $x \in \mathcal{X}$. Taking $\E_{Q_{t-1}}[\cdot]$ over both sides and using linearity of expectation gives \begin{align} \E_{Q_{t-1}}[\exp(c_t(x))] &\leq \E_{Q_{t-1}}[f(c_t(x))] \\ &= \frac{5}{4} + \frac{3}{4\log 2} \bracket{\ \E_{Q_{t-1}} [c_t (x)] } \\ &= \frac{5}{4} - \frac{3}{4} \bracket{\frac{1}{\log2} \E_{Q_{t-1}} [-c_t (x)] } \\ \label{exp-upperbound} &< \frac{5}{4} - \frac{3}{4} \gamma_Q^{t} \\
&= \exp\bracket{-\log\bracket{\frac{5 - 3\gamma_Q^{t}}{4}}^{-1}}  \\ &= \exp\bracket{-\log\bracket{\frac{4}{5 - 3\gamma_Q^{t}}}} \\ &= \exp\bracket{-\Gamma(\gamma_Q^{t})}, \end{align}
as claimed.
\end{proof}
Now we use Lemma \ref{kl-drop-lem} and Jensen's inequality since $\theta_t(\epsilon) < 1$ so that
\begin{align}
\text{KL}(P,Q_t) &= \text{KL}(P,Q_{t-1}) - \text{DROP}(\theta)\\
                 &= \kl(P,Q_{t-1}) - \theta_t(\epsilon) \cdot \int_{\mathcal{X}}c_t dP + \log\E_{Q_{t-1}}[\exp(\theta_t \cdot c_t)]\\
                 &\leq \kl(P,Q_{t-1}) - \theta_t(\epsilon) \cdot \E_P[c_t] + \theta_t \cdot \log\E_{Q_{t-1}}[\exp(c_t)]\\
                 &\leq \text{KL}(P,Q_{t-1}) - \theta_t(\epsilon) \bracket{\E_P[c_t] - \log\E_{Q_{t-1}}[\exp(c_t)]}\\
                 &=  \text{KL}(P,Q_{t-1}) - \theta_t(\epsilon) \bracket{  c_t^{*} \bracket{ \frac{1}{c_t^{*}} \E_P[c_t]} - \log \E_{Q_{t-1}}[\exp(c_t)]}\\
                 &< \text{KL}(P,Q_{t-1}) - \theta_t(\epsilon) \bracket{  c_t^{*} \gamma_P^{t} - \log \bracket{\exp\bracket{ -\Gamma(\gamma_Q^{t}) }}  }\\
                 &= \text{KL}(P,Q_{t-1}) - \theta_t(\epsilon) \bracket{  c_t^{*} \gamma_P^{t} + \Gamma(\gamma_Q^{t})   }.
\end{align}

\subsection{Proof of Theorem \ref{kl-upper-lower}}\label{proof_kl-upper-lower}

We first note that for any $Q \in \mathcal{M}^{\exp}_\varepsilon$, 
\begin{align} \text{KL}(P,Q) &= \int_{\mathcal{X}} \log\bracket{\frac{P}{Q }}dP\\ &=  \int_{\mathcal{X}} \log\bracket{\frac{P}{Q_0 \exp\bracket{\ip{\theta(\epsilon)}{c} - \phi(\theta(\epsilon)) }}}dP\\  &=\int_{\mathcal{X}} \log\bracket{\frac{P}{Q_0}} dP - \int_{\mathcal{X}} \bracket{\ip{\theta(\epsilon)}{c} - \phi(\theta(\epsilon))} dP\\ &\geq \text{KL}(P,Q_0) - \int_{\mathcal{X}} \frac{\epsilon}{2} dP \\ &\geq \text{KL}(P,Q_0) - \frac{\epsilon}{2}, \end{align}
which completes the proof of the upperbound To show \eqref{bsupKL},  we have that \begin{align} \text{KL}(P,Q_t) &\leq \text{KL}(P,Q_{T-1}) -  \theta_t(\epsilon) \cdot \Lambda_t \\ &\leq \text{KL}(P,Q_0) - \sum_{t=1}^{T-1} \theta_t(\epsilon) \cdot \Lambda_t \\ &=  \text{KL}(P,Q_0) - \sum_{t=1}^{T-1} \theta_t(\epsilon) \cdot \bracket{ c_t^{*} \gamma_P^{t} + \Gamma(\gamma_Q^{t})} \\ &\leq  \text{KL}(P,Q_0) - \sum_{t=1}^{T-1} \theta_t(\epsilon) \cdot \bracket{ \log 2 \cdot  \gamma_P + \Gamma(\gamma_Q)} \\ &\leq \text{KL}(P,Q_0) - \bracket{ \log 2 \cdot  \gamma_P + \log 2 \cdot \gamma_Q} \cdot  \sum_{t=1}^{T-1} \theta_t(\epsilon) \\ &\leq \text{KL}(P,Q_0) - \bracket{ \log 2 \cdot  \gamma_P + \log 2 \cdot \gamma_Q} \cdot  \sum_{t=1}^{T-1} \theta_t(\epsilon) \\ &= \text{KL}(P,Q_0) - \log 2 \cdot \bracket{  \gamma_P +  \gamma_Q} \cdot \theta_1 (\epsilon) \cdot  \bracket{\frac{1 - \theta_t(\epsilon)}{1 - \theta_1(\epsilon)}} \\ &=   \text{KL}(P,Q_0) - \epsilon \cdot \bracket{ \frac{ \gamma_P +  \gamma_Q}{4}  } \cdot    \bracket{1 - \theta_t(\epsilon)}, \end{align} where we used the fact that $\Gamma(x) \geq \log2 \cdot x$ and explicit geometric summation expression. 

\subsection{Proof of Theorems \ref{first-mode-capture}}\label{proof_mode-capture}

We start by a general Lemma.
\begin{lemma}
\label{first-region-bound}
For any region of the support $B$, we have that \begin{align} \int_{B} dQ_t \geq \int_B dP - \int_{B} \log\bracket{\frac{P}{Q_t}}dP\end{align}
\end{lemma}
\begin{proof}
By first noting that for any region $B$, \begin{align} \int_B (dP - dQ_t) = \int_B \bracket{1 - \frac{dQ_t}{dP}}dP \end{align} we then use the inequality $1 - x \leq \log(1/x)$ to get \begin{align} \int_B (dP - dQ_t) = \int_B \bracket{1 - \frac{dP}{dQ_t}}dP \leq \int_B \log \bracket{\frac{dP}{dQ_t}}dP  = \int_B \log \bracket{\frac{P}{Q_t}}dP \end{align}
Re-arranging the above inequality gives us the bound.
\end{proof}
Lemma \ref{first-region-bound} allows us to understand the
relationship between two distributions $P$ and $Q_t$ in terms regions
they capture. The general goal is to show that for a given region $B$
(which includes the highly dense mode regions), the amount of mass
captured by the model $\int_B dQ_t$, is lower bounded by the target
mass $\int_B dP$, and some small quantity. The inequality in Lemma
\ref{first-region-bound} comments on this precisely with the small
difference being a term that looks familiar to the KL-divergence -
rather one that is bound to the specific region $B$. Though, this term
can be understood to be small since by Theorem \ref{kl-upper}, we know
that the global KL decreases, we give further refinements to show the
importance of privacy parameters $\epsilon$. We show that the term
$\int_B \log(P/Q_t) dP$ can be decomposed in different ways, leading
to our two Theorems to prove.
\begin{lemma}
\label{ball-kl-bound}
\begin{align}
\int_B \log\bracket{\frac{P}{Q_t}}dP \leq \int_B \log\bracket{\frac{P}{Q_0}}dP - \Delta + \frac{\epsilon}{2}\bracket{1 - \int_B dP}.
\end{align}
where $\Delta = KL(P,Q_0) - KL(P, Q_t)$
\end{lemma}
\begin{proof}
We decompose the space $\mathcal{X}$ into $B$ and the complement $B^c$ to get \begin{align} \int_{B} \log\bracket{\frac{P}{Q_t}}dP &= \int_{\mathcal{X}} \log\bracket{\frac{P}{Q_t}}dP - \int_{B^c} \log\bracket{\frac{P}{Q_t}}dP\\
&= \text{KL}(P,Q_t) - \int_{B^c} \log\bracket{\frac{P}{Q_t}}dP\\
&\leq \text{KL}(P,Q_0) - \Delta - \int_{B^c}
\log\bracket{\frac{P}{Q_t}}dP, \end{align} where we used Theorem
\ref{kl-upper}, and letting $\theta = \theta(\epsilon)$ for brevity, we also have \begin{align} \int_{B^c} \log\bracket{\frac{P}{Q_t}}dP &=
  \int_{B^c} \log\bracket{\frac{P}{Q_0 \exp\bracket{\ip{\theta}{c} -
        \phi(\theta)}}}dP \\ &= \int_{B^c} \log\bracket{\frac{P}{Q_0}}
  dP -  \int_{B^c} \exp\bracket{\ip{\theta}{c} - \phi(\theta)}dP \\
  &\geq \int_{B^c} \log\bracket{\frac{P}{Q_0}} dP -  \int_{B^c}
  \frac{\epsilon}{2}dP \\ &= \int_{B^c} \log\bracket{\frac{P}{Q_0}} dP
  -  \frac{\epsilon}{2} \bracket{1 - \int_B dP} \end{align}
Combining these inequalities together gives us: \begin{align}  \int_{B} \log\bracket{\frac{P}{Q_t}}dP &\leq \text{KL}(P,Q_0) - \Delta - \bracket{\int_{B^c} \log\bracket{\frac{P}{Q_0}} dP -  \frac{\epsilon}{2} \bracket{1 - \int_B dP}} \\ &= \int_{\mathcal{X}} \log\bracket{\frac{P}{Q_0}} dP - \int_{B^c} \log\bracket{\frac{P}{Q_0}} dP -  \Delta + \frac{\epsilon}{2} \bracket{1 - \int_B dP} \\ &= \int_B \log\bracket{\frac{P}{Q_0}}dP - \Delta + \frac{\epsilon}{2}\bracket{1 - \int_B dP}\end{align}
\end{proof}
We are now in a position to prove Theorem \ref{first-mode-capture}.
Using Lemma \ref{ball-kl-bound} into the inequality in Lemma \ref{first-region-bound} yields
\begin{align} \int_B dQ_t &\geq \int_B dP - \bracket{\int_B \log\bracket{\frac{P}{Q_0}}dP - \Delta + \frac{\epsilon}{2}\bracket{1 - \int_B dP}} \\ &= \bracket{1 + \frac{\epsilon}{2}}\int_B dP  - \frac{\epsilon}{2} - \int_B \log\bracket{\frac{P}{Q_0}} + \Delta. \end{align}
Reorganising and using the Theorem's notations, we get
\begin{eqnarray}
\textsc{m}(B,Q) & \geq & \textsc{m}(B,P) - KL(P, Q_0; B) + \frac{\epsilon}{2}\cdot J(P,Q;B,\varepsilon),\label{eqfmc22}
\end{eqnarray}
where we recall that $J(P,Q;B,\varepsilon) \defeq \textsc{m}(B,P) + \frac{2\Delta(Q)}{\epsilon} - 1$. Theorem \ref{kl-upper-lower} says that we have in the high boosting regime $2\Delta(Q_T)/\varepsilon \geq (\gamma_P + \gamma_Q)/2 - \theta_T(\epsilon) \cdot (\gamma_P + \gamma_Q)/2$. Letting $\overline{\gamma}\defeq (\gamma_P + \gamma_Q)/2$ and $K \defeq 4\log 2$, we have from \pkde~in the high boosting regime:
\begin{eqnarray}
\frac{2\Delta(Q)}{\varepsilon} & \geq & \overline{\gamma} \cdot \left(1 - \left(\frac{1}{1+\frac{K}{\varepsilon}}\right)^T\right) \nonumber\\
& \geq & \overline{\gamma} \cdot \left(1 - \frac{1}{1+\frac{TK}{\varepsilon}}\right)  \nonumber\\
& & = \overline{\gamma} \cdot \frac{TK}{TK+\varepsilon}.
\end{eqnarray}
To have $J(P,Q;B,\varepsilon) \geq - (2/\varepsilon)\cdot\alpha \textsc{m}(B,P)$, it is thus sufficient that 
\begin{eqnarray}
\textsc{m}(B,P) & \geq & \frac{1}{1+\frac{2\alpha}{\varepsilon}} \cdot \left(1 - \overline{\gamma} \cdot \frac{TK}{TK+\varepsilon}\right) \nonumber\\
& & = \varepsilon \cdot \frac{\varepsilon + (1-\overline{\gamma})TK}{(\varepsilon + 2\alpha)(\varepsilon + TK)}.
\end{eqnarray}
In this case, we check that we have from \eqref{eqfmc22}
\begin{eqnarray}
\textsc{m}(B,Q) & \geq & (1 - \alpha)\textsc{m}(B,P) - KL(P, Q_0; B) ,
\end{eqnarray}
as claimed.

\subsection{Additional formal results}\label{proof_nonTrivial} 

One might ask what such a strong model of privacy as integral
privacy allows
to keep from the accuracy standpoint in general. Perhaps paradoxically
at first sight, it is
not hard to show that integral privacy can bring approximation
guarantees on learning: \textit{if} we learn $Q_\varepsilon$ within an
$\varepsilon$-mollifier $\mathcal{M}$ (hence, we get 
$\varepsilon$-integral privacy for sampling from $Q_\varepsilon$), \textit{then}
each
time \textit{some} $Q_\varepsilon$ in $\mathcal{M}$ accurately fits $P$, we are guaranteed
that the one \textit{we learn} also accurately fits $P$ --- albeit
eventually more moderately ---. We let $Q_{\epsilon}(;.)$ denote the density
learned, where $.$ is the dataset argument.
\begin{lemma}\label{nonTrivial}
Suppose $\exists$ $\varepsilon$-mollifier $\mathcal{M}$
s.t. $Q_{\epsilon} \in \mathcal{M}$, then $(\exists P, D', \delta : \text{KL}(P,Q_{\epsilon}(;D')) \leq \delta)
  \Rightarrow (\forall D, \text{KL}(P,Q_{\epsilon}(;D)) \leq \delta + \epsilon)$.
\end{lemma}
\begin{proof}
The proof is straightforward; we give it for completeness: for any dataset $D$, we have
\begin{align} 
\text{KL}(P,Q_{\epsilon}(;D)) &= \int_{\mathcal{X}} \log\bracket{\frac{P}{Q_{\epsilon}(;D) }}dP\\ &= \int_{\mathcal{X}} \log\bracket{\frac{P}{Q_{\epsilon}(;D') }}dP + \int_{\mathcal{X}} \log\bracket{\frac{Q_{\epsilon}(;D)}{Q_{\epsilon}(;D') }}dP\\ &\leq \int_{\mathcal{X}} \log\bracket{\frac{P}{Q_{\epsilon}(;D') }}dP + \epsilon \cdot \int_{\mathcal{X}} dP\\ &= \text{KL}(P,Q_{\epsilon}(;D')) + \epsilon \\ &\leq \delta + \epsilon,
\end{align}
from which we derive the statement of Lemma \ref{nonTrivial} assuming $\mathcal{A}$ is $\epsilon$-IP (the inequalities follow from the Lemma's assumption).
\end{proof}
In the jargon of (computational) information geometry \cite{bnnBV}, we can summarize Lemma \ref{nonTrivial} as saying that if there exists an eligible\footnote{Within the chosen $\varepsilon$-mollifier.} density in a small KL-ball relatively to $P$, we are guaranteed to find a density also in a small KL-ball relatively to $P$. This result is obviously good when the premises hold true, but it does not tell the full story when they do not. In fact, when there exists an eligible density outside a big KL-ball relatively to $P$, it is not hard to show using the same arguments as for the Lemma that we \textit{cannot} find a good one, and this is not a feature of \pkde: this would hold regardless of the algorithm. This limitation is intrinsic to the likelihood ratio constraint of differential privacy in \eqref{constDP} and not to the neighborhing constraint that we alleviate in integral privacy. It is therefore also a limitation of classical $\varepsilon$-differential privacy, as the following Lemma shows. In the context of $\varepsilon$-DP, we assume that all input datasets have the same size, say $m$.
\begin{lemma}\label{nonTrivial2}
Let $\mathcal{A}$ denote an algorithm learning an $\varepsilon$-differentially private density. Denote $D\sim P$ an input of the algorithm and $\mathcal{Q}_{\epsilon}(D)$ the set of all densities that can be the output of $\mathcal{A}$ on input $D$, taking in considerations all internal randomisations of $\mathcal{A}$. Suppose there exists an input $D'$ for which one of these densities is far from the target: $\exists D', \exists Q \in  \mathcal{Q}_{\epsilon}(D'): \text{KL}(P,Q(;D')) \geq \Delta$ for some "big" $\Delta > 0$. Then the output $Q$ of $\mathcal{A}$ obtained from \textbf{any} input $D\sim P$ satisfies: $\text{KL}(P,Q(; D)) \geq \Delta - m \varepsilon$.
\end{lemma}
\begin{proof}
Denote $D$ the actual input of $\mathcal{A}$. There exists a sequence $\mathcal{D}$ of datasets of the same size, whose length is at most $m$, which transforms $D$ into $D'$ by repeatedly changing one observation in the current dataset: call it $\mathcal{D} = \{D, D_1, D_2, ..., D_k, D'\}$, with $k \leq m-1$. Denote $Q(; D'')$ any element of $\mathcal{Q}_{\epsilon}(D'')$ for $D'' \in \mathcal{D}$. Since $\mathcal{A}$ is $\varepsilon$-differentially private, we have:
\begin{align} 
\Delta & \leq \text{KL}(P,Q(;D')) \\
&= \int_{\mathcal{X}} \log\bracket{\frac{P}{Q(; D')} }dP\\ &= \int_{\mathcal{X}} \log\bracket{\frac{P}{Q(; D)}}dP + \int_{\mathcal{X}} \log\bracket{\frac{Q(;D)}{Q(;D_1) }}dP + \sum_{j=1}^{k-1} \int_{\mathcal{X}} \log\bracket{\frac{Q(;D_j)}{Q(;D_{j+1}) }}dP + \int_{\mathcal{X}} \log\bracket{\frac{Q(;D_k)}{Q(;D') }}dP \\ &= \text{KL}(P,Q(;D)) + \int_{\mathcal{X}} \log\bracket{\frac{Q(;D)}{Q(;D_1) }}dP + \sum_{j=1}^{k-1} \int_{\mathcal{X}} \log\bracket{\frac{Q(;D_j)}{Q(;D_{j+1}) }}dP + \int_{\mathcal{X}} \log\bracket{\frac{Q(;D_k)}{Q(;D') }}dP \\ &\leq \text{KL}(P,Q(;D)) + m \epsilon,
\end{align}
from which we derive the statement of Lemma \ref{nonTrivial2}.
\end{proof}

\newpage
\section{Additional experiments}\label{supp-exp}

We provide here additional results to the main file. Figure \ref{RG_NLL_DPBvsUS} provides NLL values for the random 1D Gaussian problem.
Figure \ref{random:DPB-vs-proposed} displays that picking $Q_0$ a standard Gaussian does not prevent to obtain good results --- and beat DPB --- when sampling random Gaussians.

\begin{figure*}
\centering
\begin{tabular}{cc||cc}\\ \hline \hline
\multicolumn{2}{c||}{DPB} & \multicolumn{2}{c}{\pkde} 
  \\ 
\hspace{\whep} \includegraphics[trim=20bp 20bp 12bp
20bp,clip,width=\whed]{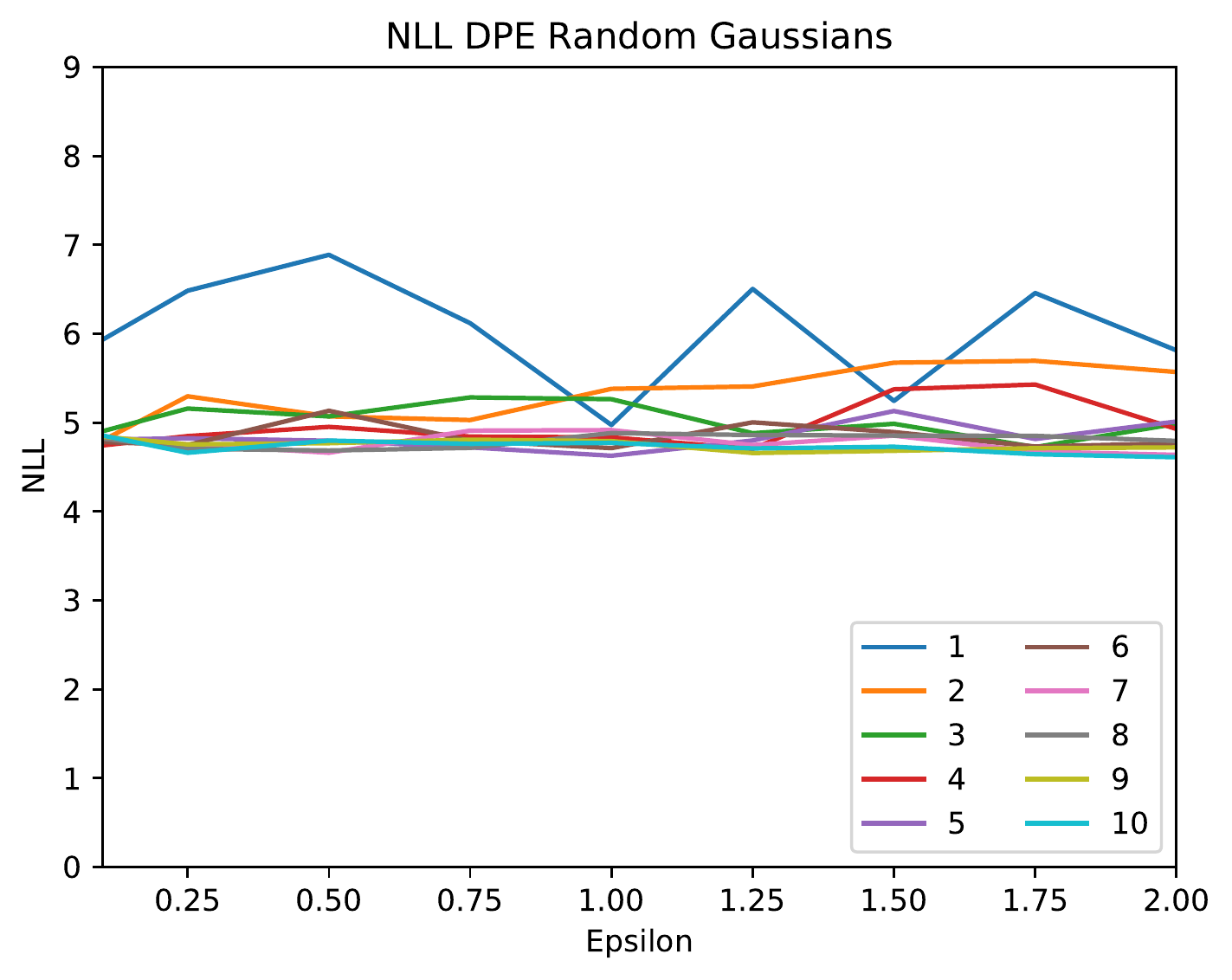}
\hspace{\whep} & \hspace{\whep} \includegraphics[trim=22bp 20bp 12bp
20bp,clip,width=\whed]{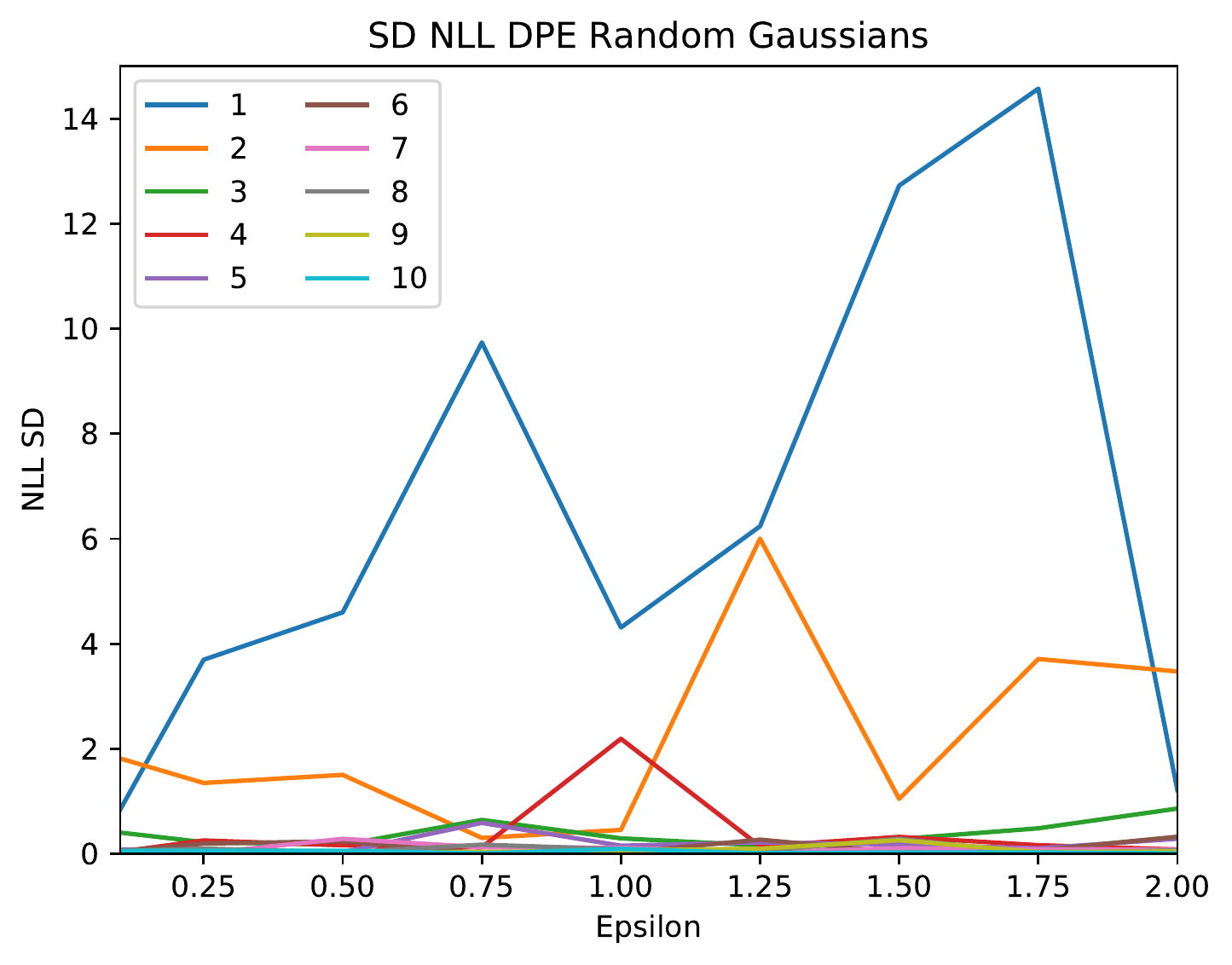} \hspace{\whep} & \hspace{\whep} \includegraphics[trim=20bp 20bp 8bp
20bp,clip,width=\whed]{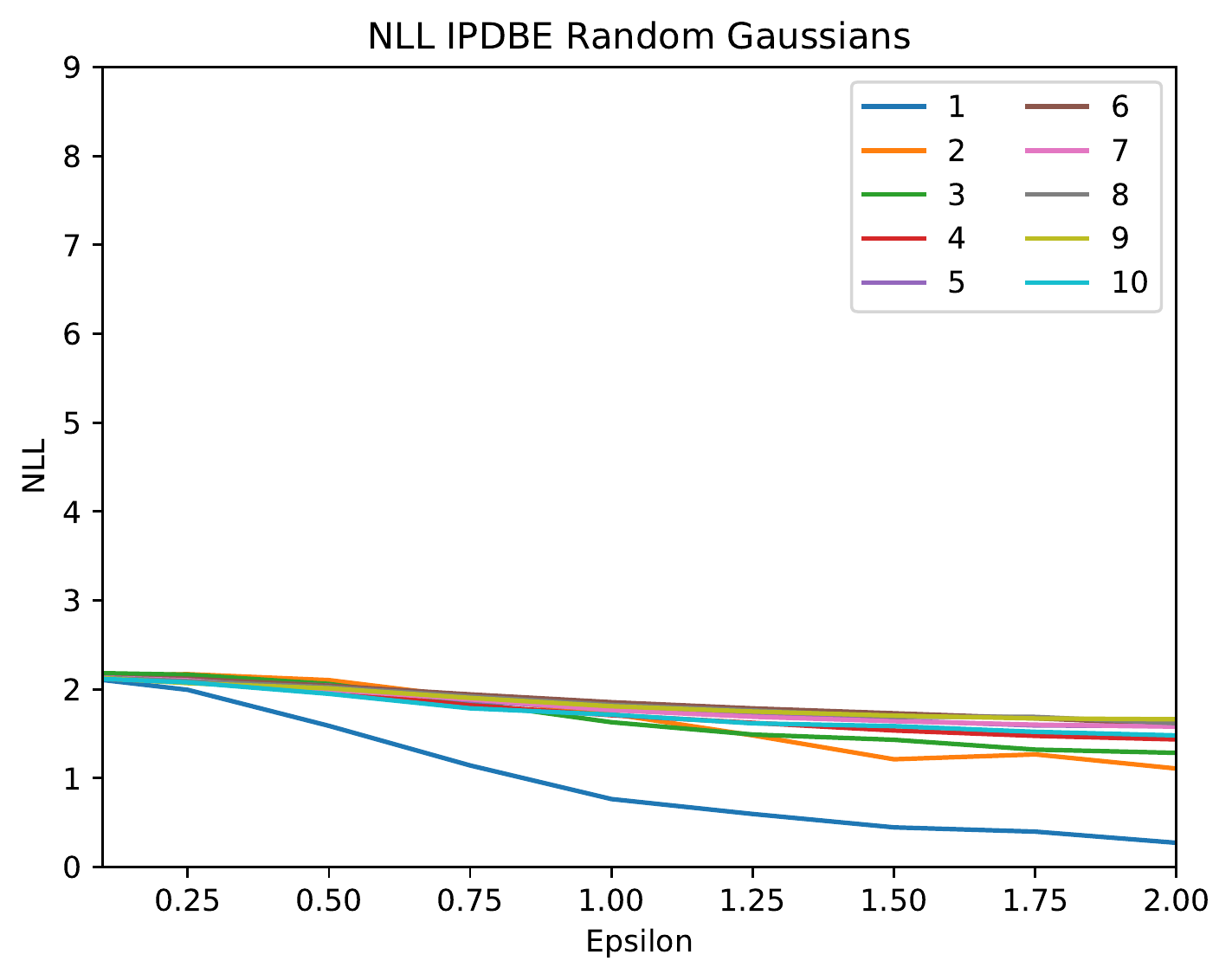} 
\hspace{\whep} & \hspace{\whep}\includegraphics[trim=20bp 20bp 12bp
22bp,clip,width=\whed]{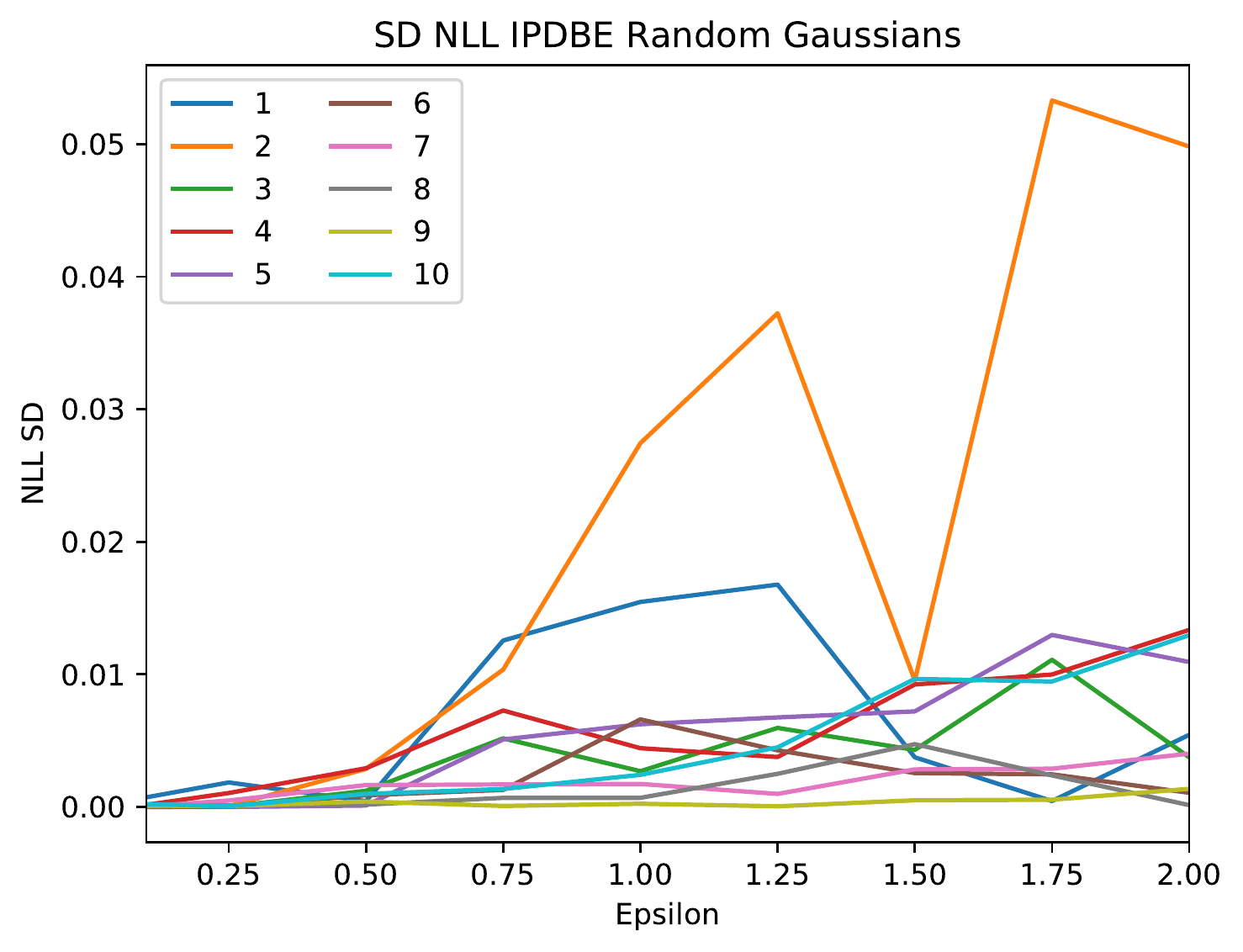} \hspace{\whep} \\
\hspace{\whep}  Mean = f($\varepsilon$) \hspace{\whep} & \hspace{\whep}
                                                        StDev
                                                        =
                                                        f($\varepsilon$)
                                                        \hspace{\whep}
               & \hspace{\whep} Mean = f($\varepsilon$) \hspace{\whep}
                                                      & \hspace{\whep}
                                                        StDev
                                                        =
                                                        f($\varepsilon$)
  \\ \hline\hline
\end{tabular}
\caption{NLL metrics (mean and standard deviation) on the 1D random Gaussian problem for DPB (left pane)
  and \pkde~(right pane), for a varying number of $m = 1,\ldots,10$
  random Gaussians. The lower the better on each metric. Remark
  the different scales for StDev (see text).}
\label{RG_NLL_DPBvsUS}
\vspace{-0.3cm}
\end{figure*}

\newcommand{\whee}{1.95cm}

\begin{figure}
\centering
\scalebox{.78}{\begin{tabular}{ccccccccc}
                 \nspp \includegraphics[width=\whee,height=\whee]{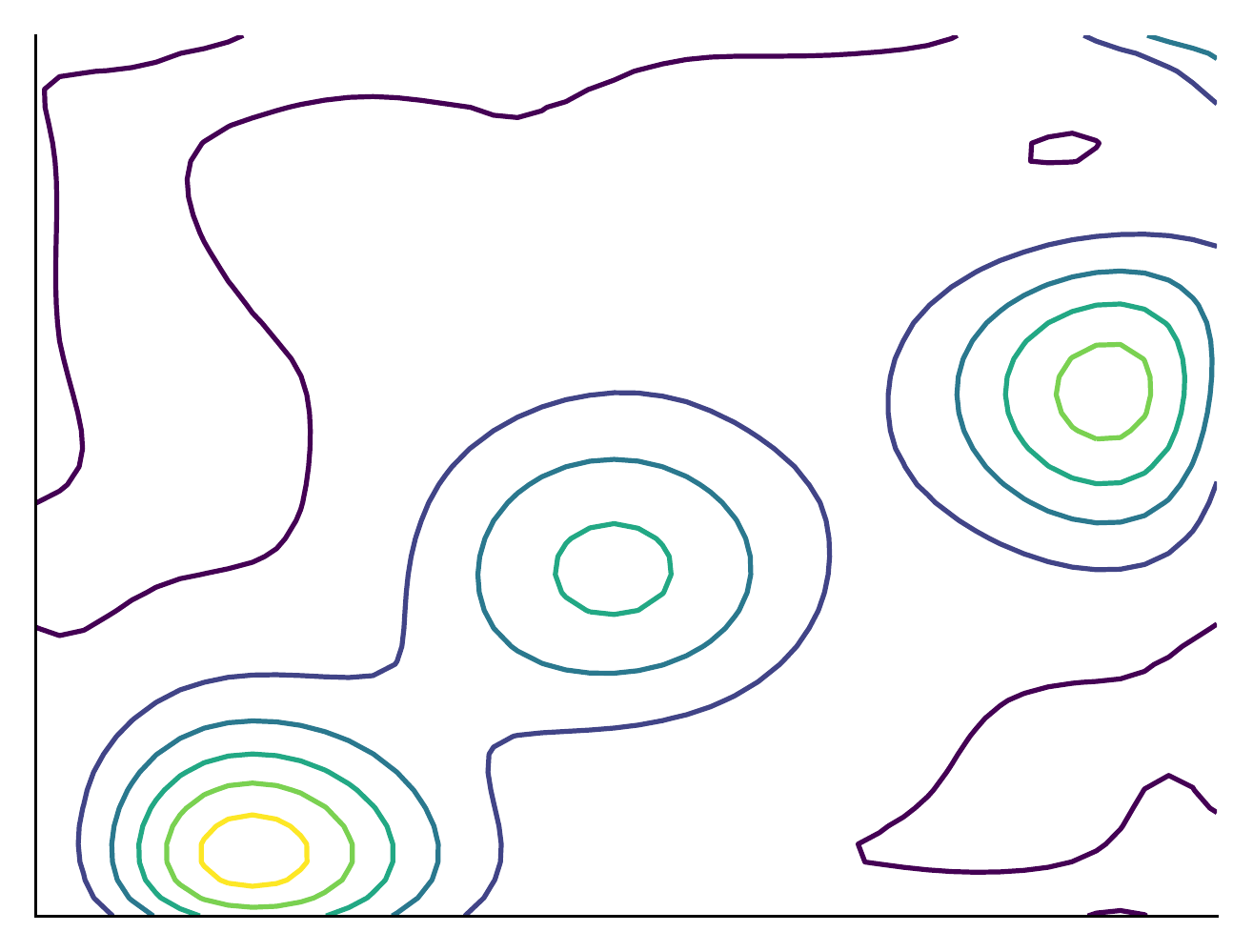}
                 \nsp & \nsp \includegraphics[width=\whee,height=\whee]{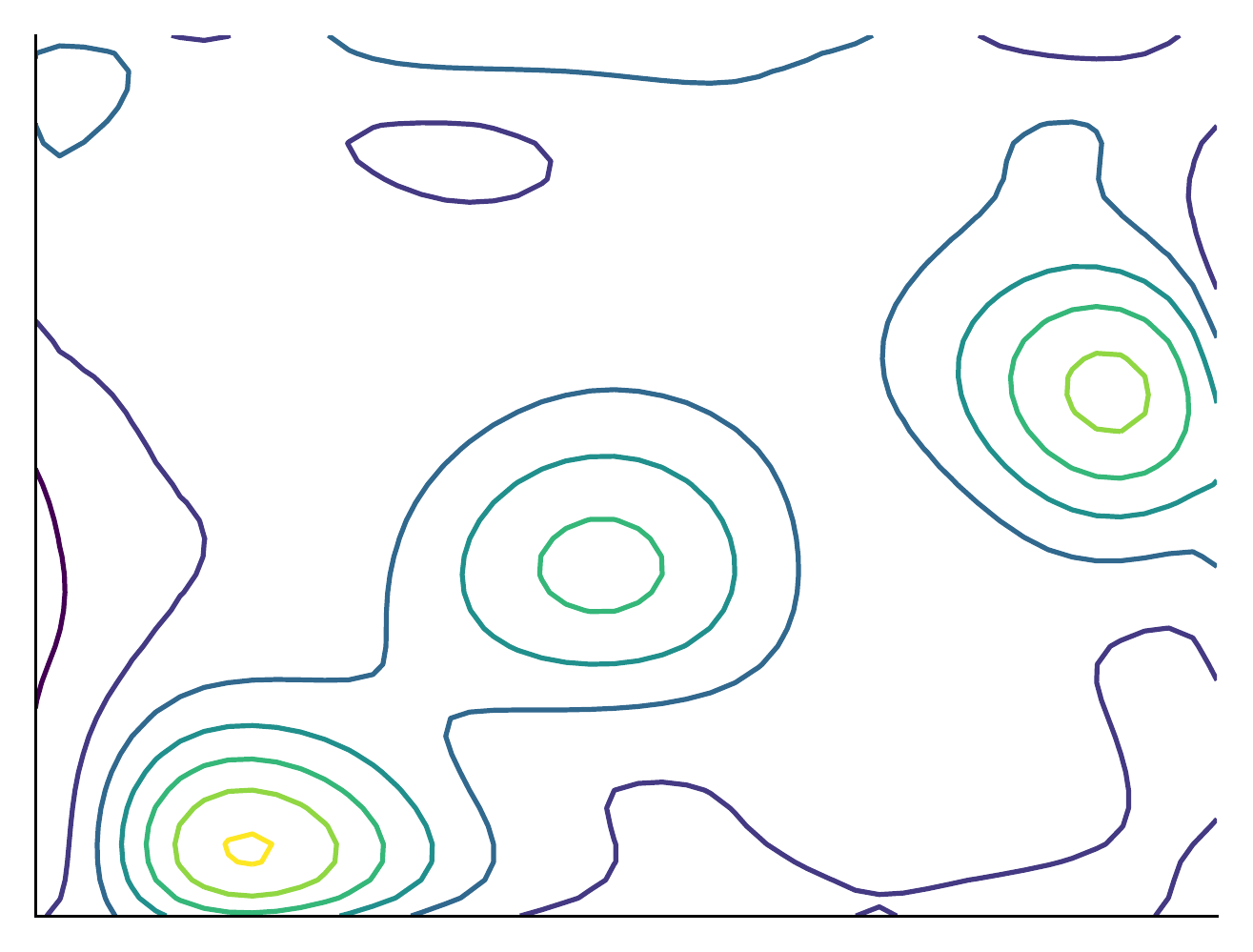}
                 \nsp & \nsp
                   \includegraphics[width=\whee,height=\whee]{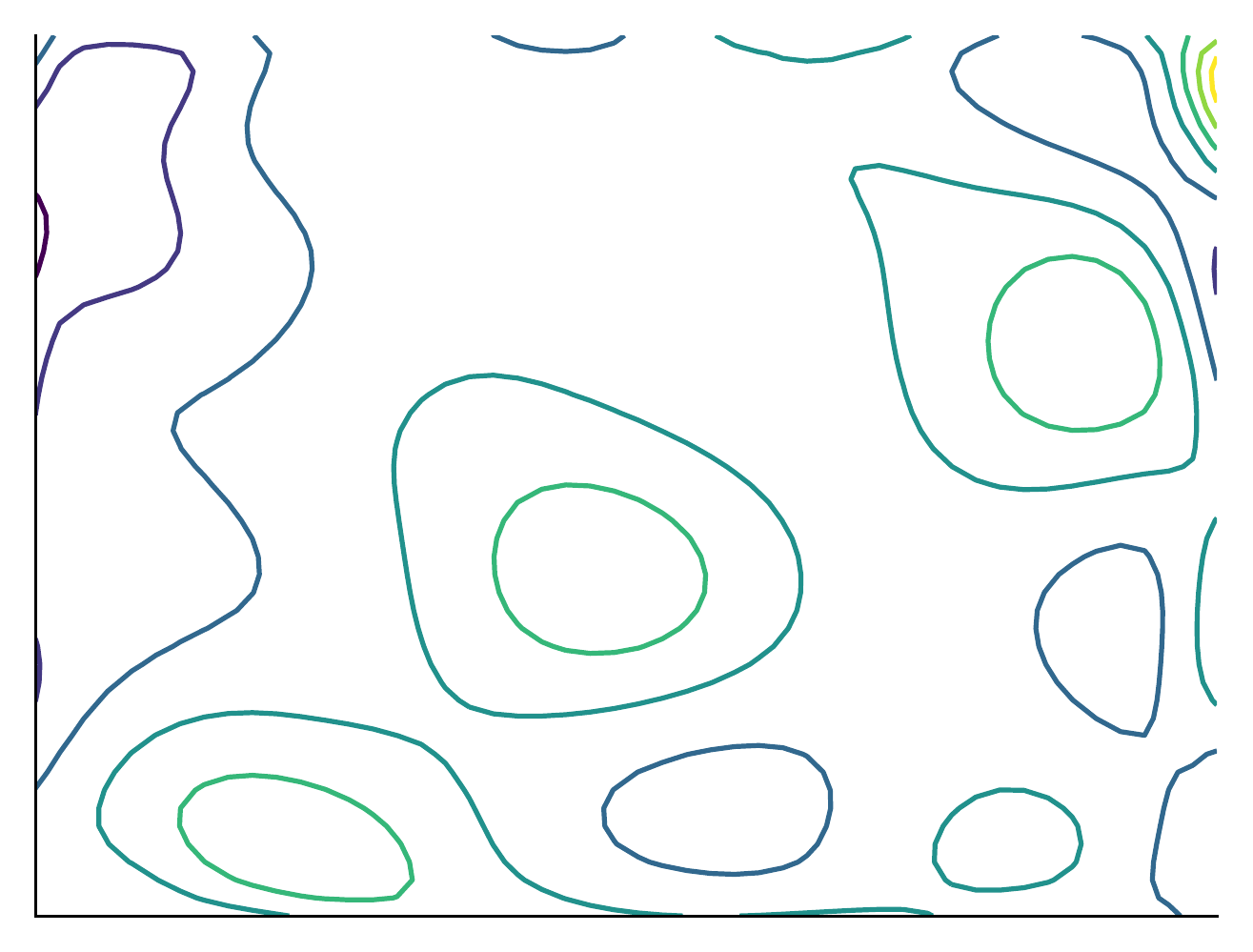}
                 \nsp & \nsp \includegraphics[width=\whee,height=\whee]{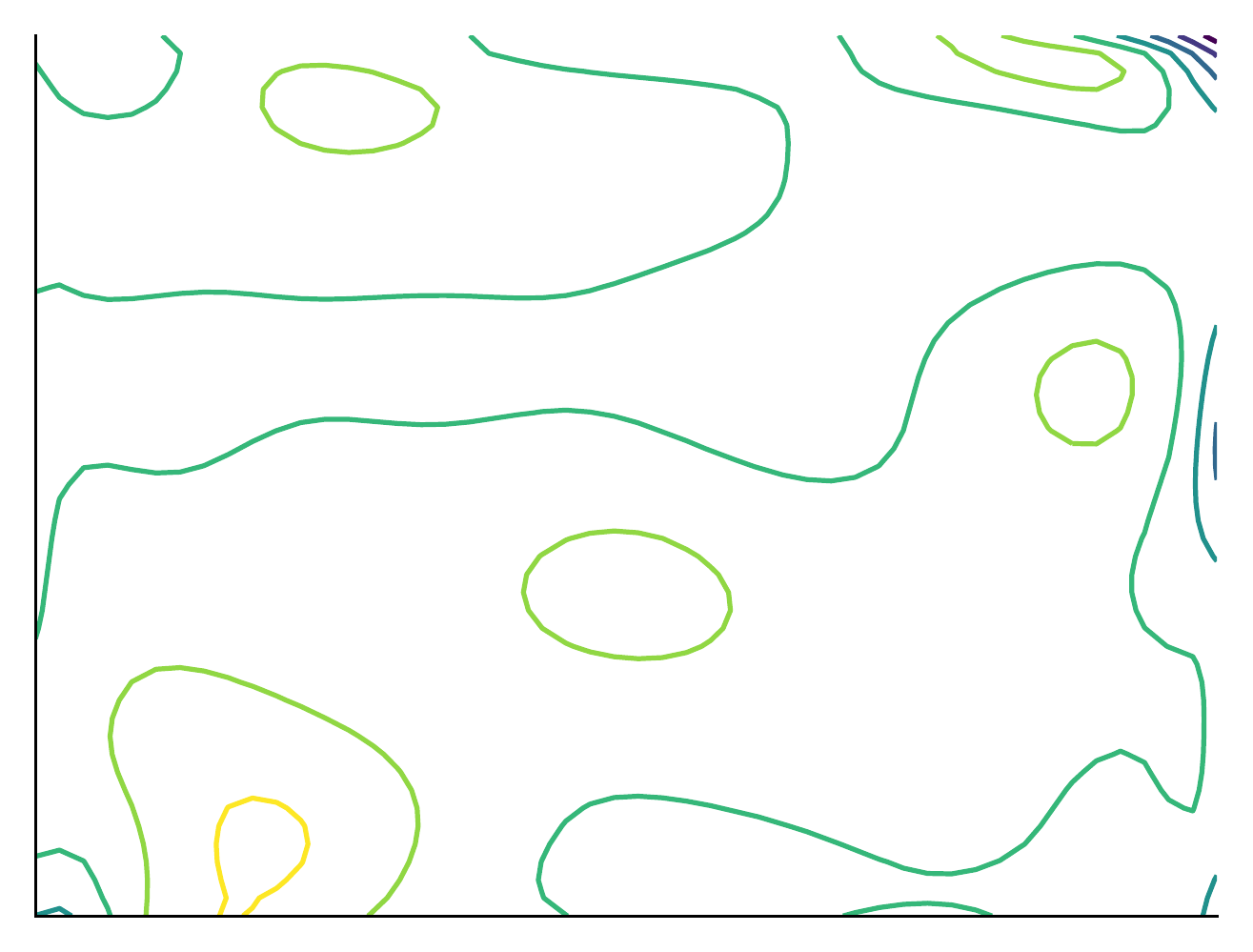}
                   \nsp & \nsp
                     \includegraphics[width=\whee,height=\whee]{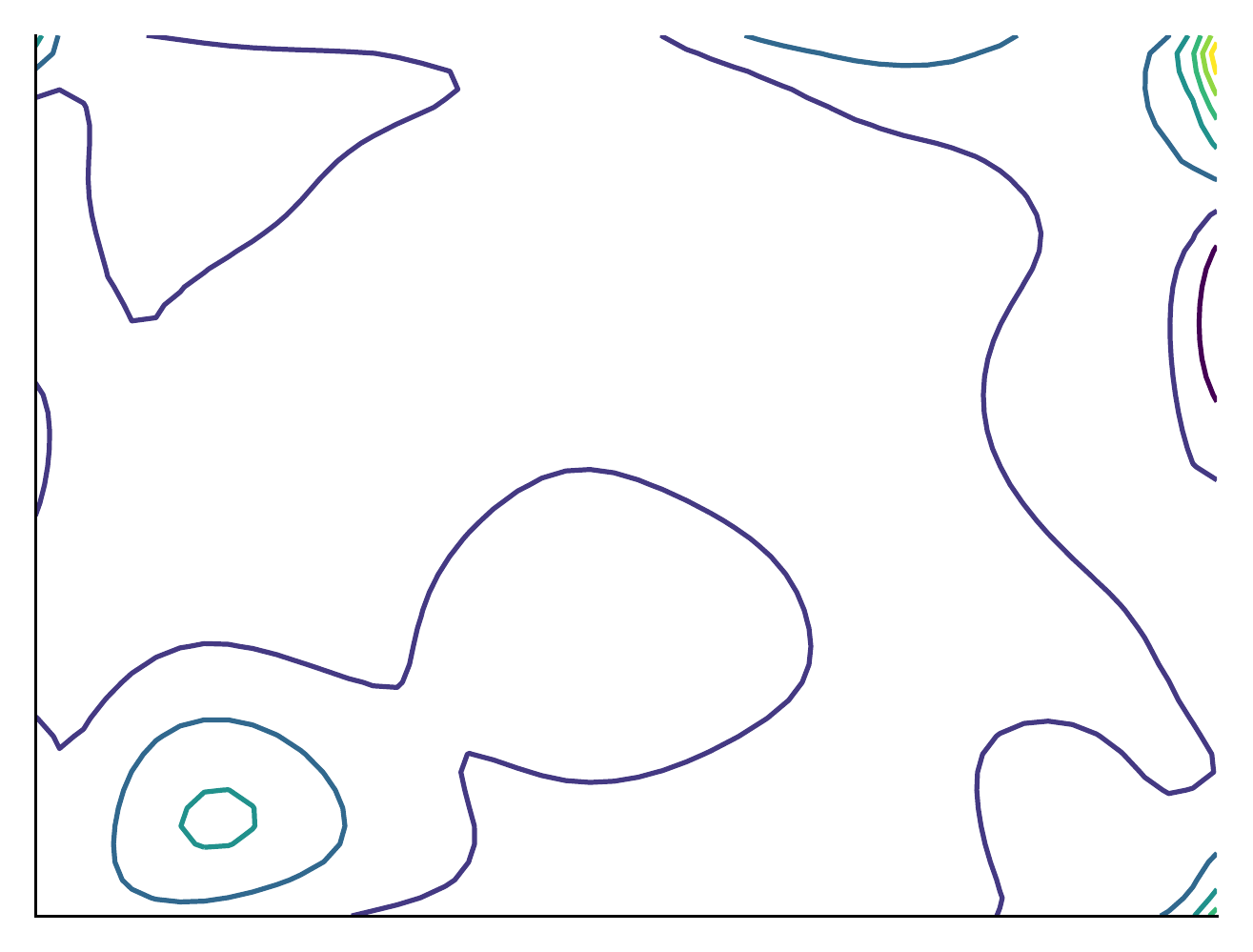}
                     \nsp & \nsp
                       \includegraphics[width=\whee,height=\whee]{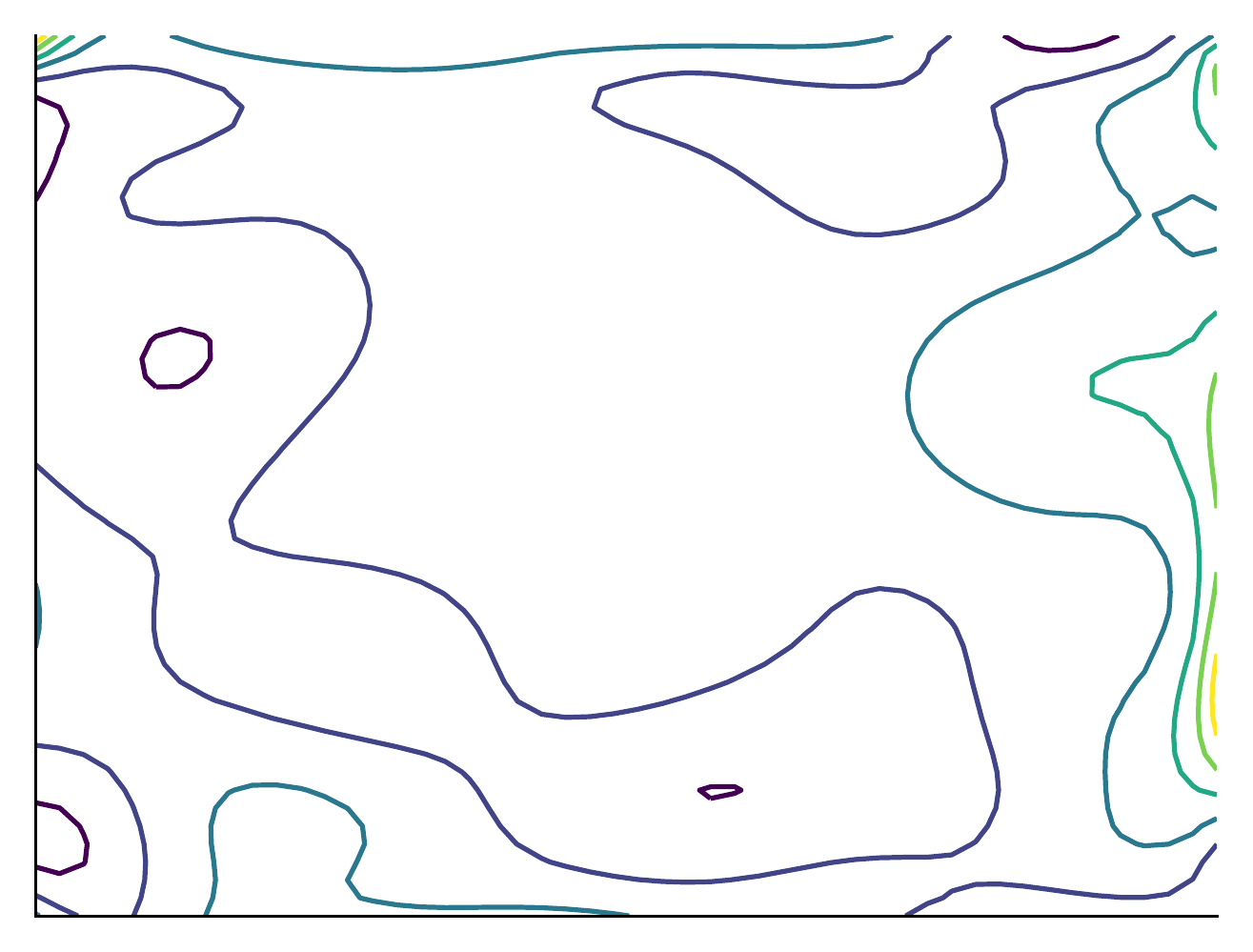}
                       \nsp & \nsp
                         \includegraphics[width=\whee,height=\whee]{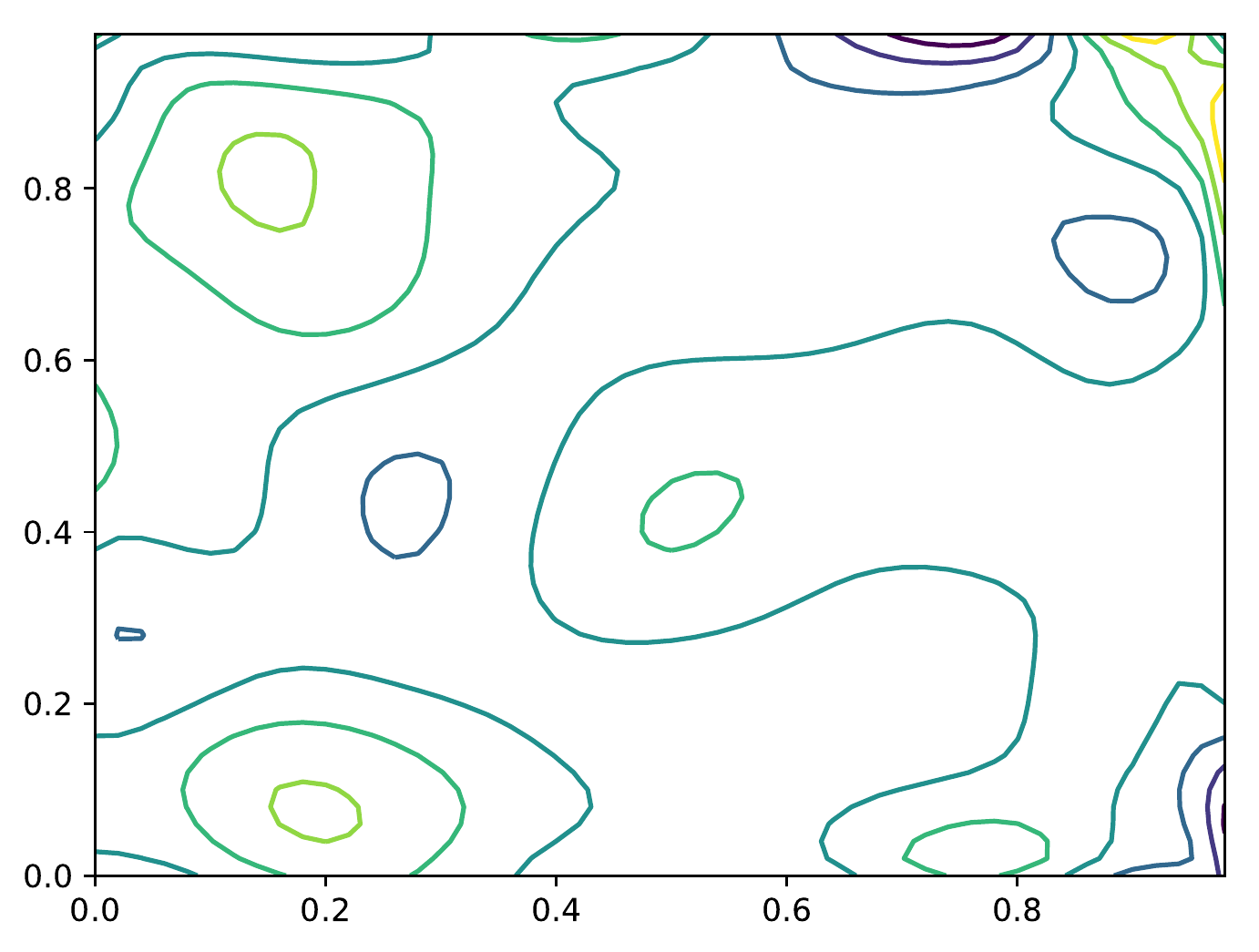}
                         \nsp & \nsp
                           \includegraphics[width=\whee,height=\whee]{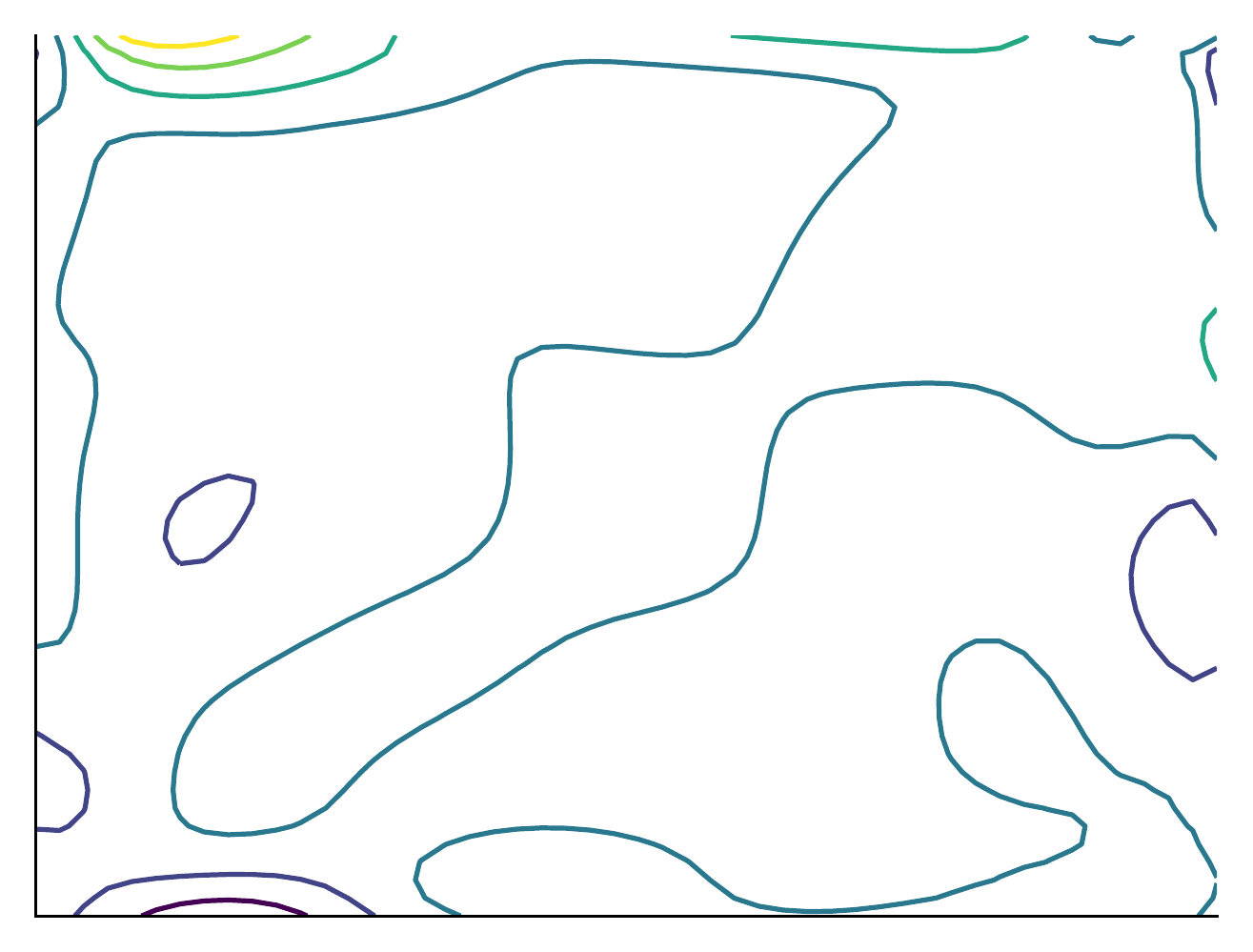}
                           \nsp & \nsp
                                  \includegraphics[width=\whee,height=\whee]{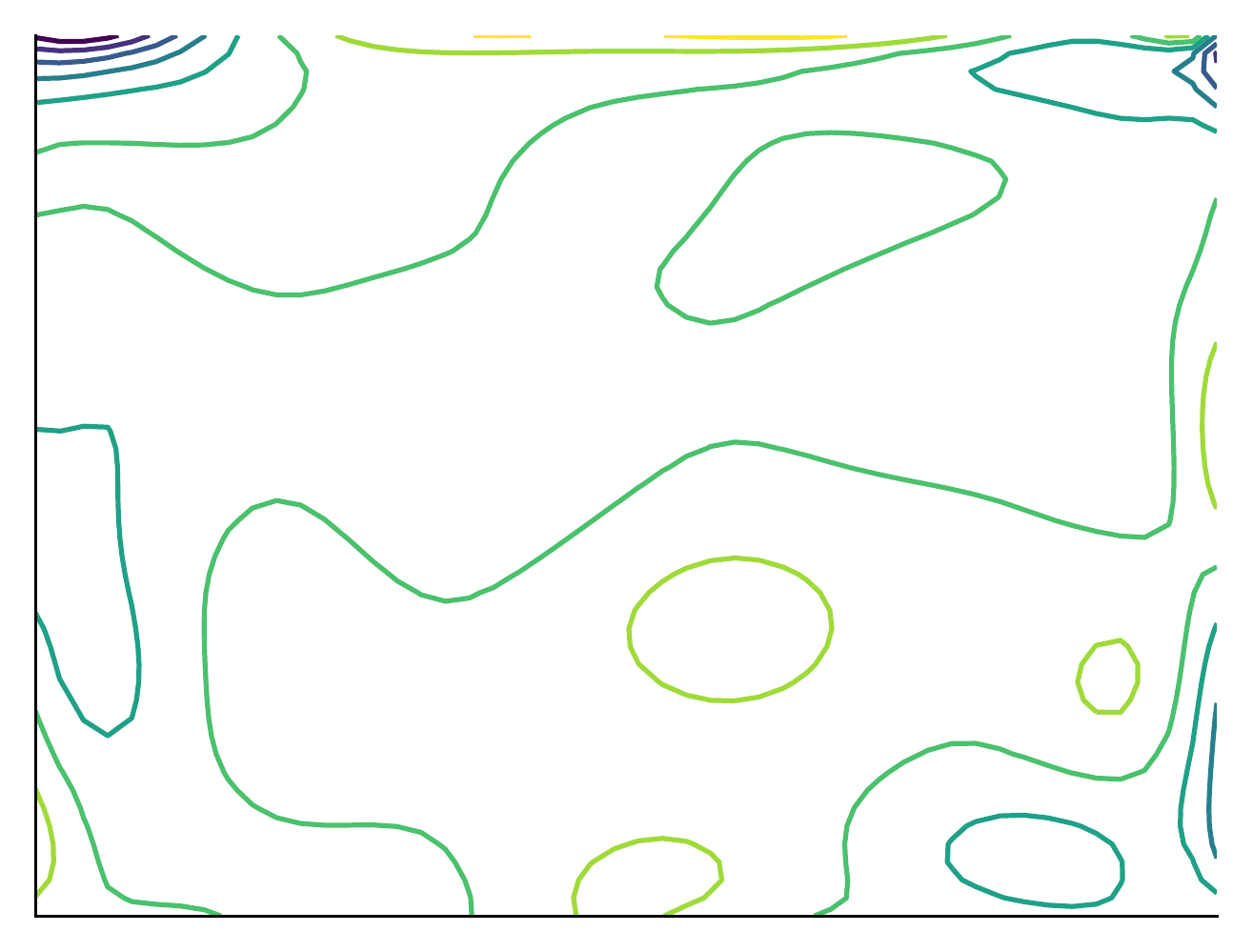}\\
                 \nspp \includegraphics[width=\whee,height=\whee]{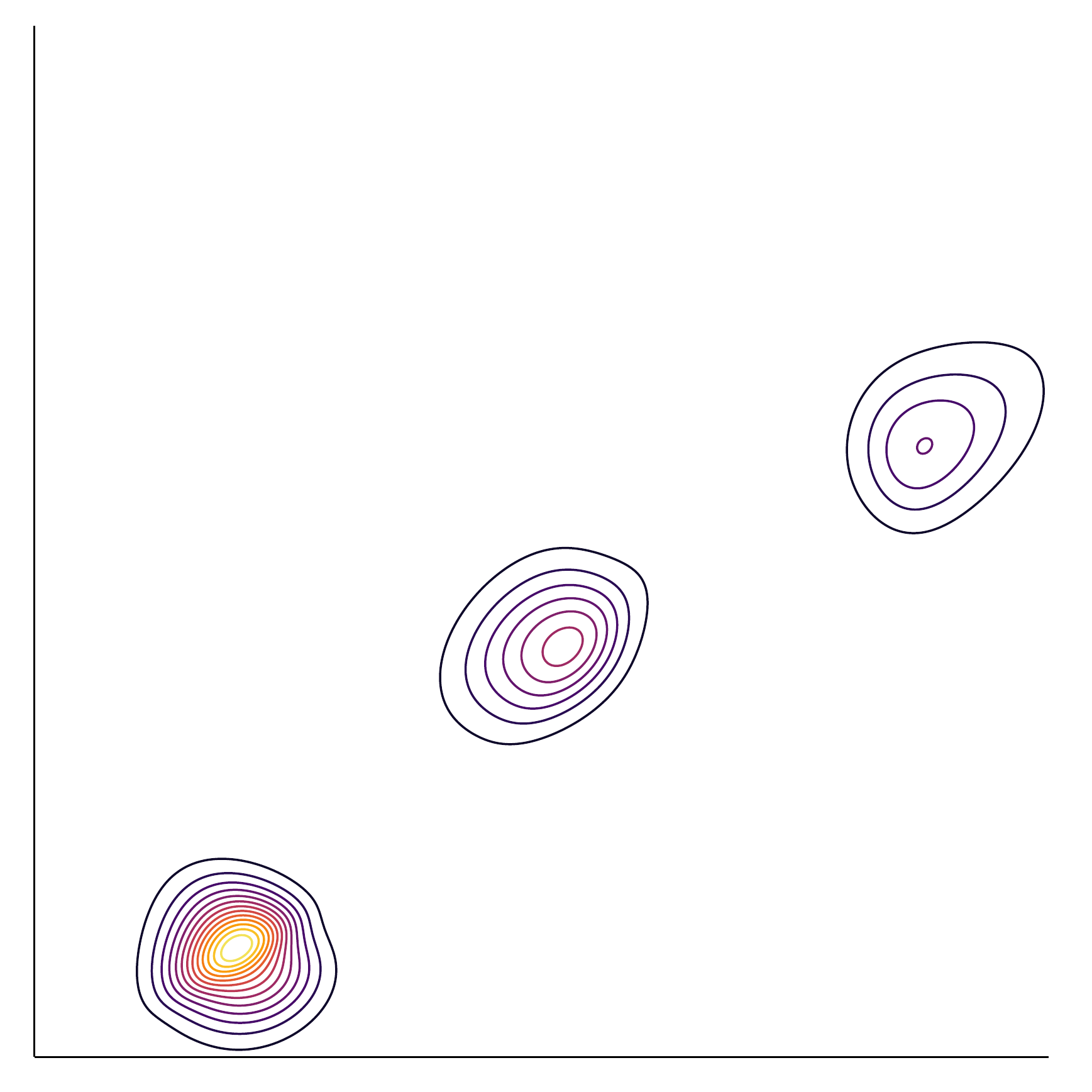}
                 \nsp & \nsp
                   \includegraphics[width=\whee,height=\whee]{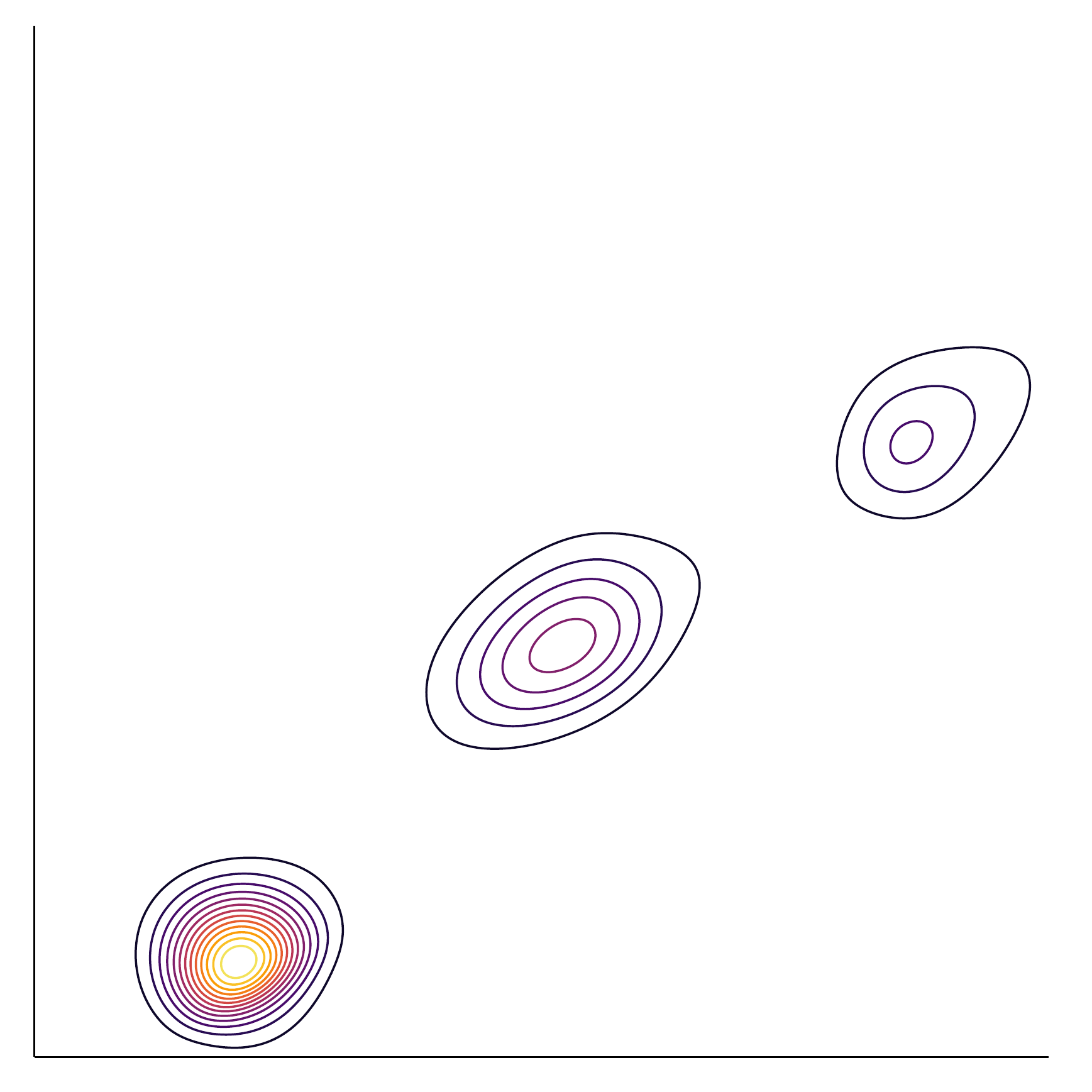}
                 \nsp &\nsp
                   \includegraphics[width=\whee,height=\whee]{Figs/random_gaussians_eps_2}
                 \nsp & \nsp \includegraphics[width=\whee,height=\whee]{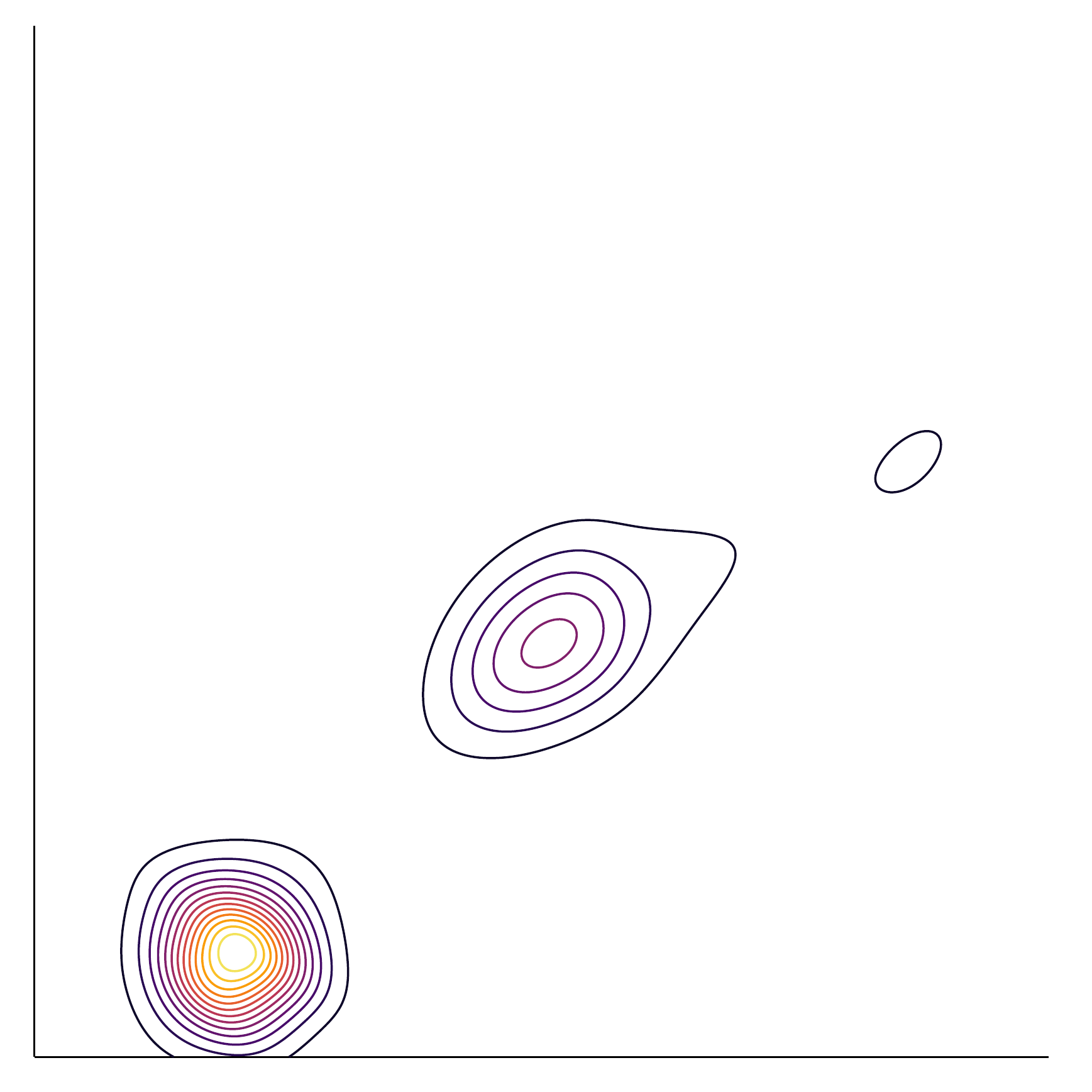}
                   \nsp & \nsp
                     \includegraphics[width=\whee,height=\whee]{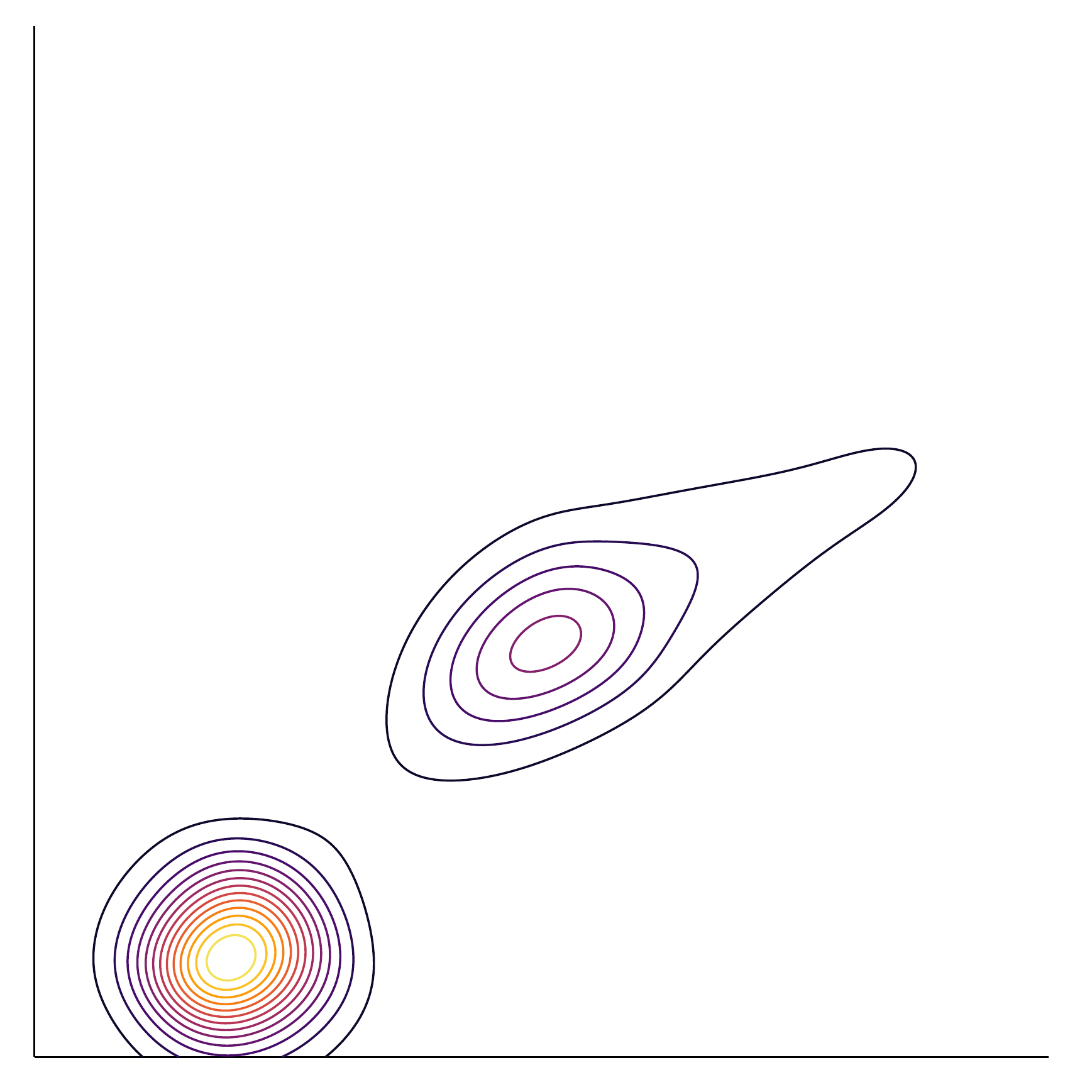}
                     \nsp & \nsp
                       \includegraphics[width=\whee,height=\whee]{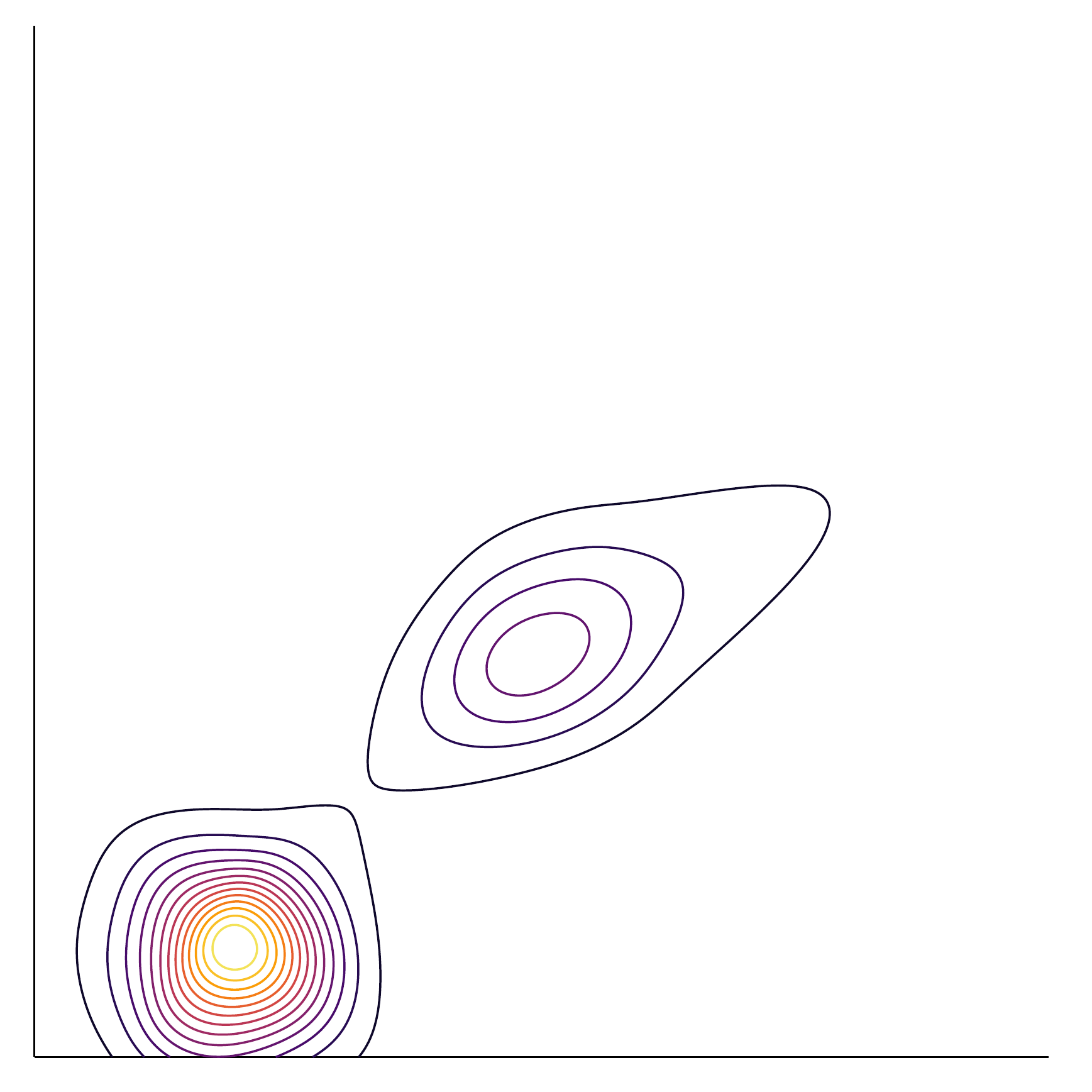}
                       \nsp & \nsp
                         \includegraphics[width=\whee,height=\whee]{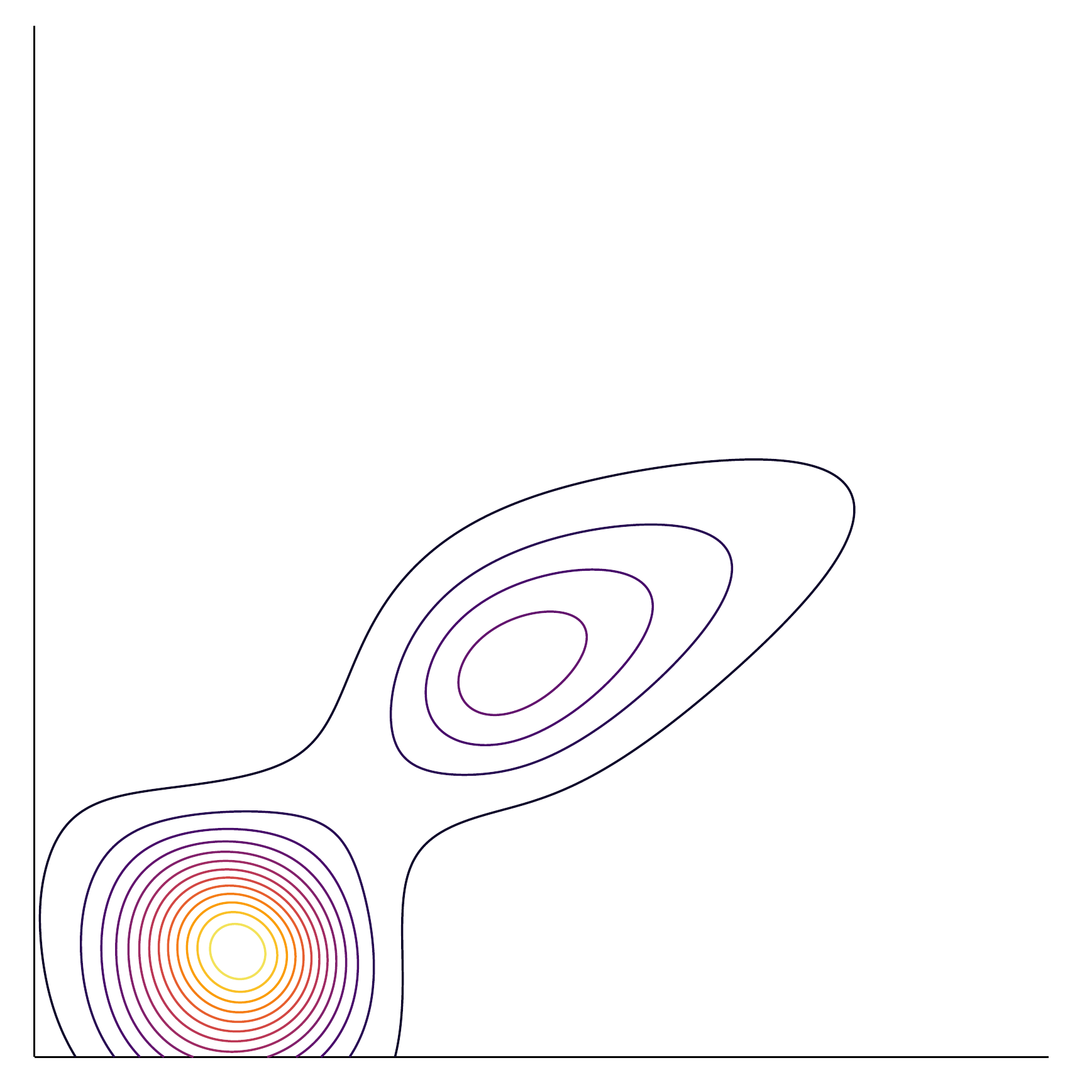}
                         \nsp & \nsp
                           \includegraphics[width=\whee,height=\whee]{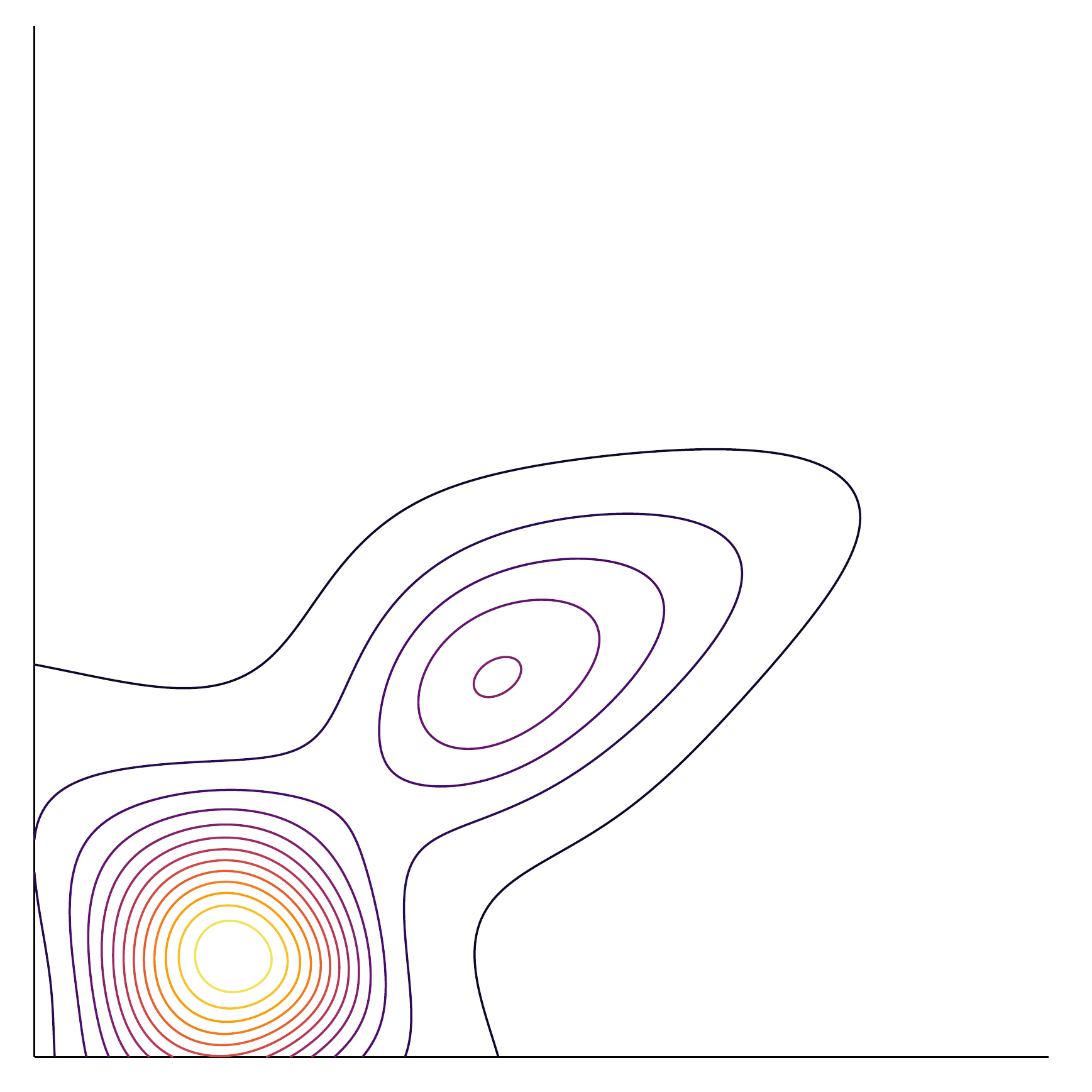}
                           \nsp & \nsp
                                  \includegraphics[width=\whee,height=\whee]{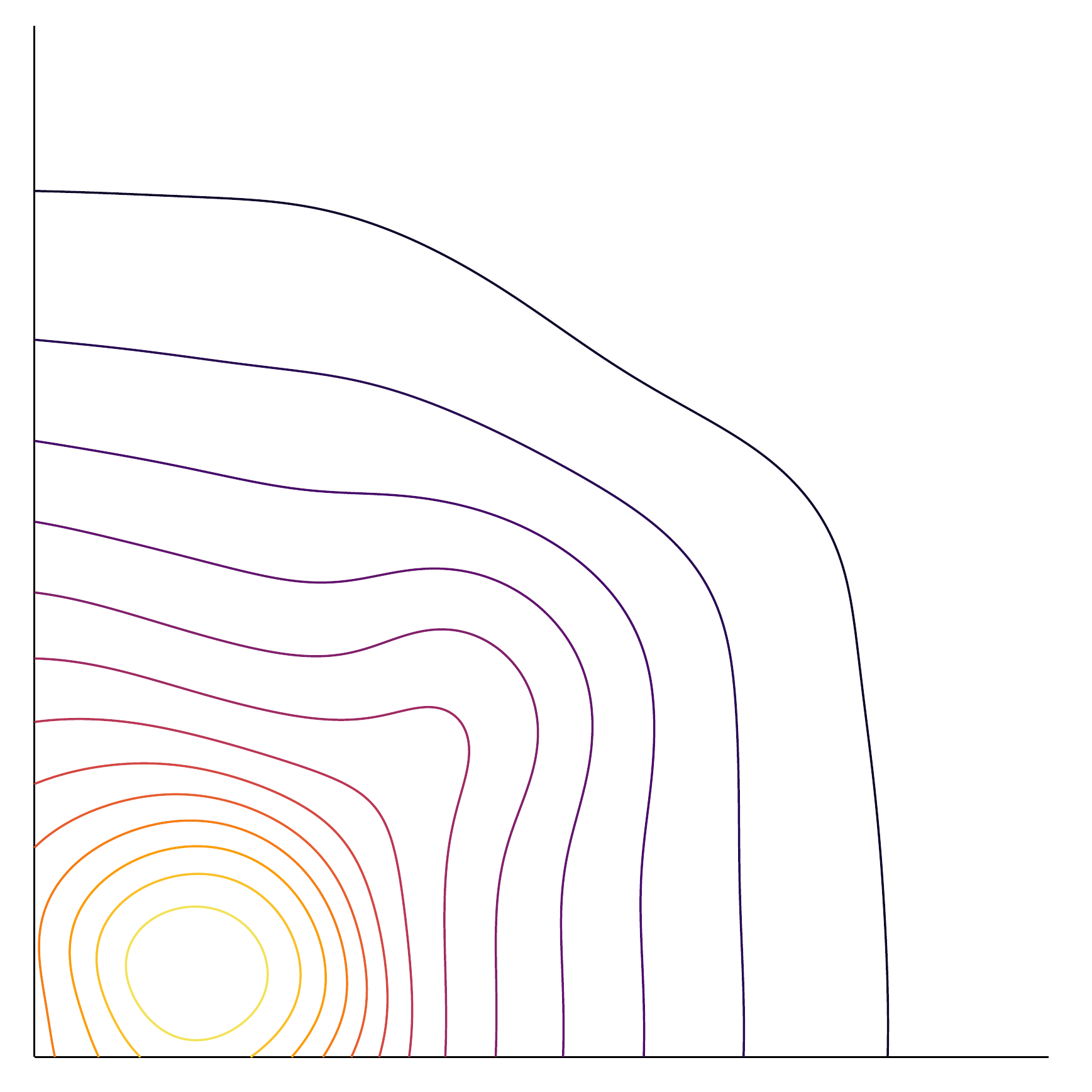}\\
                 \nspp $\epsilon = 10$
                 \nsp & \nsp
                   $\epsilon = 5$
                 \nsp & \nsp
                   $\epsilon = 2$
                 \nsp & \nsp $\epsilon = 1.5$
                   \nsp & \nsp
                     $\epsilon = 1$
                     \nsp & \nsp
                       $\epsilon = 0.75$
                       \nsp & \nsp
                         $\epsilon = 0.5$
                         \nsp & \nsp
                           $\epsilon = 0.25$
                           \nsp & \nsp
                                  $\epsilon = 0.1$
                 \end{tabular}
}
\caption{Randomly placed Gaussian convergence comparison for DPB (upper) against \pkde~ (lower).}
\label{random:DPB-vs-proposed}
\end{figure}

\end{document}